\documentclass[pcolor=EC6500, scolor=99C000, jointaffiliation]{ccs}
\usepackage[utf8]{inputenc}
\usepackage[T1]{fontenc}
\usepackage[english]{babel}

\usepackage[english=british]{csquotes}
\usepackage{mathtools, amsmath, amsthm, amssymb, dsfont}
\usepackage{wasysym}
\usepackage{tikz}
\usepackage[ruled, linesnumbered]{algorithm2e}
\usepackage[export]{adjustbox}
\usepackage{booktabs}
\usepackage{nicematrix}
\usepackage[singlelinecheck=false]{subcaption}

\usetikzlibrary{spath3, intersections, angles}
\usetikzlibrary{decorations.markings}

\definecolor{tud-0a}{HTML}{DCDCDC}
\definecolor{tud-1a}{HTML}{5D85C3}
\definecolor{tud-2a}{HTML}{009CDA}
\definecolor{tud-3a}{HTML}{50B695}
\definecolor{tud-4a}{HTML}{AFCC50}
\definecolor{tud-5a}{HTML}{DDDF48}
\definecolor{tud-6a}{HTML}{FFE05C}
\definecolor{tud-7a}{HTML}{F8BA3C}
\definecolor{tud-8a}{HTML}{EE7A34}
\definecolor{tud-9a}{HTML}{E9503E}
\definecolor{tud-10a}{HTML}{C9308E}
\definecolor{tud-11a}{HTML}{804597}

\definecolor{tud-0b}{HTML}{B5B5B5}
\definecolor{tud-1b}{HTML}{005AA9}
\definecolor{tud-2b}{HTML}{0083CC}
\definecolor{tud-3b}{HTML}{009D81}
\definecolor{tud-4b}{HTML}{99C000}
\definecolor{tud-5b}{HTML}{C9D400}
\definecolor{tud-6b}{HTML}{FDCA00}
\definecolor{tud-7b}{HTML}{F5A300}
\definecolor{tud-8b}{HTML}{EC6500}
\definecolor{tud-9b}{HTML}{E6001A}
\definecolor{tud-10b}{HTML}{A60084}
\definecolor{tud-11b}{HTML}{721085}

\definecolor{tud-0c}{HTML}{898989}
\definecolor{tud-1c}{HTML}{004E8A}
\definecolor{tud-2c}{HTML}{00689D}
\definecolor{tud-3c}{HTML}{008877}
\definecolor{tud-4c}{HTML}{7FAB16}
\definecolor{tud-5c}{HTML}{B1BD00}
\definecolor{tud-6c}{HTML}{D7AC00}
\definecolor{tud-7c}{HTML}{D28700}
\definecolor{tud-8c}{HTML}{CC4C03}
\definecolor{tud-9c}{HTML}{B90F22}
\definecolor{tud-10c}{HTML}{951169}
\definecolor{tud-11c}{HTML}{611C73}

\definecolor{tud-0d}{HTML}{535353}
\definecolor{tud-1d}{HTML}{243572}
\definecolor{tud-2d}{HTML}{004E73}
\definecolor{tud-3d}{HTML}{00715E}
\definecolor{tud-4d}{HTML}{6A8B22}
\definecolor{tud-5d}{HTML}{99A604}
\definecolor{tud-6d}{HTML}{AE8E00}
\definecolor{tud-7d}{HTML}{BE6F00}
\definecolor{tud-8d}{HTML}{A94913}
\definecolor{tud-9d}{HTML}{961C26}
\definecolor{tud-10d}{HTML}{732054}
\definecolor{tud-11d}{HTML}{4C226A}

\tikzset{
    possible positions/.style = {
        draw = black,
        fill = black,
        fill opacity = 0.2,
        rectangle,
        minimum size = 0.5cm
    },
    included positions/.style = {
        draw = tud-4b,
        fill = tud-4b,
        fill opacity = 0.3,
        dashed,
        rectangle,
        minimum size = 0.5cm
    },
    excluded positions/.style = {
        draw = tud-8b,
        fill = tud-8b,
        fill opacity = 0.3,
        dashed,
        rectangle,
        minimum size = 0.5cm
    },
    irrelevant positions/.style = {
        draw = tud-0b,
        fill = tud-0b,
        fill opacity = 0.3,
        dashed,
        rectangle,
        minimum size = 0.5cm
    },
    point/.style = {
        draw = black,
        fill = black,
        fill opacity = 1,
        circle,
        inner sep = 0pt,
        minimum size = 3pt
    },
    included point/.style = {
        draw = black,
        fill = tud-4b,
        fill opacity = 1,
        circle,
        inner sep = 0pt,
        minimum size = 3pt
    },
    excluded point/.style = {
        draw = black,
        fill = tud-8b,
        fill opacity = 1,
        circle,
        inner sep = 0pt,
        minimum size = 3pt
    },
    irrelevant point/.style = {
        draw = tud-0c,
        fill = white,
        fill opacity = 1,
        circle,
        inner sep = 0pt,
        minimum size = 3pt
    },
    leaf/.style = {
        draw=tud-2b,
        fill=tud-2b,
        fill opacity=0.2,
        draw opacity=0.3,
        circle,
        minimum size = 2cm
    }
}

\theoremstyle{definition}
\newtheorem{defi}{Definition}
\newtheorem{exone}{Example}

\theoremstyle{remark}
\newtheorem{rem}{Remark}
\newtheorem*{rem*}{Remark}
\theoremstyle{plain}
\newtheorem{thm}{Theorem}
\newtheorem{prop}{Proposition}
\newtheorem{lem}{Lemma}
\newtheorem{cor}{Corollary}

\let\P\relax

\DeclareMathOperator{\R}{\mathbb{R}}
\DeclareMathOperator{\N}{\mathbb{N}}

\DeclareMathOperator{\P}{\mathbb{P}}
\DeclareMathOperator{\F}{\mathcal{F}}
\DeclareMathOperator{\M}{\mathcal{M}}
\let\v\relax
\DeclareMathOperator{\v}{\mathbf{v}}
\DeclareMathOperator{\p}{\mathbf{p}}
\DeclareMathOperator{\m}{\mathbf{m}}
\DeclareMathOperator*{\argmax}{arg\,max}

\newcommand{\skp}[1]{\left\langle #1 \right\rangle}
\newcommand{\norm}[1]{\left\Vert #1 \right\Vert}
\newcommand{\abs}[1]{\left\vert #1 \right\vert}
\newcommand{\sgn}{\operatorname{sgn}}
\let\div\relax
\newcommand{\div}{\operatorname{div}}
\newcommand{\normal}[1]{\begin{pmatrix}\cos(#1)\\\sin(#1)\end{pmatrix}}

\newcommand{\ep}{\includegraphics[height=0.7\baselineskip]{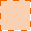}}
\newcommand{\ip}{\includegraphics[height=0.7\baselineskip]{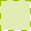}}
\newcommand{\irp}{\includegraphics[height=0.7\baselineskip]{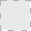}}
\newcommand{\pp}{\includegraphics[height=0.7\baselineskip]{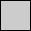}}

\SetKwInOut{Input}{input}
\SetKwInOut{Output}{output}

\usepackage[square, semicolon]{natbib}
\usepackage{url}

\logo{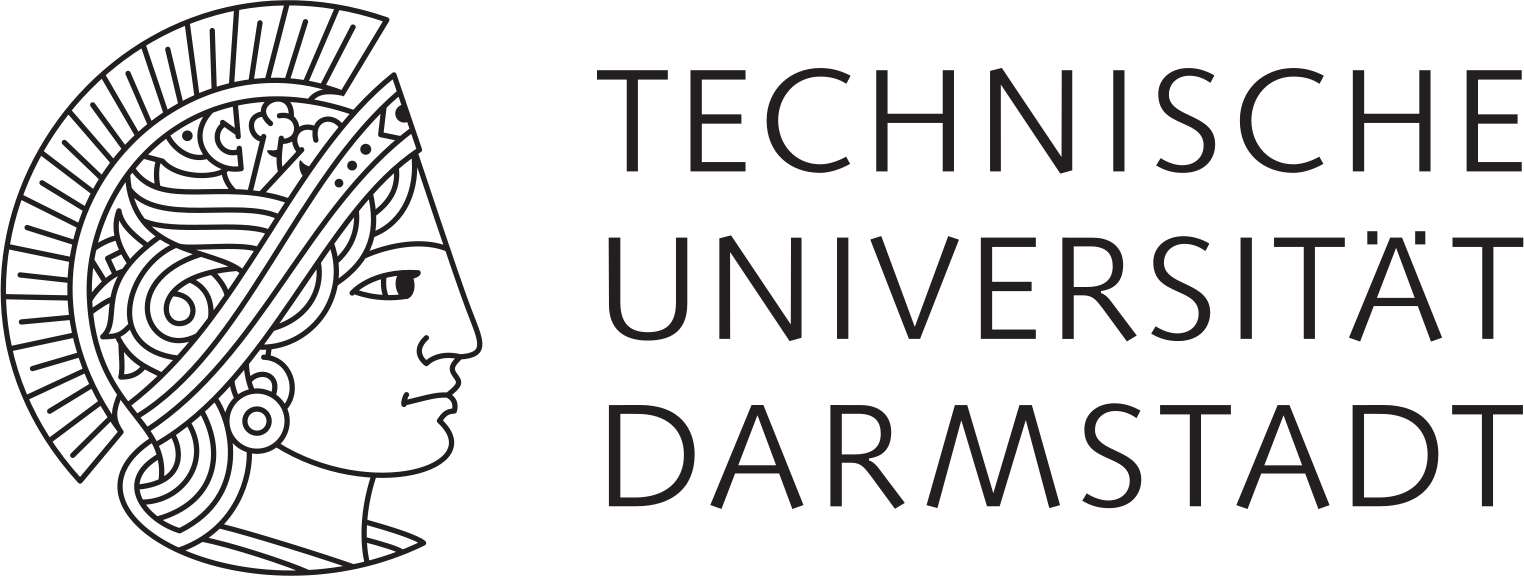}
\logo{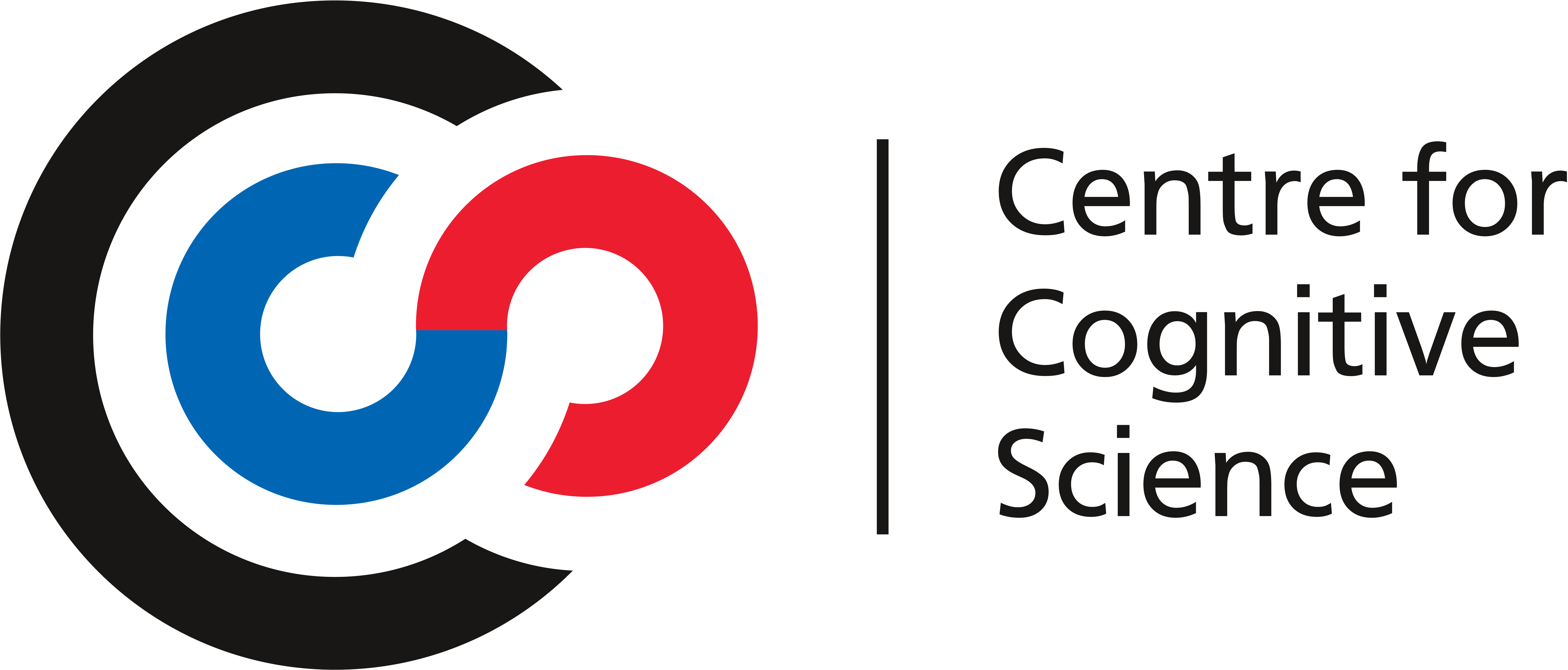}
\logo{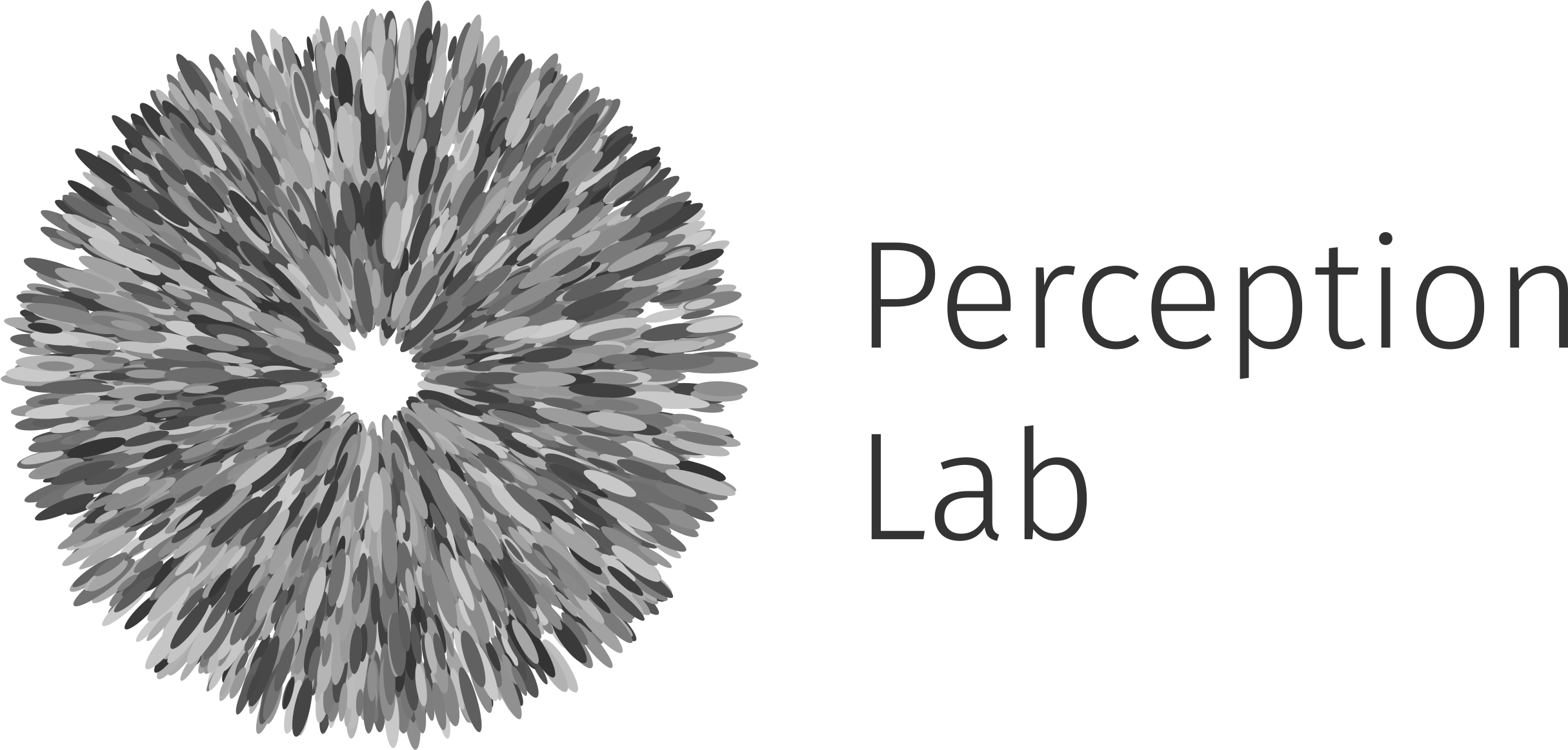}

\title{Ideal Observer for Segmentation of Dead Leaves Images}
\date{\today}

\author{
    name = {Swantje Mahncke},
    orcid = {0009-0006-2761-0038},
    email = {swantje.mahncke@tu-darmstadt.de},
    affiliation = {ccs,cmbb}
}

\author{
    name = {Malte Ott},
    orcid = {0000-0002-4962-8689},
    affiliation = math
}

\author{
    name = {Lars C. Reining},
    orcid = {0009-0007-7783-5099},
    affiliation = {ccs}
}

\author{
    name = {Thomas S.\,A. Wallis},
    orcid = {0000-0001-7431-4852},
    affiliation = {ccs,cmbb}
}

\affiliation{
    department = {Center for Mind, Brain and Behavior (CMBB)},
    organization = {Universities of},
    city = {Marburg, Giessen, and Darmstadt},
    country = {Germany},
    id = {cmbb}
}

\affiliation{
    department = {Department of Mathematics},
    organization = {TU},
    city = {Darmstadt},
    country = {Germany},
    id = {math}
}

\affiliation{
    department = {Centre for Cognitive Science, Institute of Psychology},
    organization = {TU},
    city = {Darmstadt},
    country = {Germany},
    id = {ccs}
}

\abstract{
    The visible parts of a scene are determined by occlusion among overlapping surfaces.
Here we consider ``dead leaves'' models, which replicate this by independently sampling objects (``leaves'') with position, shape, color, and texture and layering them until the image is covered.
Building on prior theory, we present a self-contained framework that rigorously defines the dead leaves model and derives an analytical Bayesian ideal observer for partitioning finite pixel sets.
The longest part of the paper spans the derivation of the prior probability, which elevates the observer beyond pixel-similarity methods by incorporating geometric information.
These computations are practical only for small pixel sets (up to 9-10 pixels).
We emphasize accessibility through step-by-step derivations, extensive visualizations, and examples.
We empirically evaluate three tractable observers (prior-only, likelihood-only, and the full ideal observer), plus a random baseline on $108$ dead leaves image datasets varying in texture intensity, leaf size, and image size.
Likelihood-only performance falls with increasing texture intensity and image size.
Prior-only performance falls with decreasing leaf size and increasing maximal image dimension.
All model-based observers strongly outperform the random baseline, and the ideal observer consistently outperforms the others by combining both information sources.
The model provides a principled upper bound on segmentation performance for limited pixel sets, enabling comparisons with human observers and algorithms.
}
\keywords{Dead leaves model, Segmentation, Bayesian ideal observer}

\begin{document}

\maketitle


If you look around right now and describe what you see, you might name things like a computer, a desk or a cup of coffee.
Objects like these and their occlusion of each other are an essential part of what humans perceive in the visual world.
The human environment is composed of different objects that are distributed in space.
Everything that humans see is a projection of these objects in the three-dimensional world onto the two-dimensional retinae.
Which parts of a scene are visible is governed by occlusion of overlapping objects.
Even though humans work with this reduced information they receive as visible input, they are good at separating their environment into its parts and judging which points of a scene belong to the same object and which do not.
Here, we refer to the assignment of each location in a scene to exactly one object as \textit{segmentation} or \textit{partition}.

It remains an open question how humans are able to achieve segmentation.
While it is clear that cues available in the real world, like binocular disparity \citep{Tsao2022} and motion \citep{Peters2021,Yao2020} are important for segmentation, here we limit our scope to the information that can be found in static 2D images.

In computer vision, the task of segmentation is often performed by first detecting the presence and overall spatial position of known objects \citep{Chen2019,Hafiz2020}.
Early algorithms for performing this object detection task were based on image features or contours \citep{Schlecht2011, Zou2023}, whereas over the past decade, the focus has shifted to deep learning techniques \citep{Hafiz2020,Sharma2022,Zou2023}.
Currently, the majority of deep neural networks perform instance segmentation in two stages, relying either on top-down (object bounding boxes) or on bottom-up (pixel-level) information in the first stage [\citealp{Gu2022}, e.\,g. \citealp{Arnab2016,He2017}].
While these techniques have increased accuracy immensely, it is also more difficult to understand how these models arrive at a specific segmentation, and how much better segmentation performance \textit{could} get, given certain constraints.

Ideal observer models offer a way to derive how an observer would perform when using all available information in an optimal way.
These models already have a long history in vision research [e.\,g. \citealp{deVries1943,Rose1948}] and have been used to model various visual tasks \citep{Burge2020, Eckstein2011, Elder2002, Geisler2009, Geisler2011, Gold2008, Hoppe2019, Najemnik2005, Neri2010, Peterson2012, Prince2007, Straub2022}.
Typically, however, these tasks use simple stimuli \citep{Burge2020}, and the models are not \textit{image-computable} in the sense that they take pixels as inputs.
While there has been progress in developing image-computable ideal observer models \citep{Burge2020}, this does not yet cover segmentation.

The advantage of simple stimuli is that they can be defined by a few parameters, while the generation process for natural scenes is highly complex and unknown.
However, we would like to be able to test and extend models of human visual processing to stimuli that are more like what humans see in the world.
While deep learning approaches to generating naturalistic images (Generative Adversarial Networks and similar) produce visually impressive results \citep{Creswell2018,Fruend2019}, due to their complexity they rely on many parameters for which it can be complicated to understand their impact on the generated image.
Therefore, we argue that developing stimulus definitions that rely on only a limited number of parameters while still capturing important structure in the world can be a fruitful direction to pursue.

The stimuli we will consider in this work can be generated with only a few parameters, while mimicking some important properties found in natural scenes.
A \textit{dead leaves} model is a probabilistic model that generates scenes by placing objects with a depth ordering such that (partial) occlusions can emerge, similarly to natural scenes.
They are so named due to their visual similarity to dead leaves on the ground in a forest (Figure~\ref{img:dead_leaves}).

\begin{figure}[ht]\centering
    \includegraphics[width=0.5\textwidth]{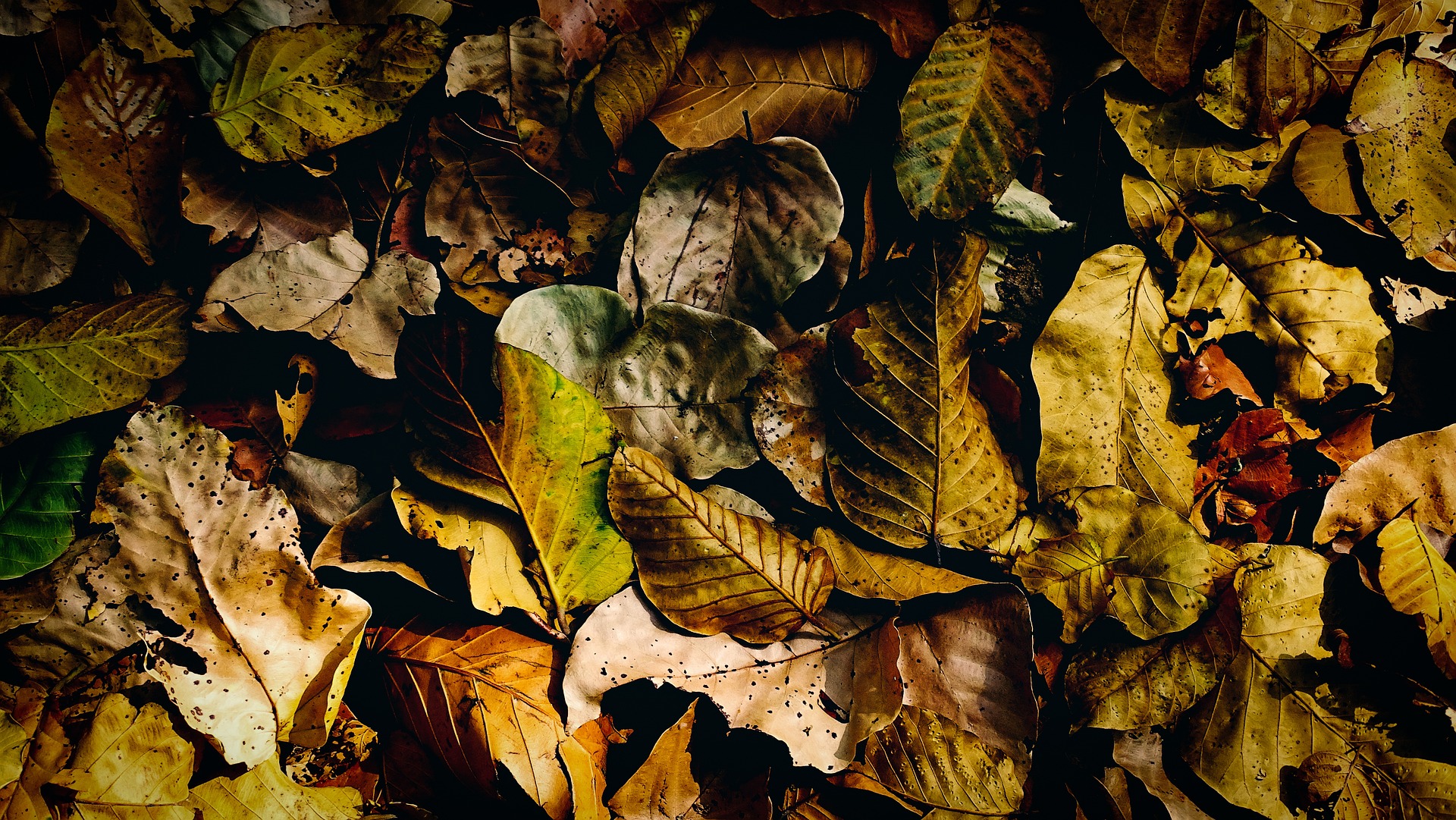}
    \caption{Dead leaves by pixabay user 7631258 (\url{https://pixabay.com/photos/dry-leaves-fallen-dry-dead-leaves-4364822/}, CC BY).}\label{img:dead_leaves}
\end{figure}

Images generated with dead leaves models can produce similar scene statistics to natural images \citep{Alvarez1999,Gousseau2007,Lee2001,Madhusudana2022}.
In particular, the scale invariance of images (because objects can take on all possible angular sizes, images should be self-similar at different scales) is present in natural images and dead leaves images alike due to the statistical independence of objects \citep{Ruderman1997,Ruderman1994}.
The occlusion of objects results in sharp edges and non-accidental features like T-junctions, which give relevant information about segmentation and figure-ground organization \citep{Walther2014,Wilder2018}.
The statistical similarity might help to extend theories derived from studies in dead leaves images to natural scenes, based on the efficient coding hypothesis \citep{Attneave1954, Barlow1961, Olshausen1996a, Olshausen1996b}, and also because segmentation can rely on statistical information, i.\,e., when considering texture as a cue \citep{Rosenholtz2014, Ariely2001, Landy2014, Whitney2014}.
For these reasons, dead leaves can be useful stimuli for psychophysical experiments \citep{Maiello2017, Taylor2015, Wallis2012} as well as for the evaluation of camera image quality \citep{Cao2010, Artmann2018, CPIQ2023}.
Additionally, more recent work in the area of computer vision has used images generated through dead leaves models as synthetic data for training neural networks, which has shown promising results compared with the use of other synthetic datasets \citep{Achddou2021,Baradad2021,Achddou2025}.

In this paper, we first present a mathematical reformulation of the work of \citet{Pitkow2010}.
Pitkow showed how to derive and compute the prior probability of a segmentation (partition) for a subset of pixels in a dead leaves model, and the pixels' resulting luminance values based on a Gaussian texture model.
Pitkow therefore derived functions that play the mathematical roles of prior, likelihood and posterior in an ideal observer paradigm.
However, Pitkow's presentation focused on the pixel-wise image statistics of dead leaves images, and how these statistics could be computed exactly from the underlying model.
Our first contribution is to reformulate Pitkow's theory, making the connection to an ideal observer for segmentation explicit and self-contained, and providing numerous examples and visualizations to support understanding.
Additionally, we make the implementation of the ideal observer available as ready-to-use Python module [doi: \href{https://doi.org/10.5281/zenodo.22209810}{$\mathrm{10.5281/zenodo.22209810}$}].
However, none of the core mathematical results we present here are novel relative to the work of Pitkow: this part of our work should be understood as a tutorial rather than a novel theoretical contribution.

Our second contribution is to empirically evaluate the ideal observer's segmentation performance across dead leaves model parameter settings, and compare it to alternative observer models.
Briefly, this shows how the contributions of the likelihood and prior change with patch size, leaf size and texture (pixel noise).
This novel analysis serves both as an additional demonstration of the application of the theory, and also to provide empirical accuracies for the ideal observer under representative dead leaves parameter settings.

\section{Dead leaves model} \label{sec: dlm}

A dead leaves model can be thought of as a model for layering objects in a visual scene.
As for any model, dead leaves models can be constructed in multiple ways.
In particular, there are different mathematical constructions which in practice can result in more or less similar model versions.
\citet{Matheron1968} has been credited with first describing a dead leaves model.
However, the version we will construct is a combination of the model descriptions given by \citet{Goussea2003}, \citet{Bordenave2006}, \citet{Lee2001}, \citet{Alvarez1999} and \citet{Madhusudana2022}.

\subsection{The model} \label{sec: the model}

The basic idea behind generating a dead leaves image starts with creating a visual environment composed of objects in space.
We can then describe how this environment projects into an observer's eyes, and we are interested in understanding how perception of the environment depends on the arrangement of different objects and their surface properties.
Subsequently, we will consider these three steps in a more detailed and formal way.

\subsubsection{Environment} \label{sec: environment}

First, we want to generate objects located somewhere in space.
We model object shapes as values of a random connected closed set.
A random connected closed set can be thought of as a random variable that takes connected closed sets as values instead of scalars.
The position of an object can be modelled through a homogeneous Poisson point process \citep[as done in][]{Goussea2003,Lee2001}, which allows a definition on the whole space, or as a uniform distribution on a bounded subset of the space \citep[as done in][]{Alvarez1999}.
While allowing positions in the whole space might be more realistic, since we will only be considering a bounded area, we will use the easier setup with a uniform position distribution.
Formal definitions and detailed descriptions about random closed sets are given in appendix~\ref{app:preliminaries}.

One problem that can occur when following this method is that objects can overlap or intersect, i.\,e., points are allocated to multiple objects.
Since this cannot happen in the real world and occlusion would not be well-defined in this case, we want to avoid an intersection of objects.
To fix this problem, we will not sample locations from a 3-dimensional space, but only from the plane.
The order in depth will then be determined by the sample order from front to back, similar to the construction of \citet{Goussea2003} and \citet{Lee2001}.
For three-dimensional objects, this does not necessarily prevent intersection of objects.
Here, we are foremost interested in the visible parts of an object, which can also be represented by a planar object.
To prevent object intersections, we therefore will only consider random sets in $\R^2$.
Taken together, these constraints yield the following construction of the environment.

\begin{defi}[Environment]\label{def:environment}
    Let $I=\N$ be an index set and let $B$ be a bounded area in $\R^2$ (typically simply connected).
    We define \emph{random positions} as random variables $(P_i)_{i\in I}$ which are independently uniformly distributed on $B$.
    The index $i$ defines the distance to the plane.
    Let $(X_i)_{i\in I}$ be independent and identically distributed (i.\,i.\,d.) random connected closed sets of the plane, independent of $P_i$.
    We call $X_i$ a \emph{random shape} and the random connected closed set
    \begin{align*}
        L_i= P_i + X_i
    \end{align*}
    is then a \emph{random leaf}.
    We call the set of leaves $E=\{L_i\}_{i\in I}$ the \emph{random environment}.
\end{defi}

In other words, to generate one sample of a visual environment we sample a series of random two-dimensional objects with shapes $(X_i)_{i\in I}$ which we uniformly distribute in the three-dimensional space at positions $(P_i)_{i\in I}$ (defined with horizontal and vertical coordinates).\footnote{
    There are options to extend this generation process to 3D objects.
    When sampling shapes and positions in three dimensions we could, for example, resample every time we get a sample which includes points that already belong to an object.
    It is, however, unclear how this would affect the statistics of the scene.
    Alternatively, we could consider the projection of a three-dimensional object to the plane and then placing these projections in space to generate the environment.
    This would add a step to the generation process for which the consequences are also not clear.
}

A simple example of an environment could look as follows.

\begin{exone}\label{ex1: environment}
    Let $B = [-2.5,4.5]^2$.
    We consider the first $i=1,\dots,3$ samples from a random environment.
    A series of locations sampled from $P_i$ could be $\{(1.7,1.7),(0.3,0.6),(2.3,0.6)\}$. These locations in $B$ can be visualized as:

    \begin{center}
        \begin{tikzpicture}[scale = 0.4]
            \foreach \x in {-2,...,4}{
                    \draw[tud-0a] (-2.5,{\x}) node[anchor=east] {\small $\x$};
                    \draw[tud-0a] ({\x},-2.5) node[anchor=north] {\small $\x$};
                }
            \draw[tud-0a, very thin] (-2.5,-2.5) grid (4.5,4.5);
            \draw[tud-0c, very thin] (0,-2.5) -- (0,4.5);
            \draw[tud-0c, very thin] (-2.5,0) -- (4.5,0);

            \foreach \x in {(1.7,1.7),(0.3,0.6),(2.3,0.6)}{
                    \node[point] at \x {};
                }
            \draw[tud-0c, very thin] (-2.5,-2.5) rectangle (4.5,4.5);

        \end{tikzpicture}
    \end{center}

    Additionally, consider the following random set samples (for simplicity, we only sample circles):
    \begin{align*}
        \mathbf{x}_1 & = \{(x,y) \in \R^2 \mid \norm{(x,y)} \leq 1.1\} \\
        \mathbf{x}_2 & = \{(x,y) \in \R^2 \mid \norm{(x,y)} \leq 1.5\} \\
        \mathbf{x}_3 & = \{(x,y) \in \R^2 \mid \norm{(x,y)} \leq 1.8\} \\
    \end{align*}

    Together with the positions we sampled, the environment is composed of three leaves:

    \begin{center}
        \begin{tikzpicture}[scale = 0.4]
            \draw[tud-0a, very thin, dashed] (-2.5,-2.5) -- (17.5,1.5);
            \foreach \x in {0,1,2}{
                    \draw[tud-0c, very thin, fill=white] ({-2.5+\x*10},{-2.5+\x*2}) rectangle +(7,7);
                }

            \draw[fill=tud-0b, draw=tud-0c] (1.7,1.7) circle (1.1);
            \draw[fill=tud-0b, draw=tud-0c] ({0.3+10},{0.6+2}) circle (1.5);
            \draw[fill=tud-0b, draw=tud-0c] ({2.3+20},{0.6+4}) circle (1.8);

            \foreach \x in {0,1,2}{
                    \draw[tud-0a, very thin] ({-2.5+\x*10},{-2.5+\x*2}) grid +(7,7);
                    \draw[tud-0c, very thin] ({\x*10},{-2.5+\x*2}) -- +(0,7);
                    \draw[tud-0c, very thin] ({-2.5+\x*10},{\x*2}) -- +(7,0);
                }

            \node[point] at (1.7,1.7) {};
            \node[point] at ({0.3+10},{0.6+2}) {};
            \node[point] at ({2.3+20},{0.6+4}) {};

            \foreach \x in {-2,...,4}{
                    \draw[tud-0a] (-2.5,{\x}) node[anchor=east] {\small $\x$};
                    \draw[tud-0a] ({\x},-2.5) node[anchor=north] {\small $\x$};
                }

            \foreach \i in {0,1,2}{
                    \draw[tud-0a] ({5.5+10*\i},{-3+2*\i}) node {\small $i=\i$};
                }

            \draw[tud-0a, very thin, dashed] (-2.5,4.5) -- (17.5,8.5);
            \draw[tud-0a, very thin, dashed] (4.5,-2.5) -- (24.5,1.5);
            \draw[tud-0a, very thin, dashed] (4.5,4.5) -- (24.5,8.5);
        \end{tikzpicture}
    \end{center}
\end{exone}

The region in space over which positions are distributed ($B$) can always be chosen in a way that it covers the area visible to an observer.
In this manner, we can, for example, generate the environment depicted in Figure~\ref{img:dead leaves model 3d}.

\begin{figure}[ht]\centering
    \includegraphics[width=0.4\textwidth]{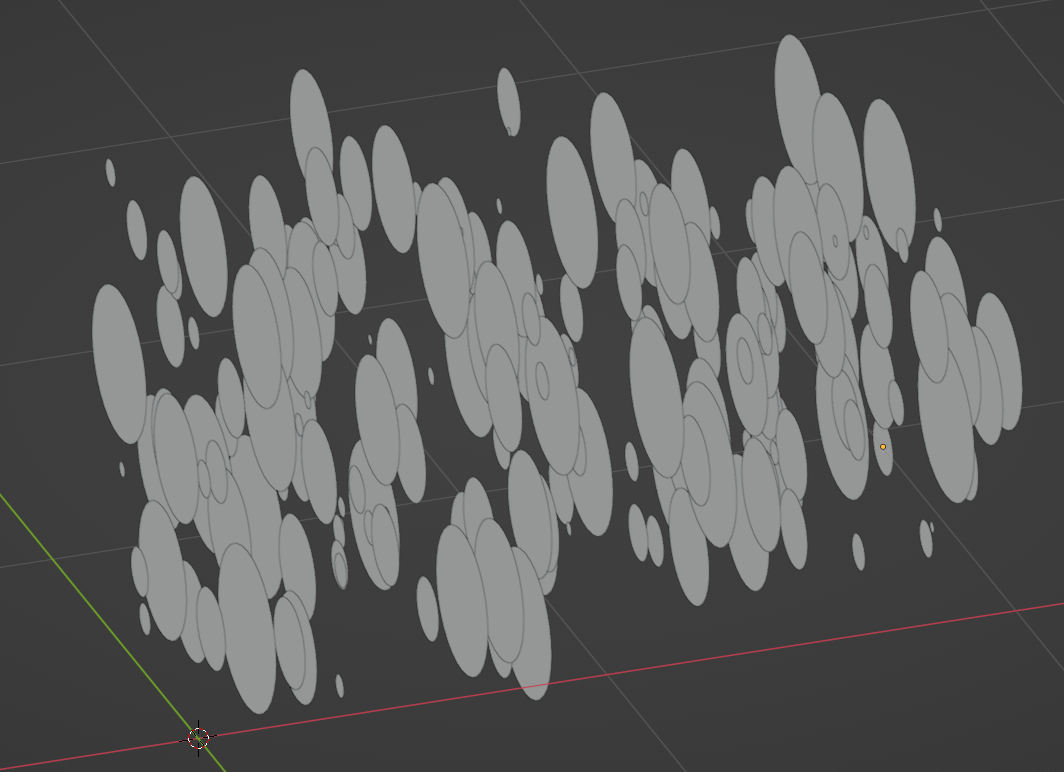}
    \caption{
        Example of a random environment using random circles as objects viewed from the side.
        The leaves are stacked in depth along the $z$-axis (red). The $x$-axes (yellow) and $y$-axes (not shown) define the horizontal and vertical coordinates of the plane, respectively.
    }\label{img:dead leaves model 3d}
\end{figure}

The next step is to determine which parts of the environment are visible to an observer.

\subsubsection{Projection} \label{sec: projection}

The visible part of the environment is what remains when projecting the environment onto the eye with occlusion by the leaves.
Since transparency would complicate the occlusion of objects during this projection, we limit ourselves to opaque objects.
A model construction including transparency was proposed by \citet{Mumford2001}.
Realistically, the projection of the environment is an orthographic projection onto a two-dimensional manifold as implemented by \citet{Mumford2001} (Figure~\ref{img: dlm projection perspective}).
We will, however, follow the more common approach and use an orthographic projection of the environment onto the plane (Figure~\ref{img: dlm projection orthographic}) \citep[compare to][]{Goussea2003,Lee2001,Alvarez1999}, since this facilitates the computation.

\begin{figure}[ht]\centering
    \adjustbox{minipage=1.3em,raise=15mm}{\subcaption{}\label{img: dlm projection orthographic}}
    \begin{subfigure}[t]{0.3\textwidth}
        \includegraphics[width=\linewidth]{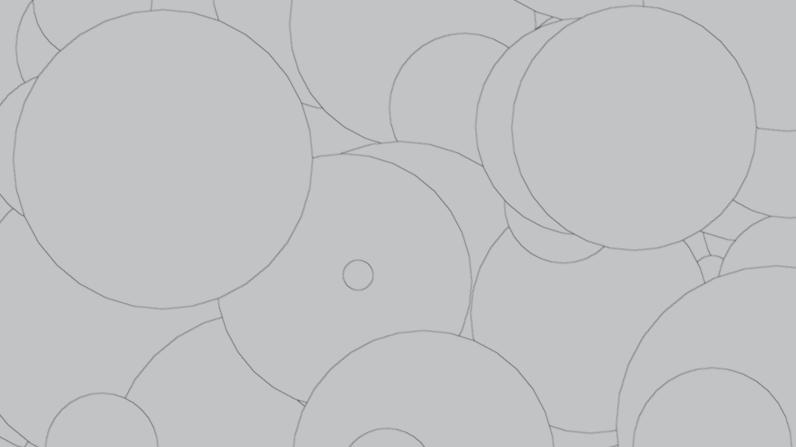}
    \end{subfigure}
    \adjustbox{minipage=1.3em,raise=15mm}{\subcaption{}\label{img: dlm projection perspective}}
    \begin{subfigure}[t]{0.3\textwidth}
        \includegraphics[width=\linewidth]{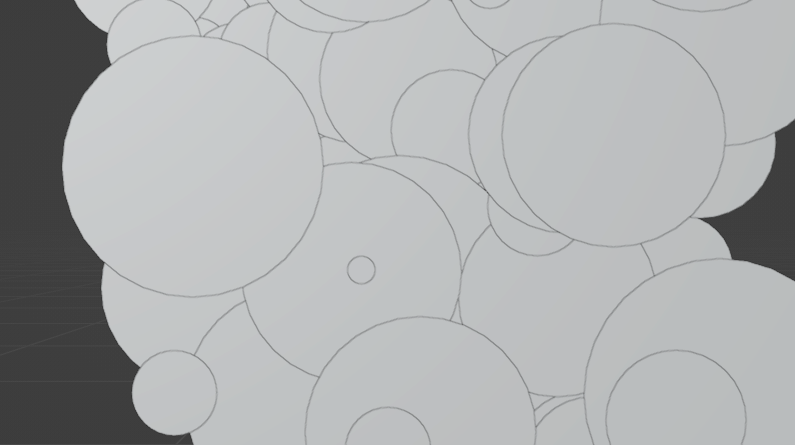}
    \end{subfigure}
    \caption{Projection of the environment from Figure~\ref{img:dead leaves model 3d} to a plane (Orthographic projection, \subref{img: dlm projection orthographic}) and a two-dimensional manifold (Projection with perspective, \subref{img: dlm projection perspective}).}\label{img:dead leaves model projection}
\end{figure}

To execute the projection, we first need to identify the visible parts that will be projected onto the plane.
Under the assumption that an observer is located at depth zero and the index $i$ of a leaf indicates its depth from the observer, we can determine which parts of the environment are visible to the observer by removing all parts of a leaf that are occluded by other leaves with smaller depth.
Additionally, since parts of the environment may fall outside the observer's field of view and thus be only partially visible, we explicitly model the visible area as a compact subset $A$ of the object position area $B$.
This allows for objects whose positions are outside the field of view but whose shapes reach into the visible area.

\begin{defi}[Visible parts of leaves]\label{def:visible parts}
    Let the \emph{visible area} of the environment be given by a compact set  $A \subset B \subset \R^2$.
    The \emph{visible part} of a random leaf $L_i$ is then given by
    \begin{align*}
        V_{i} = \left(L_i \setminus \left(\bigcup_{j < i} L_j\right)\right) \cap A
    \end{align*}
    i.\,e., the elements of $L_i$ which are not occluded by another leaf and fall into the visible area.
    The complement $L_i\setminus V_{i}$ is then the \emph{occluded part} of a leaf.
    A leaf $L_i$ with $V_{i}=\emptyset$ is called \emph{invisible}.
\end{defi}

\begin{rem}
    Since all leaves are sampled from random closed sets, the boundary between two leaves always belongs to the occluding leaf.
\end{rem}

By definition of the environment, the series of random positions $P_i$ only includes locations in $B$.
Consequently, the image of this projection mapping is also bounded, and we can consider the area that is covered by the leaves.

\begin{defi}[Random Dead leaves partition]\label{def:partition}
    Let $E=\{L_i\}_{i\in I}$ be a random environment with $B$ and $I$ as in definition~\ref{def:environment}, and let $V_{i}\subset \R^2$ be the visible leaf parts of the environment.
    Given a compact visible area $A\subset\R^2$, we call the family of random leaves $$M_A=\{V_{i}\}_{i\in I}$$ a \emph{random dead leaves partition} of $A$ if
    \begin{align*}
        \P_{(V_i)_{i\in I}}\left( \left(\v_i\right)_{i\in I} \mid \bigcup_{i\in I} \v_{i} = A \right) = 1
    \end{align*}
    i.\,e., the orthographic projection of the visible leaf parts covers $A$ almost surely.
    By construction, $M_A$ is a family of random sets and hence a random variable taking values in
    \begin{align*}
        \M_A = \{\m_A \mid \m_A\text{ partition of }A\},
    \end{align*}
    where $\M_A$ is the set of all possible partitions of $A$ and can be thought of as the set of possible outcomes of the random dead leaves partition $M_A$.
\end{defi}

By construction, if $B$ is large enough, the projection of the leaves can tile any bounded subset of $\R^2$ \cite[Lemma~2]{Bordenave2006}, yielding a dead leaves partition of the subset.
In particular, if we sample a large enough number of leaves, we can fill any finite visible area.
We can therefore abort the sampling process once all points in our visible area are covered in leaves.

By definition, visible parts of different leaves in the model cannot overlap, i.\,e., $V_{i}\cap V_{j}=\emptyset$ for $i\neq j$.
Therefore, by construction, every dead leaves partition is a partition or segmentation of $A$.
Since we are only concerned with visible leaves, we can renumber the leaves in a model such that they are numbered according to depth without skipping indices, with $V_1$ being the topmost leaf.

\begin{rem*}[Discretization]
    When considering visible areas in $\R^2$, the set $\M_A$ can be uncountable.
    To allow computations on $\M_A$, we can reduce the dead leaves partition by considering the partition of a finite subset $a\subset A$.
    This can be thought of as taking a number of measurements from the continuous environment, akin to photoreceptors when taking a photograph or viewing the scene.
    As a consequence, we consider a finite set $a\subset A$, e.\,g. $a = A \cap \mathbb{Z}^2$, and the partition $M_a$.
    Then, there is only a finite number of points to be segmented by the dead leaves and $M_a$ can have at most $\vert a\vert$ visible leaves, i.\,e., each point belongs to a different leaf.
    With this construction, the set $\M_a$ is also finite although possibly very large.
    For a given number of pixels $n$, \emph{Bell's number}, defined as
    \begin{align*}
        B_{n+1} = \sum_{k=0}^n \binom{n}{k} B_k, \qquad B_0=1,
    \end{align*}
    gives the number of all possible partitions of a set and can function as an approximation for the size of $\M_a$. For example, an image with a width and height of 10 pixels has $n=100$ pixels in total.
    Its Bell number $B_{100}$ is then roughly $4.7 \cdot 10^{115}$.
    Thus, even for relatively small images there are already numerous possible dead leaves partitions.
\end{rem*}

Whether we consider the discrete or continuous case, based on a dead leaves partition, we can always define the membership function as follows.

\begin{defi}[Membership function]\label{def:membership function}
    We define the \emph{membership function} of a dead leaves partition $\m_A=\{\v_{i}\}_{i\in I}$ as the mapping from points in $A$ to the leaf of which they are an element.
    \begin{align*}
        \mu \colon A \to I, x\mapsto i, \quad \text{s.\,t. } x\in \v_i.
    \end{align*}
\end{defi}

By definition, the image of the membership function $\operatorname{Im}(\mu)$ is the index set of visible leaves in the model.
Continuing with our previous example, we get the following process.

\begin{exone} \label{ex1: projection}
    Let $A=[0,2]^2$ be a visible area with pixels $a=\{0,1,2\}^2$, which, with $9$ pixels, already has a Bell number of $B_9 = 21,147$.
    Following definition~\ref{def:visible parts} of visible parts, through orthographic projection and excluding parts of objects that fall outside $a$, we get
    \begin{align*}
        \v_1 & = \{(1,2),(2,1),(1,1),(2,2)\}, \\
        \v_2 & = \{(0,0),(0,1),(1,0),(0,2)\}, \\
        \v_3 & = \{(2,0)\},
    \end{align*}
    with $\m_a=\{\v_1,\v_2,\v_3\}$ being a dead leaves partition of $a$.
    The membership function $\mu$ then has the following image

    \begin{center}
        \raisebox{-0.5\height}{
            \begin{tikzpicture}[scale = 0.5]
                \begin{scope}
                    \clip (0,0) rectangle (2,2);
                    \draw[fill=tud-0b, draw=tud-0c] (2.3,0.6) circle (1.8);
                    \draw[fill=tud-0b, draw=tud-0c] (0.3,0.6) circle (1.5);
                    \draw[fill=tud-0b, draw=tud-0c] (1.7,1.7) circle (1.1);
                \end{scope}
                \foreach \x in {0,1,2}{
                        \draw[tud-0a] (-1,{\x}) node {$\x$};
                        \draw[tud-0a] ({\x},-1) node {$\x$};
                    }
                \draw[tud-0a, very thin] (0,0) grid (2,2);
            \end{tikzpicture}}
        \quad $\overset{\mu(a)}{\longrightarrow}$ \quad
        \raisebox{-0.5\height}{
            \begin{tikzpicture}[scale = 0.5]
                \foreach \x in {0,1,2}{
                        \draw[tud-0a] (-1,{\x}) node {$\x$};
                        \draw[tud-0a] ({\x},-1) node {$\x$};
                    }

                \draw (0,0) node {$2$};
                \draw (0,1) node {$2$};
                \draw (1,0) node {$2$};
                \draw (1,1) node {$1$};
                \draw (2,0) node {$3$};
                \draw (2,1) node {$1$};
                \draw (2,2) node {$1$};
                \draw (0,2) node {$2$};
                \draw (1,2) node {$1$};
            \end{tikzpicture}}
    \end{center}
\end{exone}

The final step is to add color and texture to the leaves.

\subsubsection{Color \& texture} \label{sec: color texture}

Objects in the world can be made from different materials, resulting in varying color, texture, and opacity of their surfaces.
As noted before, we will only consider opaque objects.
Each point in our model should therefore appear to have some kind of color that depends on the base color and texture of the object it belongs to.
In the literature, various ways have been used to model the surface properties of leaves.
The simplest approach, as employed by \citet{Goussea2003} and \citet{Lee2001}, is to disregard texture and just randomly select a color for each leaf (Figure~\ref{img:dlm random color}).
\citet{Madhusudana2022} made this approach slightly more realistic by sampling colors from a color histogram of natural images (Figure~\ref{img:dlm sample color}).
Approaches to add texture get more complex.
\citet{Pitkow2010} generated texture by adding Gaussian noise to the color of every pixel of the image (Figure~\ref{img:dlm gaussian noise}) resulting in a global texture.
More realistic textures were used by \citet{Madhusudana2022}, who randomly chose a texture from the Brodatz texture database \cite{Brodatz1999} for each leaf and added it to the base color with alpha blending (Figure~\ref{img:dlm Brodatz texture}).
In recent work, \citet{Achddou2025} generated color and texture by combining sinusoidal fields of different frequencies between random colors and warping the resulting pattern, as well as blending it into different noise patterns.

For computational reasons, here we sample color and a global additive pixel-wise texture from a known distribution to generate what we will call a dead leaves image.

\begin{defi}[Dead leaves image]\label{def:image}
    For a given dead leaves partition $\m_A$ of the visible area $A$, let $(C_i)_{i\in I}$ be i.\,i.\,d. random vectors in some color space and $(T_{i_x})_{i\in I,\,x\in \v_i}$ i.\,i.\,d. random vectors in some texture space of the same dimension.
    The \emph{random dead leaves image} $S$ is the coloration according to the membership function, i.\,e.,
    \begin{align*}
        S_x = C_{\mu(x)} + T_{\mu(x)_x}
    \end{align*}
    where $C_i$ is the leaf's base color and $T_{i_x}$ generates some per leaf pixel-based texture.
    If $T_i$ is the same for all leaves, we might instead write $S_x = C_{\mu(x)} + T_x$ with global texture $(T_x)_{x\in A}$.
\end{defi}

Hence, with this definition, we generate a dead leaves image~$\mathbf{s}$ by sampling a color~$\mathbf{c_i}$ for each leaf in a given dead leaves partition and varying the specific pixel values by adding random noise~$\mathbf{t}_x$.

\begin{rem*}
    The color space could be continuous, such as visible light in nanometers.
    For generating an image, it is more realistic to have a discrete color space like 8-bit RGB.
    Since in this setting the color space is bounded, and we add random texture, the texture distribution needs to be chosen such that reaching the bounds of our color space when adding texture is not very likely.
    If it happens, we will cap the result at the bounds.
    Depending on the chosen color model, this will result in slight biases towards specific colors.
\end{rem*}

\begin{figure}[ht]\centering
    \begin{subfigure}[t]{0.24\textwidth}\centering
        \caption{}\label{img:dlm random color}
        \includegraphics[width=\linewidth]{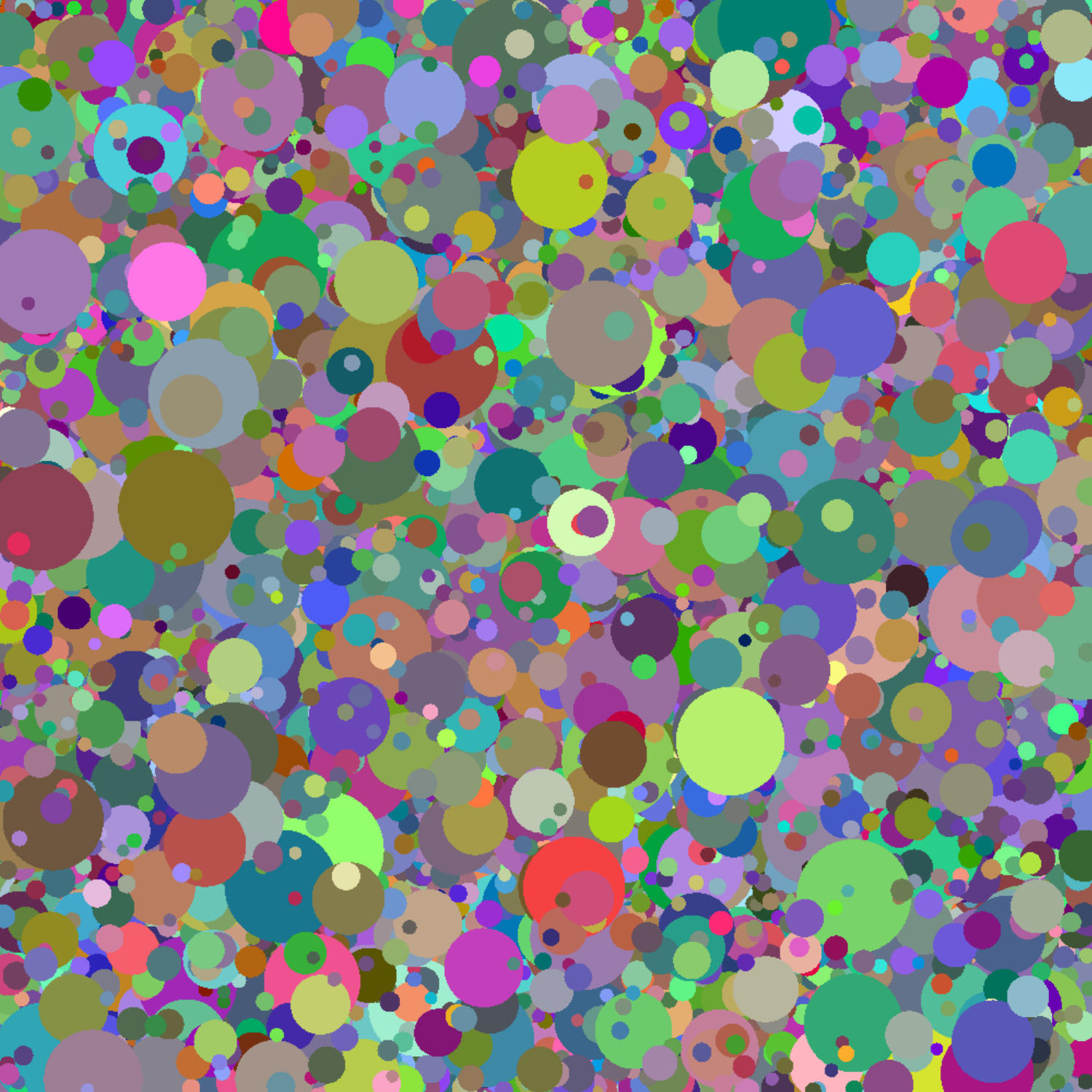}
    \end{subfigure}
    \hfill
    \begin{subfigure}[t]{0.24\textwidth}\centering
        \caption{}\label{img:dlm gaussian noise}
        \includegraphics[width=\linewidth]{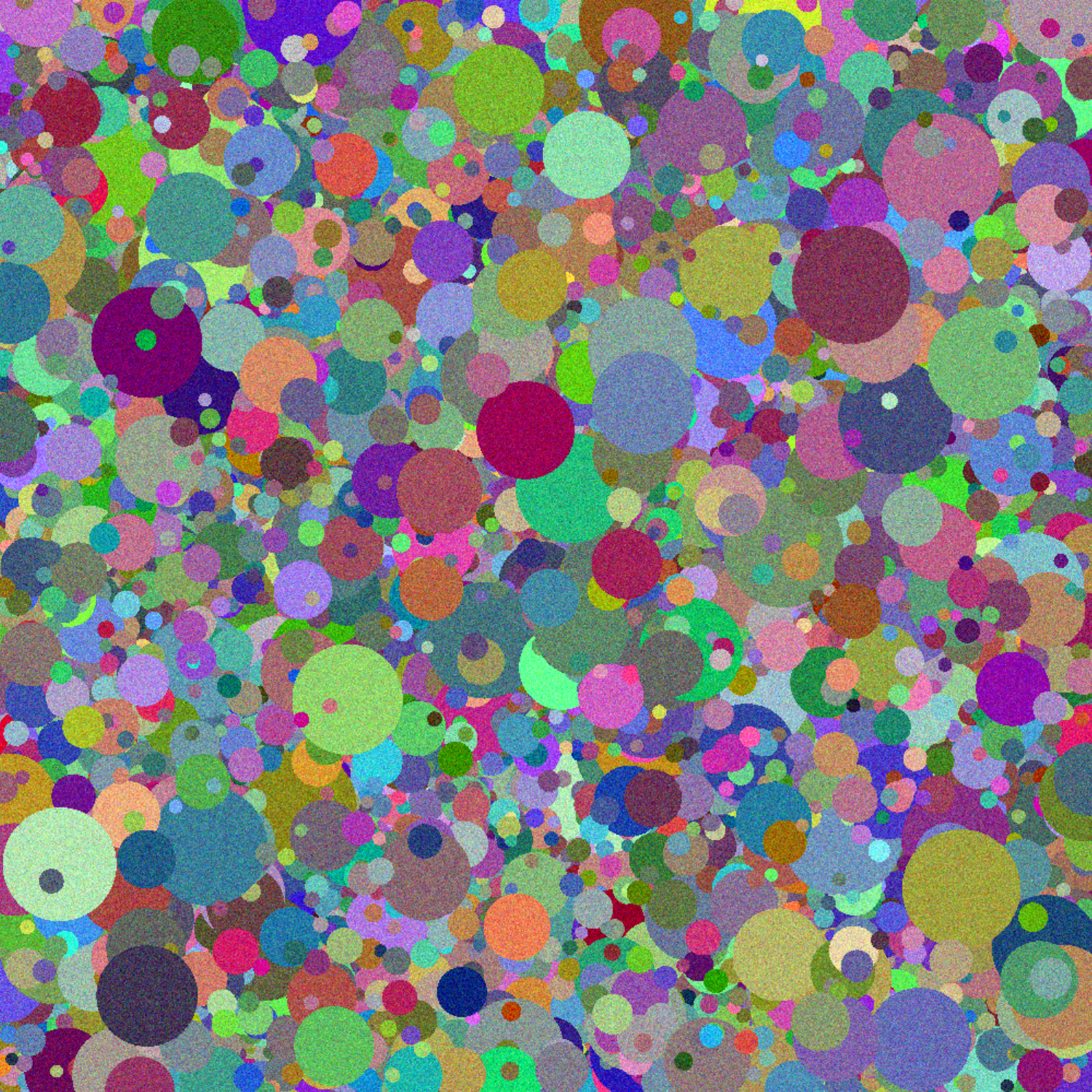}
    \end{subfigure}
    \hfill
    \begin{subfigure}[t]{0.24\textwidth}\centering
        \caption{}\label{img:dlm sample color}
        \includegraphics[width=\linewidth]{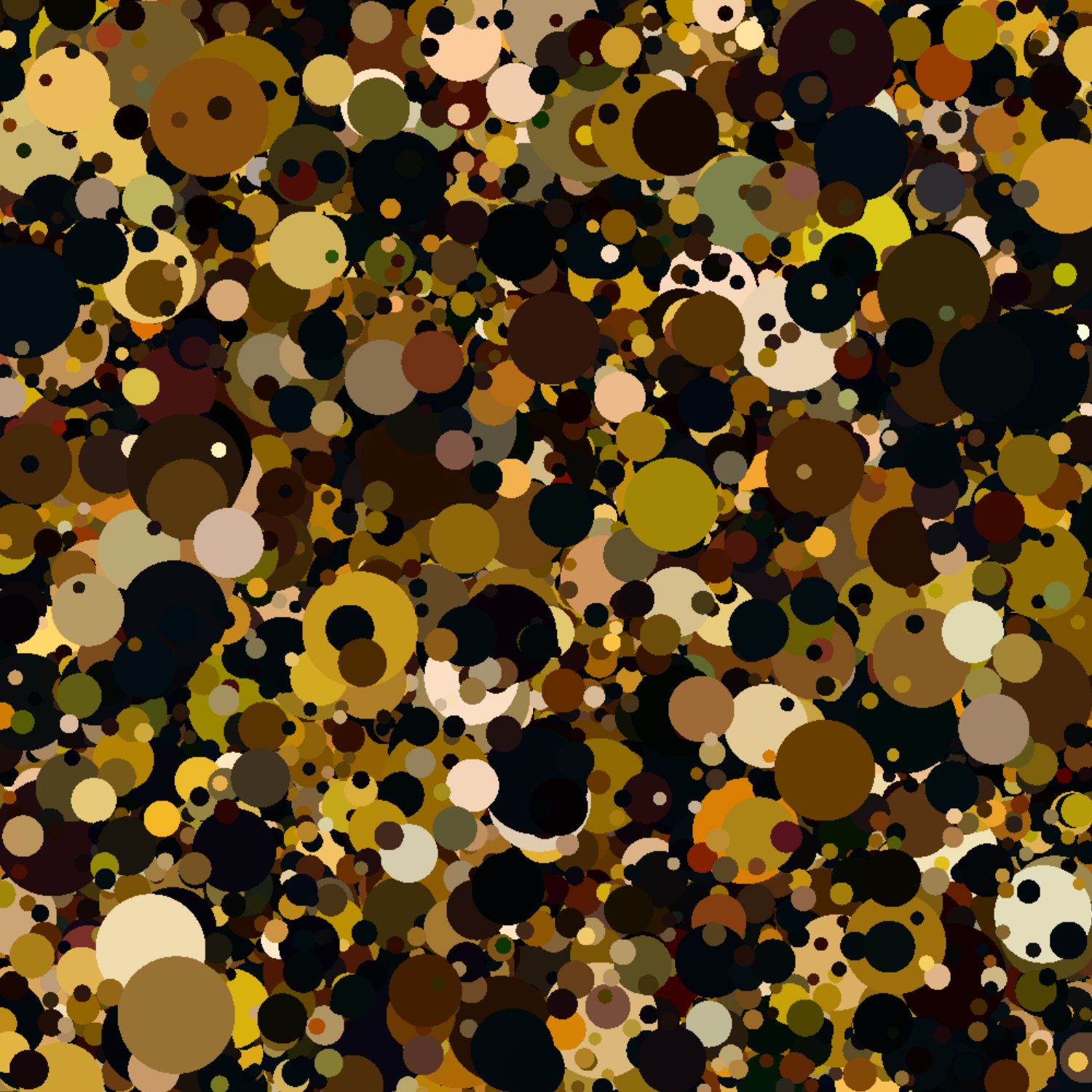}
    \end{subfigure}
    \hfill
    \begin{subfigure}[t]{0.24\textwidth}\centering
        \caption{}\label{img:dlm Brodatz texture}
        \includegraphics[width=\linewidth]{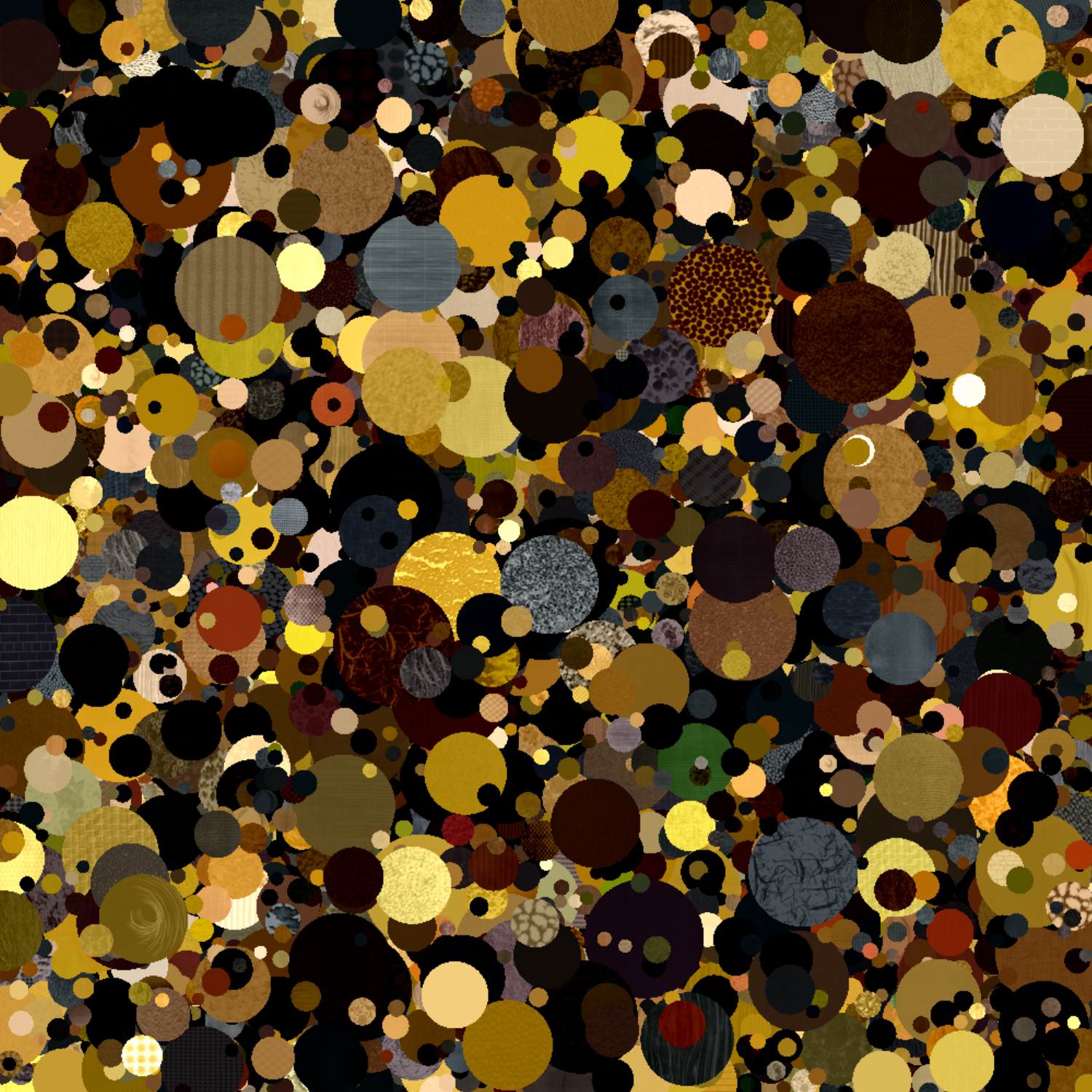}
    \end{subfigure}
    \caption{Dead leaves scene with
        (\subref{img:dlm random color}) random RGB color and
        (\subref{img:dlm gaussian noise}) Gaussian texture or with
        (\subref{img:dlm sample color}) color sampled from the histogram of Figure~\ref{img:dead_leaves} and
        (\subref{img:dlm Brodatz texture}) random Brodatz textures~\citep{Brodatz1999} blended onto leaves.}\label{fig: color texture}
\end{figure}

\begin{exone} \label{ex1: color texture}
    Continuing our previous three-leaf example, and sampling HSV colors from $[0,1]^3$, we might get the following colors:
    \definecolor{c1}{rgb}{0.55,0.215,0.6}
    \definecolor{c2}{rgb}{0.26,0.34,0.58}
    \definecolor{c3}{rgb}{0.27,0.76,0.72}
    \begin{align*}
        \mathbf{c}_1 & = \textcolor{c1}{\blacksquare}, &
        \mathbf{c}_2 & = \textcolor{c2}{\blacksquare}, &
        \mathbf{c}_3 & = \textcolor{c3}{\blacksquare}.
    \end{align*}
    Adding these colors yields the following image (left image) with the discrete points colored according to their leaf membership (right image):
    \begin{center}
        \begin{tikzpicture}[scale = 0.5]
            \begin{scope}
                \clip (0,0) rectangle (2,2);
                \draw[fill=c3, draw=c3] (2.3,0.6) circle (1.8);
                \draw[fill=c2, draw=c2] (0.3,0.6) circle (1.5);
                \draw[fill=c1, draw=c1] (1.7,1.7) circle (1.1);
            \end{scope}
            \foreach \x in {0,1,2}{
                    \draw[tud-0a] (-1,{\x}) node {$\x$};
                    \draw[tud-0a] ({\x},-1) node {$\x$};
                }
            \draw[tud-0a, very thin] (0,0) grid (2,2);
        \end{tikzpicture}
        \qquad
        \begin{tikzpicture}[scale = 0.5]
            \foreach \x in {0,1,2}{
                    \draw[tud-0a] (-1,{\x}) node {$\x$};
                    \draw[tud-0a] ({\x},-1) node {$\x$};
                }
            \draw[tud-0a, very thin] (0,0) grid (2,2);
            \foreach \x in {(0,0),(1,0),(0,1),(0,2)}{
                    \node[point, fill=c2, draw=c2] at \x {};
                }

            \foreach \x in {(1,1),(1,2),(2,1),(2,2)}{
                    \node[point, fill=c1, draw=c1] at \x {};
                }

            \foreach \x in {(2,0)}{
                    \node[point, fill=c3, draw=c3] at \x {};
                }
        \end{tikzpicture}
    \end{center}
    If we fill a pixel with the corresponding color (for better visualization; left image) and add Gaussian texture to each of the points (right image), we get the following images:
    \begin{center}
        \begin{tikzpicture}[scale = 0.5]
            \foreach \x in {(0,0),(1,0),(0,1),(0,2)}{
                    \fill[c2] \x rectangle +(1,1);
                }

            \foreach \x in {(1,1),(1,2),(2,1),(2,2)}{
                    \fill[c1] \x rectangle +(1,1);
                }

            \foreach \x in {(2,0)}{
                    \fill[c3] \x rectangle +(1,1);
                }

            \foreach \x in {0,1,2}{
                    \draw[tud-0a] (-0.5,{\x+0.5}) node {$\x$};
                    \draw[tud-0a] ({\x+0.5},-0.5) node {$\x$};
                }
        \end{tikzpicture}
        \qquad
        \definecolor{t1}{rgb}{0.2577, 0.3319, 0.5822}
        \definecolor{t2}{rgb}{0.2571, 0.3444, 0.5866}
        \definecolor{t3}{rgb}{0.2714, 0.7642, 0.7442}
        \definecolor{t4}{rgb}{0.2688, 0.3927, 0.5983}
        \definecolor{t5}{rgb}{0.5611, 0.2157, 0.6222}
        \definecolor{t6}{rgb}{0.5631, 0.2172, 0.5969}
        \definecolor{t7}{rgb}{0.2536, 0.3439, 0.5845}
        \definecolor{t8}{rgb}{0.5559, 0.2208, 0.6352}
        \definecolor{t9}{rgb}{0.5406, 0.2098, 0.6082}
        \begin{tikzpicture}[scale = 0.5]
            \fill[t1] (0,0) rectangle +(1,1);
            \fill[t2] (1,0) rectangle +(1,1);
            \fill[t3] (2,0) rectangle +(1,1);
            \fill[t4] (0,1) rectangle +(1,1);
            \fill[t5] (1,1) rectangle +(1,1);
            \fill[t6] (2,1) rectangle +(1,1);
            \fill[t7] (0,2) rectangle +(1,1);
            \fill[t8] (1,2) rectangle +(1,1);
            \fill[t9] (2,2) rectangle +(1,1);
            \foreach \x in {0,1,2}{
                    \draw[tud-0a] (-0.5,{\x+0.5}) node {$\x$};
                    \draw[tud-0a] ({\x+0.5},-0.5) node {$\x$};
                }
        \end{tikzpicture}
    \end{center}
\end{exone}

\subsubsection{Model}

Overall, this model has multiple uncertainties: object shape, position, color, and texture.
We can combine these to describe the full randomness of the process.
\begin{defi}[Dead leaves model]\label{def:model}
    Let $\mathcal{U}(B)$ be the uniform distribution over $B$ and $\mathcal{X}$ be a distribution of random shapes in $\R^2$ (as in definition~\ref{def:environment}).
    Let additionally $\mathcal{C}$ and $\mathcal{T}$ be distributions of color and texture.
    We then call the family of distributions
    \begin{align*}
        \M = \{\mathcal{U},\mathcal{X},\mathcal{C},\mathcal{T}\}
    \end{align*}
    a \emph{dead leaves model}.
\end{defi}

As mentioned, the distributions are highly complex and unknown in the real world.
Therefore, we need to use some sort of simplified approximation.
A common measure for the quality of such an approximation is scene statistics.
The closer the statistics of the generated scene are to natural scene statistics the better.

\subsection{Scene statistics} \label{sec: scenes statistics}

The statistics of a scene, such as its contrast distribution, largely depend on the size of objects and how they occlude each other, which depends on the distribution of object positions.
\citet{Lee2001} studied statistical properties of dead leaves images.
They found that dead leaves images share statistical properties with natural images when two conditions are fulfilled:
first, the leaves have uniformly distributed positions and second, they are circles with power-law distributed radii ($f(r)\propto r^{-3}$).
\citet{Goussea2003} additionally give a theoretical reasoning for this choice of the radius distribution.
In particular, this power-law distribution preserves the scale invariance observed in natural scenes \citep{Lee2001,Goussea2003}.

The distribution of color in natural images can vary a lot depending on the image \citep{Webster1997}.
For example, the seasons influence the average color of a scene as well as the overall distribution \citep{Webster2007}.
Therefore, there is no single simple distribution to generate a dead leaves image with colors similar to natural scenes.

Finally, regarding texture, \citet{Madhusudana2022} and \citet{Achddou2025} found that scene statistics can be replicated even closer with dead leaves models, if the leaves are textured rather than only a solid color.
Considering these results, we arrive at the subsequent implementation.

\subsection{Implementation} \label{sec: dlm implementation}

Algorithm~\ref{alg:dlm} describes the generation of a dead leaves image in two steps.
First, the dead leaves partition is generated and then color and texture are added to the resulting leaves.
Following the argument on scene statistics, each leaf has the shape of a circle and radii are sampled from a distribution with
\begin{align*}
    f_X(r) \propto r^{-3}.
\end{align*}
In order to obtain visually interesting images, minimum $r_{\min}$ and maximum $r_{\max}$ values for the sampled radii are included\footnote{Without radius bounds we would obtain images containing only one leaf or images where each pixel is one individual leaf~\citep{Lee2001}}.
To obtain a probability density function, i.\,e., $\int_{-\infty}^{\infty} f_X(r) \mathbf{d}r = 1$, we also add a normalizing factor as done by \citet{Pitkow2010}.
We will refer to the resulting distribution $\mathcal{P}(r_{\min}, r_{\max})$ with density function
\begin{align}\label{eq:power law distribution}
    f_X(r) = \begin{cases}
                 \frac{2}{r_{\min}^{-2}-r_{\max}^{-2}}\cdot r^{-3}, & r_{\min} \leq r \leq r_{\max} \\
                 0,                                                 & \text{else,}
             \end{cases}
\end{align}
as \emph{power-law distribution} (Line~\ref{alg:dlm radius}).

The leaves' positions $\p$ are sampled from a uniform distribution over the square image with edge length $s$ and padding $r_{\max}$ in both directions (Line~\ref{alg:dlm location}).
The padding allows leaves to reach into the image while having a center outside the image.
With the previous notation, this means $A = [0,s]^2$ and $B = [-r_{\max}, s + r_{\max}]^2$.
The density function for the position is constant and given by
\begin{align}
    f_P(p) = \frac{1}{\lambda(B)} = \frac{1}{(s+2\cdot r_{\max})^2}.
\end{align}

All points that are not already part of a leaf and fall into the generated leaf are labelled with the leaf's number (Line~\ref{alg:dlm membership function}).
New leaves are generated iteratively until the leaves cover the full image (Line~\ref{alg:dlm while}).

For the implementation of color and texture there are many options to choose from.
By choosing the RGB color model we would obtain a large variety of hues.
Using an HSV/HSL model alternatively would result in the decoupling of color and luminance in the sampling.
We refrain from choosing specific color spaces and distributions here.
Therefore, colors are sampled from some distribution $\mathcal{C}$ with parameters $\theta_c$ (Line~\ref{alg:dlm color}).
Subsequently, texture from some texture distribution $\mathcal{T}$ with parameters $\theta_t$ (Line~\ref{alg:dlm texture}) is added to each color channel of each pixel.

\begin{algorithm}\caption{Dead leaves image generation algorithm}\label{alg:dlm}
    \KwIn{$(r_{\min},r_{\max},s,\theta_c,\theta_t)$}
    $\operatorname{model} \leftarrow \operatorname{zeros}(s,s)$\;
    $i \leftarrow 1$\;
    \While{
    $\operatorname{any}(\operatorname{model} == 0)$ \label{alg:dlm while}
    }{
    $\mathbf{r} \sim \mathcal{P}(r_{\min}, r_{\max})$\; \label{alg:dlm radius}
    $\p \sim \mathcal{U}(- r_{\max},s + r_{\max},2)$\; \label{alg:dlm location}
    \If{
    $\operatorname{any}\left(\operatorname{model} \cap \operatorname{leaf}_{\mathbf{r},\p} == 0\right)$
    }{
    $\operatorname{model}\left[\operatorname{leaf}_{\mathbf{r},\p} \cap (\operatorname{model} == 0)\right] \leftarrow i$\; \label{alg:dlm membership function}
    $i \leftarrow i+1$\;
    }
    }
    $\operatorname{image} \leftarrow \operatorname{zeros}(s,s,3)$\;
    $\mathbf{c} \sim \mathcal{C}(\theta_c, i)$\; \label{alg:dlm color}
    $\mathbf{t} \sim \mathcal{T}(\theta_t, (s,s))$\; \label{alg:dlm texture}
    \For{
        $x,y=1,\dots,s$
    }{
        $\operatorname{image}_{x,y} \leftarrow \mathbf{c}[\operatorname{model}_{x,y}] + \mathbf{t}_{x,y}$\;
    }
    \KwOut{$\operatorname{model}, \operatorname{image}$}
\end{algorithm}

Using this algorithm to randomly generate dead leaves images will give a variety of results as depicted in Figure~\ref{img:random dlm}.
A different dead leaves generation algorithm based on similar ideas was proposed by \citet{Achddou2021}.

\begin{figure}[ht]\centering
    \begin{subfigure}{0.24\textwidth}
        \includegraphics[width=\linewidth]{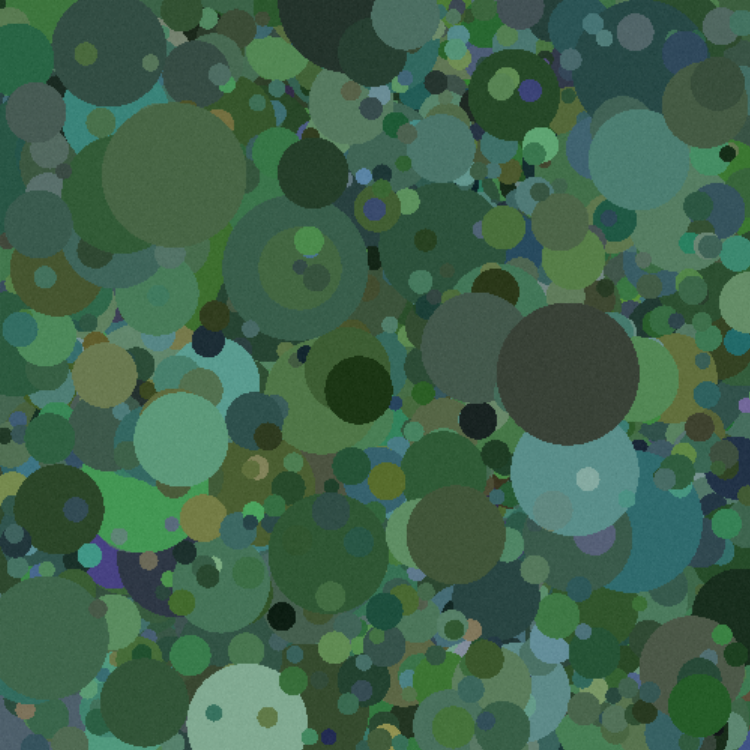}
    \end{subfigure}
    \hfill
    \begin{subfigure}{0.24\textwidth}
        \includegraphics[width=\linewidth]{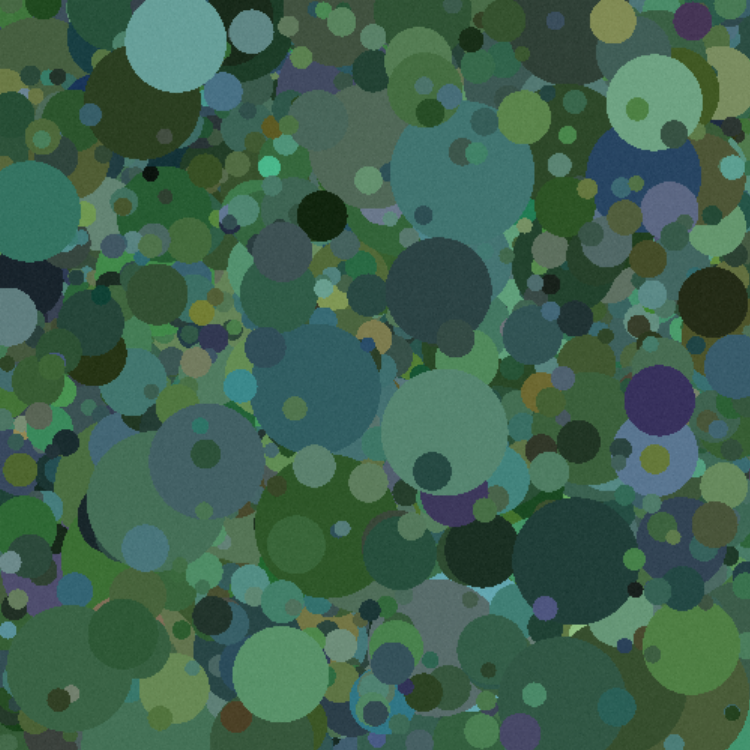}
    \end{subfigure}
    \hfill
    \begin{subfigure}{0.24\textwidth}
        \includegraphics[width=\linewidth]{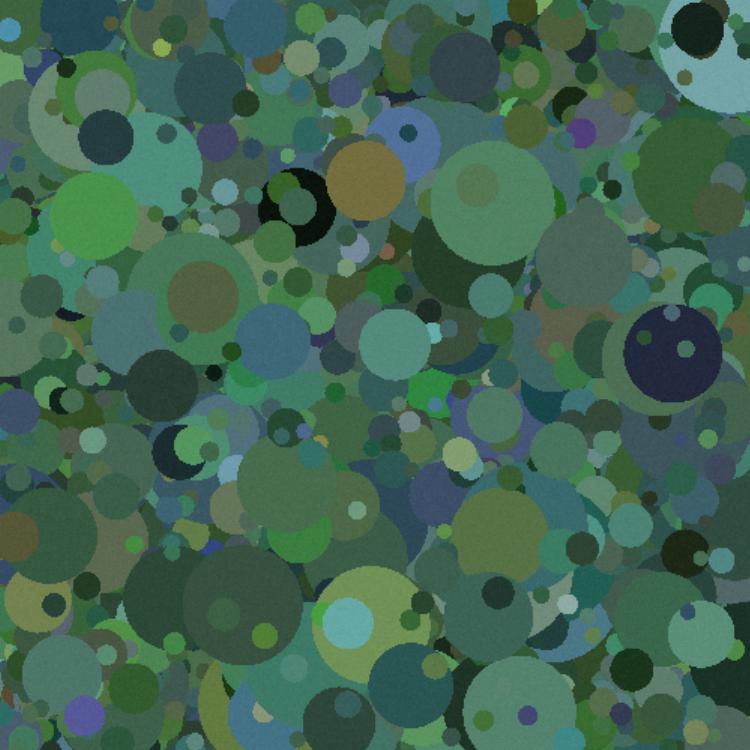}
    \end{subfigure}
    \hfill
    \begin{subfigure}{0.24\textwidth}
        \includegraphics[width=\linewidth]{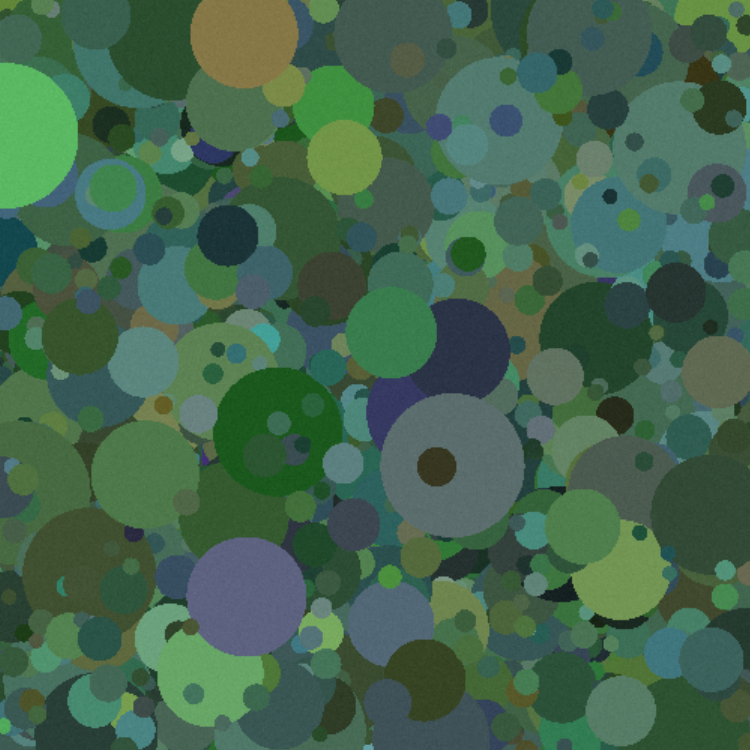}
    \end{subfigure}
    \caption{Random dead leaves image generated with algorithm~\ref{alg:dlm} with normally distributed colors and mean-zero Gaussian texture in HSV ranging from $0$ to $1$ ($s=500$\,px, $r_{\min}=5$\,px, $r_{\max}=50$\,px, $\mu_c=(0.5,0.5,0.5)$, $\sigma_c=(0.1,0.1,0.1)$, $\sigma_t=(0.05,0.05,0.05)$).}\label{img:random dlm}
\end{figure}

Using the generation process described here, we investigate the segmentation a human would assume when viewing such an image.
We will model this choice with a Bayesian ideal observer.

\section{Ideal observer} \label{sec: ideal observer}

A possible goal of a perceiving system can in general be described as trying to infer the state of the world from sensory observations. 
In our case, the sensory observation is then the dead leaves image, and the goal of an observer in this setting could be to infer the most probable partition to cause this image.
Following Bayes' Theorem the probability of a dead leaves partition $\m$ given some image $\mathbf{s}$ is
\begin{align}\label{eq:Bayes}
    \P_{M\mid S}\left(\m\mid \mathbf{s}\right) & = \frac{\P_{S\mid M}\left(\mathbf{s}\mid \m\right)\P_{M}\left(\m\right)}{\P_S\left(\mathbf{s}\right)}.
\end{align}
Since in a continuous distribution every single point has probability zero, the idea of discretization comes in handy at this point.
Let $a\subset A$ be a finite subset of the visible area, e.\,g., a grid of pixels, and let $\m_a = \{\v_{a,i}\}_{i\in[n]}$ be a partition of $a$.
We can now consider the set of dead leaves partitions on $A$ for which the reduction from $A$ to $a$ yields the partition $\m_a$:
\begin{align*}
    \mathbf{M}_a =
    \left\{ \begin{aligned}
                 & \m_A = \{\v_{A,k}\}_{k\in[m]} \text{ partition s.\,t.}                                                           \\
                 & \exists\text{ strictly monotone series } (k_i)_{i\in[n]}, k_i\in[m], \v_{A,k_i} \cap~a =\v_{a,i}~\forall i\in[n]
            \end{aligned}
    \right\} \subset \M_A.
\end{align*}
The probability of sampling $\m_a$ on $a$ from a random dead leaves partition $M_A$ is given through $\P(\m_a) = \P\left[M_A \in \mathbf{M}_a\right]$.
When we want to ignore the order of leaves in $\m_a$ we drop the condition that the sequence $k_i$ needs to be monotone.

We can restrict a partition $\m_a$ to any subset of $a$.
When doing so we drop any resulting empty sets from the partition, in particular $\m_{a\setminus \v_{a,i}} = \m_a \setminus \v_{a,i}$ is then a partition of $a\setminus \v_{a,i}$.

The most probable partition $\m_a$ on $a$ is then the partition which has the maximal posterior probability.
\begin{align}\label{eq:MAP}
    \begin{split}
        \m_a^{\ast}
         & \coloneq \argmax_{\m_a \in \M_a} \P\left(\m_a \mid \mathbf{s}\right)
        = \argmax_{\m_a \in \M_a}\frac{\P\left(\mathbf{s}\mid \m_a\right)\P(\m_a)}{\P(\mathbf{s})}
        = \argmax_{\m_a \in \M_a}\P\left(\mathbf{s}\mid \m_a\right)\P(\m_a)
    \end{split}
\end{align}

To derive the most probable partition for a fixed image we only require knowledge of the likelihood of an image ($\P(\mathbf{s}\mid \mathbf{M}_a)$) and the prior probability of possible partitions ($\P(\mathbf{M}_a)$) since the probability of the specific image ($\P(\mathbf{s})$) is independent of $\m_a$ for any given image $\mathbf{s}$.

In the subsequent sections we derive how these components can be computed based on the modelling introduced before.

\subsection{Prior} \label{sec: prior}
The prior probability of a partition is the probability of the partition's memberships occurring based on the generative process of the dead leaves partition.
\footnote{All propositions (Theorems, Lemmas, etc.) in this section are reformulations of the work done by \cite{Pitkow2010}.
The proofs for Theorems~\ref{thm:prior decomposition} and \ref{thm:leaf prob} as well as for Proposition~\ref{prop:constructive computation} are based on and extend the reasoning of \cite{Pitkow2010}.}

\subsubsection{Decomposition of prior probability}

Since each partition consists of independent leaves, a key component to calculate the partition probability is the leaf probability, i.\,e., the probability of a random leaf being equal to a specific leaf or set of pixels.

\begin{thm}[Decomposition of prior probability] \label{thm:prior decomposition}
    Let $a$ be a finite set of visible points in $\R^2$ and $\m_a = \{\v_{a,i}\}_{i\in[n]}$ an ordered partition of $a$.
    Then, the prior probability of $\m_a$ based on the previous dead leaves partition generation is given by
    \begin{align}
        \P(\m_a) = \prod_{i=1}^{n}\frac{\P(L_{a_{\setminus i-1},i}=\v_{a,i})}{\P(L_{a_{\setminus i-1},i}\neq \emptyset)},
    \end{align}
    where $a_{\setminus i} \coloneq a \setminus \v_{a,1}\cup\dots\cup\v_{a,i}$ and $L_{a_{\setminus i-1},i}$ is a random leaf on $A$ reduced to points in $a_{\setminus i-1}$.
    If we ignore the order of the partition $\m_a$, the probability increases to
    \begin{align}
        \P(\m_a) = \sum_{j=1}^{n!} \prod_{i=1}^{n}\frac{\P(L_{a_{\setminus \pi_j(i-1)},\pi_j(i)}=\v_{a,\pi_j(i)})}{\P(L_{a_{\setminus \pi_j(i-1)},\pi_j(i)}\neq \emptyset)},
    \end{align}
    where $\{\pi_j\}_{j\in [n!]}$ is the list of all permutations of $\m_a$ and $a_{\setminus \pi_j(i)}\coloneq a \setminus \v_{a,\pi_j(1)}\cup\dots\cup\v_{a,\pi_j(i)}$, i.\,e., the set of pixels remaining after removing the top $i-1$ leaves in the partition ordered according to $\pi_j$.
\end{thm}

The idea to prove this claim is to separate the prior probability using the law of total probability and going iteratively through the stack of leaves.

\begin{proof}
    We start by considering the set of pixels $\v_{a,1}\in\m_{a}$.
    By construction, any random leaf on $a$ is an unoccluded random leaf.
    Let $L_1$ be such a leaf without occlusion that is distributed according to the previous construction, i.\,e., with uniform position and power-law distributed radius.
    We then distinguish three cases:
    \begin{enumerate}[label = Case \arabic*:, leftmargin = 1.6cm]
        \item The leaf hits none of the selected pixels: \[L_{a,1} \coloneq L_1 \cap a = \emptyset,\]
        \item The leaf and the given partition match on the selected pixels: \[L_{a,1} = \v_{a,1},\]
        \item The leaf and the given partition do not match: \[L_{a,1} \neq \v_{a,1} \text{ and } L_{a,1} \neq \emptyset.\]
    \end{enumerate}

    In this last case, the partition $\m_a$ already becomes impossible, i.\,e., $\P(\m_a\mid L_{a,1} \neq \emptyset \text{ and } L_{a,1} \neq \v_{a,1}) = 0$.
    Hence, by the law of total probability, we can separate the prior probability into the two other cases:

    \begin{align}\label{eq:top leaf prob}
        \begin{split}
            \P(\m_a) = & \P(\m_a \mid L_{a,1} = \v_{a,1})\P(L_{a,1} = \v_{a,1}) \\ &+ \P(\m_a \mid L_{a,1} =\emptyset)\P(L_{a,1} =\emptyset).
        \end{split}
    \end{align}
    By construction, the infinite series $\{L_i\}_{i\in \N}$ consists of independent and identically distributed leaves.
    Consequently, the leaves $\{L_i\}_{i\geq 2}$ have the same joint distribution as $\{L_i\}_{i\geq 1}$.
    It follows that the probability of generating $\m_a$ with leaves $\{L_i\}_{i\geq 1}$ is equal to the probability of generating $\m_a$ with leaves $\{L_i\}_{i\geq 2}$.
    Conditioning on any event that depends only on $L_1$ and that ensures $L_1$ does not intersect with $a$ has the same effect as switching from generating $\m_a$ with $\{L_i\}_{i\geq 1}$ to generating $\m_a$ with $\{L_i\}_{i\geq 2}$, and thus gives the same probability.
    Hence,
    \begin{align*}
        \P(\m_a) = \P(\m_a \mid L_1 \cap a = \emptyset) = \P(\m_a \mid L_{a,1} = \emptyset).
    \end{align*}
    We can apply the same reasoning to the smaller pixel set $a\setminus \v_{a,1}$.
    The condition $L_{a,1} = \v_{a,1}$ only depends on $L_1$ and ensures that $L_{a,1}$ does not intersect with $a\setminus \v_{a,1}$.
    Therefore, we get
    \begin{align*}
        \P(\m_{a \setminus \v_{a,1}}) = \P(\m_{a \setminus \v_{a,1}} \mid L_{a,1} = \v_{a,1}) = \P(\m_a \mid L_{a,1} = \v_{a,1}).
    \end{align*}
    Solving equation~\ref{eq:top leaf prob} for $\P(\m_a)$ then comes down to
    \begin{align}
        \P(\m_a)
        = \frac{\P(L_{a,1}=\v_{a,1})}{1-\P(L_{a,1}=\emptyset)}\P(\m_{a\setminus \v_{a,1}})
        = \frac{\P(L_{a,1}=\v_{a,1})}{\P(L_{a,1}\neq\emptyset)}\P(\m_{a\setminus \v_{a,1}}).
        \label{eq:top leaf prop reduced}
    \end{align}

    Since $\m_{a\setminus \v_{a,1}}$ is the dead leaves partition reduced to a now smaller set of points, we can apply the same logic as before.
    With the set getting smaller and smaller in each iteration, we can continue this loop through the full partition of $a$ until $a\setminus \v_{a,1}\cup\dots\cup\v_{a,n} = \emptyset$.
    We will write $a\setminus \v_{a,1}\cup\dots\cup\v_{a,i}$ as $a_{\setminus i}$ for ease of notation, where $a_{\setminus 0} =a$.
    For each iteration, $i\in[n]$, we get
    \begin{align}\label{eq:sub model prob}
        \P(\m_{a_{\setminus i-1}}) = \frac{\P(L_{a_{\setminus i-1}, i}=\v_{a,i})}{\P(L_{a_{\setminus i-1}, i} \neq \emptyset)}\P(\m_{a_{\setminus i}}).
    \end{align}
    Concatenating the subpartition probabilities (Eq.~\ref{eq:sub model prob}) and using $\P(\m_{a_{\setminus n}})=\P(\m_\emptyset)=1$ yields
    \begin{align}\label{eq:model prob}
        \P(\m_a) = \prod_{i=1}^n \frac{\P(L_{a_{\setminus i-1}, i}=\v_{a,i})}{\P(L_{a_{\setminus i-1}, i} \neq \emptyset)}.
    \end{align}

    The above equation assumes that the sets in $\m_a$ are ordered by depth such that the first set $\v_{a,1}$ has no occlusion.
    If the sets in $\m_a$ are not ordered by depth, we need to adjust equation~\ref{eq:top leaf prob} to allow for any possible ordering, i.\,e., permutation $\{\pi_j\}_{j\in [n!]}$ such that
    \begin{align*}
        \P(\m_a) = \sum_{j=1}^{n!} \prod_{i=1}^{n}\frac{\P(L_{a_{\setminus \pi_j(i-1)},\pi_j(i)}=\v_{a,\pi_j(i)})}{\P(L_{a_{\setminus \pi_j(i-1)},\pi_j(i)}\neq \emptyset)}.
    \end{align*}
\end{proof}

\begin{rem}
    The probability for a random leaf to be non-empty can be derived by summing the probabilities of hitting a given pixel set for all possible subsets of $a$:
    \begin{align*}
        \P(L_a \neq \emptyset) = \sum_{\v \in \mathcal{P}(a) \setminus \{\emptyset\}} \P(L_a = \v).
    \end{align*}
    In order to compute the prior probability of a dead leaves partition, we therefore need to compute the leaf probability for all subsets of $a$ and its layers.
    Alternatively, we can compute the probability through
    \begin{align*}
        \P(L_a \neq \emptyset) = 1 - \P(L_a = \emptyset)
    \end{align*}
    We can think of the term $\P(L_a \neq \emptyset)$ as a normalization factor since it removes the dependence of the probability on the density of $a$ in $A$.
\end{rem}

To get a better intuition of this result, we consider our example from the construction.
Keep in mind that, for this computation, we only know the partition of $a$ that resulted from the construction in example~\ref{ex1: projection}.

\begin{exone}
    Recall our previous example with nine pixels, $a=[0,1,2]^2$, and three leaves
    \begin{align*}
        \v_1 & = \{(1,2),(2,1),(1,1),(2,2)\}, \\
        \v_2 & = \{(0,0),(0,1),(1,0),(0,2)\}, \\
        \v_3 & = \{(2,0)\},
    \end{align*}
    This model is visualized in Figure~\ref{fig:exone m_a}.
    Since we know the order of the leaves, we can use the above Theorem to get
    \begin{align*}
        \P(\m_a) = \frac{\P(L_{a,1} = \v_{a,1})}{\P(L_{a,1}\neq \emptyset)} \cdot \frac{\P(L_{a_{\setminus 1}, 2}=\v_{a,2})}{\P(L_{a_{\setminus 1},2}\neq \emptyset)} \cdot \frac{\P(L_{a_{\setminus 2}, 3}=\v_{a,3})}{\P(L_{a_{\setminus 2},3}\neq \emptyset)}.
    \end{align*}
    As the above remark states, the probabilities in the denominators are given by summing the probabilities for all possible subsets of $a$.
    Since this is a very large number, we will focus on the subsets that are leaves in our model.
    The first leaf probability $\P(L_{a,1}=\v_{a,1})$ is the probability of having a leaf which includes all pixels in $\v_{a,1}$ while excluding all remaining pixels of $a$ (Figure~\ref{fig:exone a v1}).
    We then go to the next level by removing the pixels of $\v_{a,1}$.
    The next leaf probability then is $\P(L_{a_{\setminus 1},2}=\v_{a,2})$, i.\,e., the probability of having a leaf which includes all pixels in $\v_{a,2}$ and excludes all remaining pixels of $a\setminus \v_{a,1}$ (Figure~\ref{fig:exone a_1 v2}).
    The last layer then only has the pixels of $\v_{a,3}$ left, i.\,e., $a_{\setminus 2} = \v_{a,3}$.
    So the last leaf probability $\P(L_{a_{\setminus 2},3} = \v_{a,3})$ is only the probability of having a leaf which includes all remaining pixels, since we do not have any pixels we are not allowed to hit (Figure~\ref{fig:exone a_12 v3}).
    If we did not know the order of the leaves, we would have to perform this computation for all possible orders of leaves, which results in a lot more cases to cover (Figure~\ref{fig:exone a v1}--\ref{fig:exone a_32 v1}).
\end{exone}

\begin{figure*}[ht]\centering
    \begin{subfigure}{2.3cm}\centering
        \caption{}\label{fig:exone m_a}
        \definecolor{c1}{rgb}{0.55,0.215,0.6}
        \definecolor{c2}{rgb}{0.26,0.34,0.58}
        \definecolor{c3}{rgb}{0.27,0.76,0.72}
        \begin{tikzpicture}[scale=0.5]
            \foreach \x in {(1,1),(1,2),(2,1),(2,2)}{
                    \node[point, fill=c1, draw=c1] at \x {};
                }
            \foreach \x in {(0,0),(1,0),(0,1),(0,2)}{
                    \node[point, fill=c2, draw=c2] at \x {};
                }
            \foreach \x in {(2,0)}{
                    \node[point, fill=c3, draw=c3] at \x {};
                }
        \end{tikzpicture}
    \end{subfigure}\vspace{2mm}

    \begin{tabularx}{\textwidth}{lXXXXXX}
        \toprule
        $1$st level                         & \multicolumn{2}{c}{
            \begin{subfigure}{2.3cm}\centering
                \caption{}\label{fig:exone a v1}
                \begin{tikzpicture}[scale=0.5]
                    \foreach \x in {(1,1),(1,2),(2,1),(2,2)}{
                            \node[included point] at \x {};
                        }

                    \foreach \x in {(0,0),(1,0),(0,1),(0,2)}{
                            \node[excluded point] at \x {};
                        }

                    \foreach \x in {(2,0)}{
                            \node[excluded point] at \x {};
                        }
                \end{tikzpicture}
            \end{subfigure}
        }                                   &
        \multicolumn{2}{c}{
            \begin{subfigure}{2.3cm}\centering
                \caption{}\label{fig:exone a v2}
                \begin{tikzpicture}[scale=0.5]
                    \foreach \x in {(1,1),(1,2),(2,1),(2,2)}{
                            \node[excluded point] at \x {};
                        }

                    \foreach \x in {(0,0),(1,0),(0,1),(0,2)}{
                            \node[excluded point] at \x {};
                        }

                    \foreach \x in {(2,0)}{
                            \node[included point] at \x {};
                        }
                \end{tikzpicture}
            \end{subfigure}
        }                                   &
        \multicolumn{2}{c}{
            \begin{subfigure}{2.3cm}\centering
                \caption{}\label{fig:exone a v3}
                \begin{tikzpicture}[scale=0.5]
                    \foreach \x in {(1,1),(1,2),(2,1),(2,2)}{
                            \node[excluded point] at \x {};
                        }

                    \foreach \x in {(0,0),(1,0),(0,1),(0,2)}{
                            \node[included point] at \x {};
                        }

                    \foreach \x in {(2,0)}{
                            \node[excluded point] at \x {};
                        }
                \end{tikzpicture}
            \end{subfigure}
        }
        \\ \midrule
        $2$nd level                         &
        \begin{subfigure}{2.3cm}\centering
            \caption{}\label{fig:exone a_1 v2}
            \begin{tikzpicture}[scale=0.5]
                \foreach \x in {(1,1),(1,2),(2,1),(2,2)}{
                        \node[irrelevant point] at \x {};
                    }

                \foreach \x in {(0,0),(1,0),(0,1),(0,2)}{
                        \node[included point] at \x {};
                    }

                \foreach \x in {(2,0)}{
                        \node[excluded point] at \x {};
                    }
            \end{tikzpicture}
        \end{subfigure}
                                            &
        \begin{subfigure}{2.3cm}\centering
            \caption{}\label{fig:exone a_1 v3}
            \begin{tikzpicture}[scale=0.5]
                \foreach \x in {(1,1),(1,2),(2,1),(2,2)}{
                        \node[irrelevant point] at \x {};
                    }

                \foreach \x in {(0,0),(1,0),(0,1),(0,2)}{
                        \node[excluded point] at \x {};
                    }

                \foreach \x in {(2,0)}{
                        \node[included point] at \x {};
                    }
            \end{tikzpicture}
        \end{subfigure}
                                            &
        \begin{subfigure}{2.3cm}\centering
            \caption{}\label{fig:exone a_2 v1}
            \begin{tikzpicture}[scale=0.5]
                \foreach \x in {(1,1),(1,2),(2,1),(2,2)}{
                        \node[included point] at \x {};
                    }

                \foreach \x in {(0,0),(1,0),(0,1),(0,2)}{
                        \node[excluded point] at \x {};
                    }

                \foreach \x in {(2,0)}{
                        \node[irrelevant point] at \x {};
                    }
            \end{tikzpicture}
        \end{subfigure}
                                            &
        \begin{subfigure}{2.3cm}\centering
            \caption{}\label{fig:exone a_2 v3}
            \begin{tikzpicture}[scale=0.5]
                \foreach \x in {(1,1),(1,2),(2,1),(2,2)}{
                        \node[excluded point] at \x {};
                    }

                \foreach \x in {(0,0),(1,0),(0,1),(0,2)}{
                        \node[included point] at \x {};
                    }

                \foreach \x in {(2,0)}{
                        \node[irrelevant point] at \x {};
                    }
            \end{tikzpicture}
        \end{subfigure}
                                            &
        \begin{subfigure}{2.3cm}\centering
            \caption{}\label{fig:exone a_3 v1}
            \begin{tikzpicture}[scale=0.5]
                \foreach \x in {(1,1),(1,2),(2,1),(2,2)}{
                        \node[included point] at \x {};
                    }

                \foreach \x in {(0,0),(1,0),(0,1),(0,2)}{
                        \node[irrelevant point] at \x {};
                    }

                \foreach \x in {(2,0)}{
                        \node[excluded point] at \x {};
                    }
            \end{tikzpicture}
        \end{subfigure}
                                            &
        \begin{subfigure}{2.3cm}\centering
            \caption{}\label{fig:exone a_3 v2}
            \begin{tikzpicture}[scale=0.5]
                \foreach \x in {(1,1),(1,2),(2,1),(2,2)}{
                        \node[excluded point] at \x {};
                    }

                \foreach \x in {(0,0),(1,0),(0,1),(0,2)}{
                        \node[irrelevant point] at \x {};
                    }

                \foreach \x in {(2,0)}{
                        \node[included point] at \x {};
                    }
            \end{tikzpicture}
        \end{subfigure}                         \\
        \midrule
        $3$rd level                         &
        \begin{subfigure}{2.3cm}\centering
            \caption{}\label{fig:exone a_12 v3}
            \begin{tikzpicture}[scale=0.5]
                \foreach \x in {(1,1),(1,2),(2,1),(2,2)}{
                        \node[irrelevant point] at \x {};
                    }

                \foreach \x in {(0,0),(1,0),(0,1),(0,2)}{
                        \node[irrelevant point] at \x {};
                    }

                \foreach \x in {(2,0)}{
                        \node[included point] at \x {};
                    }
            \end{tikzpicture}
        \end{subfigure} &
        \begin{subfigure}{2.3cm}\centering
            \caption{}\label{fig:exone a_13 v2}
            \begin{tikzpicture}[scale=0.5]
                \foreach \x in {(1,1),(1,2),(2,1),(2,2)}{
                        \node[irrelevant point] at \x {};
                    }

                \foreach \x in {(0,0),(1,0),(0,1),(0,2)}{
                        \node[included point] at \x {};
                    }

                \foreach \x in {(2,0)}{
                        \node[irrelevant point] at \x {};
                    }
            \end{tikzpicture}
        \end{subfigure} &
        \begin{subfigure}{2.3cm}\centering
            \caption{}\label{fig:exone a_21 v3}
            \begin{tikzpicture}[scale=0.5]
                \foreach \x in {(1,1),(1,2),(2,1),(2,2)}{
                        \node[irrelevant point] at \x {};
                    }

                \foreach \x in {(0,0),(1,0),(0,1),(0,2)}{
                        \node[included point] at \x {};
                    }

                \foreach \x in {(2,0)}{
                        \node[irrelevant point] at \x {};
                    }
            \end{tikzpicture}
        \end{subfigure} &
        \begin{subfigure}{2.3cm}\centering
            \caption{}\label{fig:exone a_23 v1}
            \begin{tikzpicture}[scale=0.5]
                \foreach \x in {(1,1),(1,2),(2,1),(2,2)}{
                        \node[included point] at \x {};
                    }

                \foreach \x in {(0,0),(1,0),(0,1),(0,2)}{
                        \node[irrelevant point] at \x {};
                    }

                \foreach \x in {(2,0)}{
                        \node[irrelevant point] at \x {};
                    }
            \end{tikzpicture}
        \end{subfigure} &
        \begin{subfigure}{2.3cm}\centering
            \caption{}\label{fig:exone a_31 v2}
            \begin{tikzpicture}[scale=0.5]
                \foreach \x in {(1,1),(1,2),(2,1),(2,2)}{
                        \node[irrelevant point] at \x {};
                    }

                \foreach \x in {(0,0),(1,0),(0,1),(0,2)}{
                        \node[irrelevant point] at \x {};
                    }

                \foreach \x in {(2,0)}{
                        \node[included point] at \x {};
                    }
            \end{tikzpicture}
        \end{subfigure} &
        \begin{subfigure}{2.3cm}\centering
            \caption{}\label{fig:exone a_32 v1}
            \begin{tikzpicture}[scale=0.5]
                \foreach \x in {(1,1),(1,2),(2,1),(2,2)}{
                        \node[included point] at \x {};
                    }

                \foreach \x in {(0,0),(1,0),(0,1),(0,2)}{
                        \node[irrelevant point] at \x {};
                    }

                \foreach \x in {(2,0)}{
                        \node[irrelevant point] at \x {};
                    }
            \end{tikzpicture}
        \end{subfigure}
    \end{tabularx}
    \caption{Example for prior probability decomposition.
    (\subref{fig:exone m_a}) Dead leaves model on $a = [0,1,2]^2$ from the previous model.
    The leaf membership of each pixel is indicated through color.
    (\subref{fig:exone a v1}) - (\subref{fig:exone a_32 v1}) For each iteration, we have two reference points: the overall set of pixels and the leaf we are trying to match.
    With each iteration, we move one level down through the model ($1$st to $3$rd level) and only consider the remaining pixels.
    For each layer, pixels are marked green if they should be part of the target leaf and red if not.
    Grey pixels were already allocated in a previous level and are therefore not relevant.
    }
\end{figure*}

Following the notation of previous work \citep{Matheron1968,Goussea2003,Pitkow2010}, we can denote the leaf or inclusion probabilities
\footnote{Originally \citet{Matheron1968} used $Q(K)$ to denote the probability that a compact set $K$ is part of some leaf. \citet{Goussea2003} extended his computation of this probability to a series of compact sets. Hence, the term \emph{inclusion probability} was used.}
as
\begin{align*}
    Q_{a,i}\colon \mathcal{P}(a) \to [0,1], \v \mapsto \P(L_{a,i} = \v).
\end{align*}

Using this notation, equation~\ref{eq:top leaf prop reduced} matches the result of \citet[Eq.~14]{Pitkow2010}:
\begin{align*}
    \P(\m_a) = \frac{Q_{a,1}(\v_{a,1})}{1-Q_{a,1}(\emptyset)} \P(\m_{a_{\setminus 1}}).
\end{align*}
Equation~\ref{eq:model prob} can be written as
\begin{align*}
    \P(\m_a) = \prod_{i=1}^n\frac{Q_{a_{\setminus i-1},i}(\v_{a,i})}{1-Q_{a_{\setminus i-1},i}(\emptyset)}.
\end{align*}

\subsubsection{Leaf probabilities}

To derive the necessary probabilities of each leaf, we can use the law of total probability to decompose the probability into integrals over the radius and the position of the leaf (this equation is, in particular, true for $\v=\emptyset$):

\begin{align}\label{eq:leaf prob}
    \P(L_a = \v) = \int_{-\infty}^{\infty} f_X(r) \int_{\R^2} f_P(p)\P(L_a(r,p) = \v\mid r,p)\,\mathbf{d}p\mathbf{d}r.
\end{align}

Since a random leaf with fixed radius and position is not random anymore, the probability $\P(L_a(r,p)=\v \mid r,p)$ only takes values in $\{0,1\}$ depending on $r$, $p$, $a$, and $\v$.
A key point in computing this integral will be to consider the area in $\R^2$ where a leaf can be placed such that $L_a(r,p) = \v$, i.\,e., where $\P(L_a(r,p)=\v \mid r,p)=1$.
Therefore, we will first spend some time on showing how this area, and especially its boundary, is constructed.

\begin{rem}
    The analytical solution of these integrals is possible for the scenario setup here, and we will derive it in the following pages.
    However, we can also approximate the leaf probability by employing discretizations of the probability distributions and summing instead of taking the integral.
    This can be done only for the positions (half discrete), e.\,g., with a grid approximation, or for both positions and radius (full discrete).
    The probability is then computed by counting for each radius which positions of the grid are possible ($\P(L_a(r,p)=\v \mid r,p)$) and computing the weighted sums.
    An advantage of such an approximation is that we can compute it for any leaf shape as long as we can apply a mask to the grid of positions.
    In particular, this is possible for shapes for which we cannot derive a closed analytical solution to the integral.
\end{rem}

\paragraph{Area of possible positions}

\begin{lem}[Area of possible positions]\label{lem:area of possible positions}
    Let $\v\in \mathcal{P}(a)$ be a set of pixels. The area of possible leaf positions for $r \in \R$ is then given by
    \begin{align*}
        \p_{r,\v} & = \{p\in\R^2 \mid L_a(r,p)=\v\}
        = \bigcap_{x\in\v} L(r,x) \setminus \bigcup_{x\in a\setminus\v} L(r,x)                                                         \\
                  & = \{p \in \R^2 \mid \norm{p - x}\leq r~\forall x\in\v  \text{ and } \norm{p - x} > r~\forall x\in a\setminus \v\}.
    \end{align*}
    Additionally, the area of impossible leaf positions for an empty leaf is
    \begin{align*}
        \bar{\p}_{r,\emptyset} = \{p\in\R^2\mid L_a(r,p)\neq \emptyset\} = \bigcup_{x\in a}L(r,x) = \{p\in\R^2 \mid \norm{p-x}\leq r~\exists x\in a\}.
    \end{align*}
    In particular, for all $r \leq r_{\max}$ and $\v\neq \emptyset$, we have $\p_{r,\v} \subset B$ and $\bar{\p}_{r,\emptyset} \subset B$.
\end{lem}

\begin{proof}
    For a given pixel $x\in a$ and fixed radius $r$, the pixel lies within the leaf $L_a(r,p)$ if the distance between $x$ and $p$ is smaller than $r$ (Figure~\ref{fig: leaf position single point})\footnote{
        For an arbitrary leaf shape, the pixel lies within the leaf if the position of the leaf is within the point reflection of the leaf at $x$. While for circles this reflection equals the original shape, for non-circular shapes the corresponding set of positions is a point reflection of the shape across the position $p$.
    }:
    \begin{align*}
        x \in L_a(r,p) \iff \norm{p-x} \leq r \iff p \in L(r,x).
    \end{align*}
    In particular, the area $L(r,x)$ is now not limited to points of $a$, but is a simply connected subset of $\R^2$.
    A set of pixels $\v\in \mathcal{P}(a)$ is part of a leaf $L_a(r,p)$ if the position $p$ is within the intersection of possible positions for each single pixel (Figure~\ref{fig: leaf position multiple points}), in other words, if $\norm{p-x}\leq r$ for all $x\in \v$.
    The set of pixels equals the leaf if the leaf furthermore does not include any pixels of $a$ which are not in $\v$ (Figure~\ref{fig: leaf position set}), i.\,e., if $\norm{p-x} > r$ for all $x \in a\setminus \v$.
    Similarly, if we want to avoid the empty set, any position $p$ within a radius of $r$ around any of the points in $a$ is feasible (Figure~\ref{fig: leaf position non-empty set}), resulting in the area of impossible positions above.
    Finally, for any position $p \in \p_{r,\v}$ or $p \in \bar{\p}_{r,\emptyset}$, there is a pixel $x \in a \subset A$ such that $\norm{x-p} \leq r \leq r_{\max}$. Since $B$ has a padding of $r_{\max}$ around $A$, it follows directly that $p \in B$.
\end{proof}

\begin{figure*}[ht]\centering
    \begin{minipage}{0.4\textwidth}
        \begin{subfigure}[t]{\linewidth}\centering
            \caption{}
            \includegraphics[width=\linewidth]{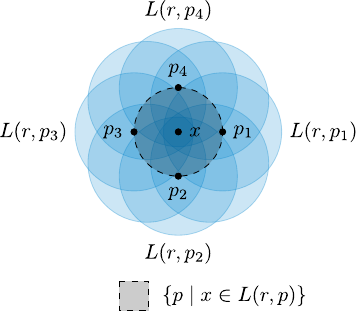}
            \label{fig: leaf position single point}
        \end{subfigure}
    \end{minipage}
    \hfill
    \begin{minipage}{0.53\textwidth}
        \begin{subfigure}[t]{\linewidth}
            \caption{}
            \includegraphics[width=\linewidth]{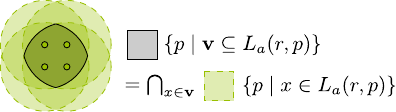}
            \label{fig: leaf position multiple points}
        \end{subfigure}
        \begin{subfigure}[t]{\linewidth}
            \caption{}\label{fig: leaf position set}
            \includegraphics[width=\linewidth]{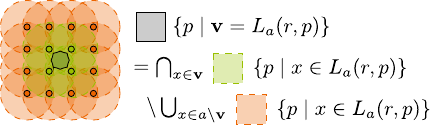}
        \end{subfigure}
        \begin{subfigure}[t]{\linewidth}
            \caption{}\label{fig: leaf position non-empty set}
            \includegraphics[width=\linewidth]{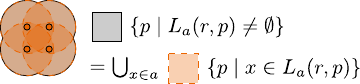}
        \end{subfigure}
    \end{minipage}
    \caption{Boundary for possible positions of a leaf (shaded areas).
        (\subref{fig: leaf position single point}) For including one single point $x$, each of the shown leaves (blue) of radius $r$ would suffice. The shaded area shows all possible positions for a leaf with radius $r$ that would include $x$.
        (\subref{fig: leaf position multiple points}) For including multiple points within one leaf, the leaf has to be positioned such that its center is within the intersection (shaded area) of each of the possible regions for every single point (green).
        (\subref{fig: leaf position set}) If we want to include some points (green) and exclude others (orange), the position of the leaf has to be within the possible area for the included points (green) while avoiding the areas that would reach the points to be excluded (orange).
        (\subref{fig: leaf position non-empty set}) If the leaf can include any of the selected pixels, i.\,e., be non-empty on $a$, possible positions (shaded area) are those which are close enough to any of the points in $a$ (orange).
    }
\end{figure*}

In order to compute the leaf probability, we want to derive the size of the area of possible positions through a contour integral.
Therefore, we need a parametrization of the boundary.

\begin{lem}[Parametrization of boundary]\label{lem:boundary parametrization}
    Let $\v \in \mathcal{P}(a)$ be a non-empty set of pixels and $\gamma_{r,\v}$ the curve describing the boundary of $\p_{r,\v}$.
    Then $\gamma_{r,\v}$ is piecewise smooth and has a finite number $K$ of singular points s.\,t. it can be divided into smooth subcurves $\gamma_{r,\v}^k\colon [t_k,t_k'] \to \R^2$ which are disjoint up to a finite number of points and for which the union of images is the boundary of $\p_{r,\v}$.
    Additionally, for each smooth subcurve $\gamma_{r,\v}^k$ there is a pixel $x_{i_k} \in a$ such that
    \begin{align*}
        \gamma_{r,\v}^k(t)
        = x_{i_k} + r \cdot \normal{t}
    \end{align*}
    and the outward pointing unit normal vector of each subcurve is given by
    \begin{align*}
        \vec{n}(\gamma_{r,\v}^k(t)) = c_{\v}(x_{i_k})\cdot \normal{t} \quad\text{with}\quad
        c_{\v}(x_{i_k}) = \begin{cases}
                              1,  & \text{if } x_{i_k} \in \v,            \\
                              -1, & \text{if } x_{i_k} \in a\setminus \v.
                          \end{cases}
    \end{align*}
    The same claim holds for the boundary $\gamma_{r,\emptyset}$ of $\bar{\p}_{r,\emptyset}$ with $c_{\emptyset}(x_{i_k}) = 1$.
\end{lem}

\begin{proof}
    Since the set $\p_{r,\v}$ is generated through union, intersection, and subtraction of smooth objects (Lemma~\ref{lem:area of possible positions}), i.\,e., circles, it has a piecewise smooth boundary.
    In particular, its boundary is a subset of the union of the individual circle boundaries.
    By definition, it can, therefore, be described by a set of smooth curves $\gamma_{r,\v}^k$ which go from one intersection (singular) point to another one (not necessarily ordered).
    Since each subcurve is then part of a circle boundary (Figure~\ref{fig:parametrization}), each subcurve can be written as:
    \begin{align*}
        \gamma_{r,\v}^k(t) & = x_{i_k} + r \cdot \normal{t}.
    \end{align*}
    The full boundary is then
    \begin{align*}
        \partial\p_{r,\v} = \bigcup_{k=1}^K\gamma_{r,\v}^k\left([t_k,t_k']\right)
    \end{align*}
    for some series $\{x_{i_k}\}_{k\in[K]}$ with $x_{i_k} \in a$ for all $k\in [K]$.
    In particular, a given pixel $x_i\in a$ can be the reference point for multiple subcurves.

    Finally, at each point of the boundary there are two unit normal vectors given by
    \begin{align*}
        \pm (\cos(t),\sin(t)).
    \end{align*}
    Since the positive vector always points away from the respective reference point and the negative vector always points towards it, the outward pointing unit vector is given by $(\cos(t),\sin(t))$ if $x_{i_k}\in \v$ and $(-\cos(t),-\sin(t))$ if $x_{i_k} \in a\setminus \v$.
    For the boundary of the area of impossible leaf positions for an empty leaf, the normal vector has to always point away from the reference point, and hence its sign is always positive.
\end{proof}

\begin{figure*}[ht]\centering
    \begin{subfigure}[t]{0.45\textwidth}\centering
        \caption{}\label{subfig:polar parametrization}
        \includegraphics[width=\linewidth]{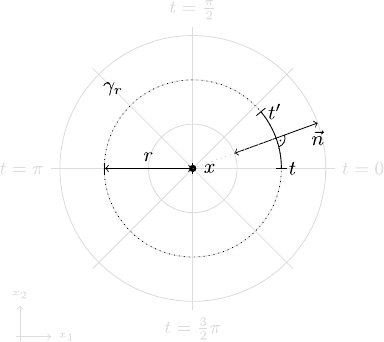}
    \end{subfigure}
    \hfill
    \begin{subfigure}[t]{0.45\textwidth}\centering
        \caption{}\label{subfig:boundary subcurves}
        \includegraphics[width=0.8\linewidth]{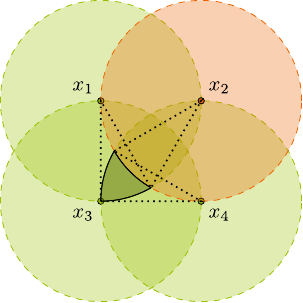}
        \vspace*{4pt}

        \footnotesize
        \begin{itemize*}[itemjoin=\quad]
            \item[\pp] Possible positions
            \item[\ip] Included positions
            \item[\ep] Excluded positions
        \end{itemize*}
    \end{subfigure}
    \caption{Parametrization of the boundary of $\p_{r,\v}$ (area of possible positions) as curve.
        (\subref{subfig:polar parametrization}) Each smooth boundary piece can be parametrized through radius $r$ and an interval of rotation angles $[t,t']$ w.\,r.\,t. some point $x$.
        (\subref{subfig:boundary subcurves}) The full boundary consists of multiple arcs of the same radius $r$ with different rotation centers. In this example, the relevant points for the boundary parametrization are $x_1$, $x_2$ and $x_4$, where $x_1$ and $x_4$ are part of $\v$ and $x_2$ is not.
    }
    \label{fig:parametrization}
\end{figure*}

Through this parametrization we can additionally derive the positions of the intersection points in the boundary.

\begin{lem}[Singular points in boundary]\label{lem:singular points}
    Let $\v \in \mathcal{P}(a)$ be a set of pixels and $\gamma_{r,\v}$ the curve describing the boundary of possible positions $\p_{r,\v}$. Then, each singular point in $\gamma_{r,\v}$ is given by the intersection of two circles of radius $r$ around some pixels $x_i$ and $x_j$ in $a$. In particular, for each intersection point there are $x_i$ and $x_j$ in $a$ such that
    \begin{align*}
        x_{ij\pm}
        = x_i + r \cdot \normal{t_{ij\pm}}
        = x_j + r \cdot \normal{t_{ji\mp}}
        = x_{ji\mp}
    \end{align*}
    where $t_{ij\pm} = \alpha_{ij} \pm \beta_{ij}(r)$ with
    \begin{align*}
        \alpha_{ij}   & = s_{ij}\cdot\cos^{-1}\left(\frac{x_{1j} - x_{1i}}{\norm{x_i - x_j}}\right), &
        \beta_{ij}(r) & = \cos^{-1}\left(\frac{\norm{x_i-x_j}}{2r}\right),
    \end{align*}
    where $s_{ij} =\sgn(x_{2j}-x_{2i})+(1-\sgn(x_{2j}-x_{2i})^2)\cdot\sgn(x_{1j}-x_{1i})$
    and $t_{ji\pm} = \alpha_{ji} \pm \beta_{ji}(r)$ with
    \begin{align*}
        \alpha_{ji} & = \alpha_{ij}-\pi & \beta_{ji}(r) = \beta_{ij}(r).
    \end{align*}
\end{lem}

With this representation, each intersection point can be given in two ways depending on the reference circle center, and through this parametrization we can track the intersection points, i.\,e., possible singularities, for a given set of pixels $\v$ and changing radius $r$.

\begin{proof}
    Let $x_i$ and $x_j$ be the center points of two circles with radius $r$ such that $\norm{x_i - x_j} \leq 2r$ (otherwise the circles do not intersect).
    The intersection points are then determined through
    \begin{align}\label{eq:intersection points}
        x_i + r \cdot \normal{t_{ij}} = x_j + r \cdot \normal{t_{ji}}.
    \end{align}
    We can separate the angle $t_{ij}$ into the angle of the vector connecting the two circle centers $\alpha_{ij}$ and a radius-dependent, symmetric offset $\beta_{ij}(r)$ (Figure~\ref{fig:boundary}).

    \begin{figure*}[ht]
        \begin{subfigure}[t]{0.45\textwidth}\centering
            \caption{}\label{subfig:intersection angles 1}
            \includegraphics[width=\linewidth]{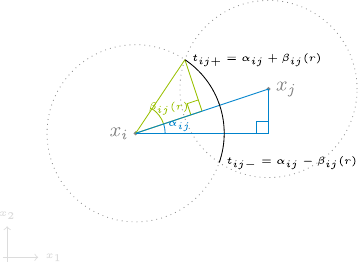}
        \end{subfigure}
        \hfill
        \begin{subfigure}[t]{0.45\textwidth}\centering
            \caption{}\label{subfig:intersection angles 2}
            \includegraphics[width=\linewidth]{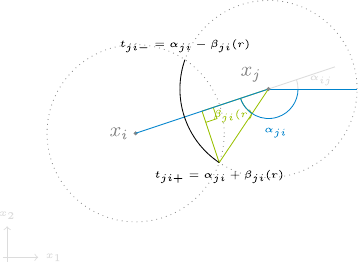}
        \end{subfigure}
        \caption{Visualization of singular points in boundary.
            (\subref{subfig:intersection angles 1}) Boundary and intersection points w.\,r.\,t. $x_i$.
            (\subref{subfig:intersection angles 2}) Boundary and intersection points w.\,r.\,t. $x_j$.
        }
        \label{fig:boundary}
    \end{figure*}

    By definition of the arccosine for $x_{2j} > x_{2i}$ (Figure~\ref{subfig:intersection angles 1}), the angle $\alpha_{ij}$ is given by
    \begin{align*}
        \cos^{-1}\left(\frac{x_{1j}-x_{1i}}{\norm{x_i-x_j}}\right).
    \end{align*}
    We extend the range of arccosine to the full circle $(-\pi,\pi)$ by multiplying with the sign $\sgn(x_{2j}-x_{2i})$, which is negative for $x_{2j} < x_{2i}$.
    Since this sign is zero for $x_{2j} = x_{2i}$, we set the sign to $\sgn(x_{1j}-x_{1i})$ for these cases, resulting in the combined sign $s_{ij}$.
    The signed angle $\alpha_{ij}$ then is shifted by a half rotation when swapping circle centers (Figure~\ref{subfig:intersection angles 2}).

    The offset angle $\beta_{ij}(r)$ similarly emerges from the isosceles triangle between $x_i$, $x_j$, and the intersection point as the arccosine of the hypotenuse with length $r$ and the adjacent with length $\frac{1}{2}\norm{x_i-x_j}$ (Figure~\ref{subfig:intersection angles 1}).
    In order to reach the same intersection point when swapping the two circle centers, we need to invert the direction (sign) of this offset (Figure~\ref{subfig:intersection angles 2}) such that
    \begin{align*}
        t_{ij\pm} = \alpha_{ij} \pm \beta_{ij}(r) = \alpha_{ji} - \pi \mp \beta_{ji}(r) = t_{ji\mp} - \pi.
    \end{align*}
\end{proof}

The angle $\alpha_{ij}$ can alternatively be described through the 2-argument arctangent
\begin{align*}
    \arctan2\colon \R^2_{\neq (0,0)} \to (-\pi,\pi],
    (x_1, x_2)                                 \mapsto & \tan^{-1}\left(\frac{x_2}{x_1}\right)\sgn(x_1)^2                                  \\
                                                       & +\frac{\pi}{2}\left(-\sgn(x_2)^2+\sgn(x_2)+1\right)\cdot\left(1-\sgn(x_1)\right).
\end{align*}
In this instance, the angle $\alpha_{ij}$ is given by $\arctan2(x_i,x_j)$.

We now have a representation of the boundary segments and their endpoints, which is piecewise smooth w.\,r.\,t. the radius, and when integrating over the radius, the only critical points we have are when a new segment is added to the boundary or a segment vanishes.

\begin{rem*}
    If the set of pixels $\v$ only contains a single point $x_i$ and $\norm{x_i-x_j}\geq 2r$ for all $x_j\in a\setminus \v$, then the area of possible positions is simply the circle of radius $r$ around $x_i$.
    In this case, $K=1$ and the contour endpoints are $t_1=0$ and $t_1'=2\pi$.
\end{rem*}

Since for a fixed set of points the boundary segments change depending on the radius, we will now consider what happens when the radius changes.

\paragraph{Changing radii}

We have already established that the singular points in the boundary are the intersections between two circles.
What remains is to identify which of the intersections are relevant to the boundary and the boundary of which circle is the relevant one.
We will start with identifying the relevant intersection points.

\begin{lem}[Relevant intersection points]\label{lem:relevant ip}
    Let $\v\in \mathcal{P}(a)$ be a set of pixels and $\gamma_{r,\v}$ be the curve describing the boundary of possible positions $\p_{r,\v}$.
    For circle center points $x_i$ and $x_j\in a$, the intersection point $x_{ij\pm}$ is a point in $\gamma_{r,\v}$ if and only if it is in the closure of $\p_{r,\v}$, i.\,e.,
    \begin{align*}
        \norm{x_{ij\pm}-x}\leq r ~\forall x\in \v \text{ and } \norm{x_{ij\pm}-x}\geq r ~\forall x\in a\setminus\v.
    \end{align*}
    Additionally, the intersection point $x_{ij\pm}$ is a point in the boundary of the impossible leaf positions of an empty leaf if and only if
    \begin{align*}
        \norm{x_{ij\pm} -x} \geq r ~\forall x \in a.
    \end{align*}
\end{lem}

\begin{proof}
    Let $x_{ij\pm}$ be an intersection point.
    If $x_{ij\pm}$ is a singular point in $\gamma_{\v,r}$, then it, in particular, is part of the closure of $\p_{\v,r}$.
    It remains to show that any intersection point in the closure of $\p_{r,\v}$ is a point of $\gamma_{\v,r}$.
    Let $x_{ij\pm}$ be such a point.
    Assume $x_{ij\pm}$ is not part of the boundary of $\p_{\v,r}$, i.\,e., $x_{ij\pm} \in \p^\circ_{\v,r}$.
    Then $\norm{x_{ij\pm} - x}<r$ for all $x\in \v$ and $\norm{x_{ij\pm}-x}>r$ for all $x\in a\setminus \v$.
    In particular, $\norm{x_{ij\pm} - x_i}\neq r$, which contradicts the definition of $x_{ij\pm}$, and, consequently, the intersection point $x_{ij\pm}$ has to lie on the boundary of $\p_{\v,r}$.
    Analogously, for an empty leaf and $\bar{\p}_{r,\emptyset} = \{p\mid L_a(r,p)\neq \emptyset\}$, we have
    \begin{align*}
        x_{ij\pm} \in \partial\bar{\p}_{r,\emptyset} \iff x_{ij\pm} \in \partial \bigcup_{x\in a} L(r,x) \iff \norm{x_{ij\pm}-x}\geq r \,\forall x\in a.
    \end{align*}
\end{proof}

This result geometrically implies that points in $\v$ are only relevant to the boundary if they lie at extreme positions, i.\,e., they are on the boundary of the convex hull of $\v$ (Figure~\ref{subfig:irrelevant positions}).
Consequently, the maximal number of relevant intersection points for included positions is reached when all points of $\v$ lie on a circle of radius smaller than $r$ (Figure~\ref{subfig:worst positions}).
These intersection points can then only become irrelevant if an excluded position is close to said point.
In particular, no intersection point can lie in the interior of the area of possible positions.

\begin{figure*}[ht]\centering
    \begin{subfigure}[t]{0.49\textwidth}\centering
        \caption{}\label{subfig:irrelevant positions}
        \includegraphics[width=\linewidth]{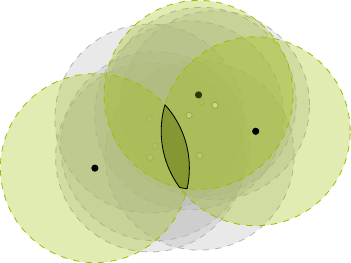}
    \end{subfigure}
    \hfill
    \begin{subfigure}[t]{0.49\textwidth}\centering
        \caption{}\label{subfig:worst positions}
        \includegraphics[width=0.8\linewidth]{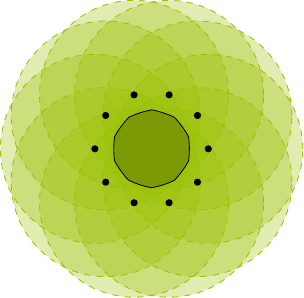}
    \end{subfigure}
    \vspace*{4pt}

    \footnotesize
    \begin{itemize*}[itemjoin=\quad]
        \item[\pp] Area of included positions
        \item[\ip] Relevant included positions
        \item[\irp] Irrelevant included positions
    \end{itemize*}
    \caption{Set of included points for ten random pixels (\subref{subfig:irrelevant positions}) and ten pixels on a circle (\subref{subfig:worst positions}) to be included.
        (\subref{subfig:irrelevant positions}) The set of included points only depends on the pixels that are extreme values in $\v$.
        (\subref{subfig:worst positions}) Due to the circular arrangement, every pixel contributes to the boundary of the intersection.
    }
\end{figure*}

Lemma~\ref{lem:relevant ip} allows us to define a mapping describing whether an intersection point $x_{ij\pm}$ is a point in $\gamma_{r,\v}$ through a boolean:
\begin{align*}
    \delta_{r,\v}: \mathcal{P}(a) \times a^2 & \to \{0,1\},                                                                                                                                                                        \\
    \delta_{r,\v}(i,j,\pm)                   & = \begin{cases}
                                                     1, & \text{if } \forall x\in \v\colon \norm{x_{ij\pm}-x} \leq r \text{ and } \forall x\in a\setminus\v \colon \norm{x_{ij\pm}-x} \geq r \\
                                                     0, & \text{else (in particular if there is no intersection point),}
                                                 \end{cases}
\end{align*}
with $0$ indicating that the intersection does not lie on $\gamma_{r,\v}$ and $1$ indicating the intersection is a point in the boundary.

We can then also reconstruct for which sets of pixels $\v$ an intersection point is a boundary point of $\gamma_{r,\v}$ as described by \citet{Pitkow2010}.
The intersection point $x_{ij\pm}$ contributes to the area of possible positions for four pixel sets given by
\begin{align*}
    \v = \left\{x_k \in a \mid \bar{\delta}_r(k;i,j,\pm) = 1\right\}
\end{align*}
where
\begin{align*}
    \bar{\delta}_r(k;i,j,\pm)
    =\begin{cases}
         0               & \text{for }\norm{x_{ij\pm} - x_k} \geq r, k\neq i,j \\
         1               & \text{for }\norm{x_{ij\pm} - x_k} < r, k\neq i,j    \\
         0 \text{ or } 1 & \text{for } k=i,j.
     \end{cases}
\end{align*}
Note that with this criterion an intersection point might lie on the boundary without being a singular point, which would result in breaking down the integral into more segments than necessary, but would still yield the same result.

Since we derived how the boundary of the set of possible positions is composed for a fixed radius, we can consider how this boundary changes depending on the radius.

\begin{lem}[Constant number of subcurves]
    \label{lem:constant number of subcurves}
    Let $\v\in \mathcal{P}(a)$ be a set of pixels and $\p_{r,\v}$ the set of possible positions given $r$.
    Then the range of radii $[r_{\min},r_{\max}]$ can be partitioned into subintervals $[r_l,r_{l+1})$ with $r_1=r_{\min}$ and $r_L=r_{\max}$ such that the number of smooth subcurves $K_l$ in $\gamma_{r,\v}$ is constant for all $r\in [r_l,r_{l+1})$.
    These critical radii are given by
    \begin{align*}
        r_{ij}^{\ast}
         & = \frac{\norm{x_i-x_j}}{2}                                                                                                                                     &
        r_{ijk}^{\ast}
         & = \frac{\norm{x_i-x_j}\cdot\norm{x_j-x_k}\cdot\norm{x_i-x_k}}  {4\sqrt{s_{ijk}(s_{ijk} - \norm{x_i-x_j})(s_{ijk} - \norm{x_j-x_k})(s_{ijk} - \norm{x_i-x_k})}}   \\
    \end{align*}
    for all pairwise distinct $x_i$, $x_j$, and $x_k \in a$.
\end{lem}

\begin{proof}
    From Lemma~\ref{lem:relevant ip} it follows directly that the number of singular points (and therefore the number of subcurves $K$) for a given set $\v$ only depends on the radius.
    In particular, the singular points only change when a new intersection point is created or when $\norm{x_{ij\pm}-x_k}=r$ for $x_i$, $x_j$, $x_k \in a$ with pairwise different $i$, $j$, $k$.
    New intersection points are generated every time the radius surpasses half the distance between two points in $a$ (Figure~\ref{subfig:kissing point}), resulting in the critical radii $r_{ij}^{\ast}$ above.
    Since $\norm{x_{ij\pm} - x_i}=\norm{x_{ij\pm}-x_j}=r$ always holds by definition, if additionally $\norm{x_{ij\pm}-x_k}=r$ then all three points lie on a circle of radius $r$ centered around $x_{ij\pm}$, i.\,e., the circum circle of the triangle with vertices $x_i$, $x_j$, and $x_k$.
    Hence, the corresponding radius can be determined by computing the radius $r_{ijk}^{\ast}$ of the circum circle.
    Consequently, the number of subcurves for any given $\p_{\v,r}$ is constant for each interval of adjacent critical radii.
    In other words, let $\{r_l^{\ast}\}_{l\in[L]}$ be the ordered sequence of critical radii; then $K_l$ is constant for all $r\in [r_l^{\ast},r_{l+1}^{\ast})$ for $l=1,\dots,L-1$.
\end{proof}

Geometrically, this implies a change in the shape of the area of possible positions in the following cases:
\begin{enumerate*}[label= (\arabic*)]
    \item When two circles start intersecting each other (Figure~\ref{subfig:kissing point}),
    \item When an intersection point moves out of (Figure~\ref{subfig:intersection moving out}) or into (Figure~\ref{subfig:intersection moving in}) another circle.
\end{enumerate*}
In particular, the area of possible positions only changes its shape until the radius has surpassed all critical radii for the set of pixels.

\begin{figure*}[ht]\centering
    \begin{subfigure}[t]{\textwidth}\centering
        \caption{}\label{subfig:kissing point}
        \includegraphics[width=0.8\linewidth]{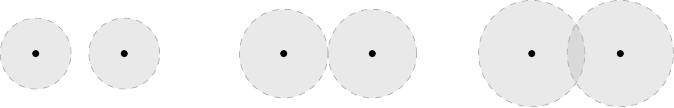}
    \end{subfigure}
    \begin{subfigure}[t]{\textwidth}\centering
        \caption{}\label{subfig:intersection moving out}
        \includegraphics[width=0.8\linewidth]{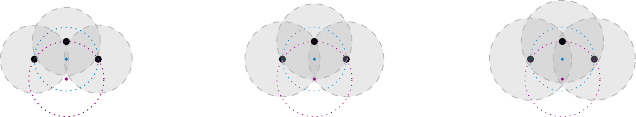}
    \end{subfigure}
    \begin{subfigure}[t]{\textwidth}\centering
        \caption{}\label{subfig:intersection moving in}
        \includegraphics[width=0.8\linewidth]{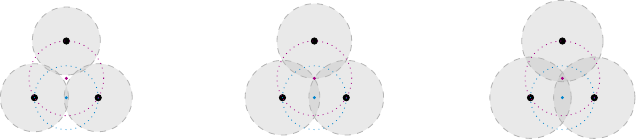}
    \end{subfigure}
    \begin{tabularx}{0.95\linewidth}{XXX}
        \centering $r < r_{\text{crit}}$ & \centering $r = r_{\text{crit}}$ & \centering $r > r_{\text{crit}}$
    \end{tabularx}
    \vspace*{4pt}

    \footnotesize
    \begin{itemize*}[itemjoin=\quad]
        \item[$\textcolor{tud-10b}{\cdots}$] Three-point circumcircle
        \item[$\textcolor{tud-2b}{\cdots}$] Two-point circumcircle
    \end{itemize*}
    \caption{
        Change of set overlap with increasing radius for different arrangements of center points.
        (\subref{subfig:kissing point}) Creation of new intersection as the radius increases.
        (\subref{subfig:intersection moving out}) If a point is within the circumcircle of two other points (\textcolor{tud-2b}{blue}) and $r$ is smaller than the circumcircle radius of all three points (\textcolor{tud-10b}{purple}), then the intersection of all three circles is the same as the intersection of the two other points.
        (\subref{subfig:intersection moving in}) If a point is outside the circumcircle of two other points (\textcolor{tud-2b}{blue}) and $r$ is smaller than the circumcircle radius of all three points (\textcolor{tud-10b}{purple}), then the intersection of all three circles is empty.
        When $r$ is larger than the circumcircle radius of all three points as in (\subref{subfig:intersection moving out}) and (\subref{subfig:intersection moving in}), then the intersection of all three circles always is non-empty and a proper subset of any pairwise intersection.
    }
\end{figure*}

With these preliminaries we can compute the actual probability of leaves:

\begin{thm}[Leaf probability]
    \label{thm:leaf prob}
    Let $\v\in \mathcal{P}(a)$ be a non-empty set of pixels.
    The leaf probability $\P(L_a=\v)$ can be written as
    \begin{align*}
        \P(L_a=\v)
        = \frac{1}{\abs{B}}\sum_{l=1}^L \left(\int_{r_l}^{r_{l+1}} f_X(r) \sum_{k \in K_l} \left(\frac{1}{2} \int_{t_{k}}^{t_{k}'} \skp{\gamma_{r,\v}^k(t),\vec{n}(\gamma_{r,\v}^k(t))}\norm{\dot{\gamma}_{r,\v}^k(t)}\mathbf{d}t\right)\mathbf{d}r\right)
    \end{align*}
    where $\gamma_{r,\v}$ is the curve describing the piecewise smooth boundary of $\p_{r,\v}$, $\gamma_{r,\v}^k$ are its smooth subcurves and $\vec{n}$ is its unit normal vector.
    The function $f_{X}(r)$ is the density function of the radius distribution.

    The probability of a non-empty leaf is given by
    \begin{align*}
        \P(L_a\neq\emptyset)
        = \frac{1}{\abs{B}}\sum_{l=1}^L\left(\int_{r_l}^{r_{l+1}}f_X(r)\sum_{k\in K_l}\left(\frac{1}{2}\int_{t_k}^{t_k'}\skp{\gamma_{r,\emptyset}^k(t),\vec{n}(\gamma_{r,\emptyset}^k(t))}\norm{\dot{\gamma}_{r,\emptyset}^k(t)}\mathbf{d}t\right)\mathbf{d}r\right)
    \end{align*}
    where $\gamma_{r,\emptyset}$ is the piecewise smooth boundary of $\bar{\p}_{r,\emptyset}$.
\end{thm}

According to this theorem, we can compute the leaf probability (Equation~\ref{eq:leaf prob}) by splitting the radius integral into intervals ($[r_l,r_{l+1}]$) for which the shape of the area of (im)possible positions does not change.
Each segment of the area with smooth boundary is then given by integrating from one singular point to the next ($[t_k,t_k']$).

\begin{proof}
    We recall that by the law of total probability the leaf probability is given by
    \begin{align*}
        \P(L_a=\v)
         & = \int_{\R} f_X(r)\P(L_a(r) = \v\mid r)\mathbf{d}r                                \\
         & = \int_{\R} f_X(r)\int_{\R^2} f_P(p)\P(L_a(r,p)=\v\mid r,p)\mathbf{d}p\mathbf{d}r
    \end{align*}
    where $f_P(p)$ is the density function of the position distribution, $f_X(r)$ is the radius density function, $L_a(r)$ is a leaf of radius $r$ with random position, and $L_a(r,p)$ is a leaf with radius $r$ at $p$.
    In particular, the probability $\P(L_a(r,p)=\v\mid r,p)$ then can only take binary values $\{0,1\}$ as the leaf $L_a(r,p)$ either equals the set of pixels $\v \in \mathcal{P}(a)$ or it does not.

    Since $f_X(r)$ is by definition zero outside $[r_{\min}, r_{\max}]$ and $f_P(p)$ is zero outside $B$, we have
    \begin{align*}
        \P(L_a=\v) = \int_{r_{\min}}^{r_{\max}} f_X(r)\int_{B} f_P(p)\P(L_a(r,p)=\v\mid r,p)\mathbf{d}p\mathbf{d}r.
    \end{align*}
    Since the positions are uniformly distributed on $B$, the density function $f_P$ is piecewise constant with $\frac{1}{\abs{B}}$ inside $B$ and $0$ everywhere else.
    Therefore,
    \begin{align*}
        \P(L_a=\v) = \frac{1}{\abs{B}}\int_{r_{\min}}^{r_{\max}} f_X(r)\int_{B} \P(L_a(r,p)=\v\mid r,p)\mathbf{d}p\mathbf{d}r.
    \end{align*}

    Since $\P(L_a(r,p)=\v\mid r,p)=1$ if and only if $p\in \p_{\v,r}\subset B$ (Lemma~\ref{lem:area of possible positions}), we can further write this probability as
    \begin{align*}
        \P(L_a=\v) = \frac{1}{\abs{B}}\int_{r_{\min}}^{r_{\max}} f_X(r)\int_{\p_{r,\v}}1\mathbf{d}p\mathbf{d}r.
    \end{align*}

    Utilizing that the vector field $\vec{F}\colon \R^2 \to \R^2, p \mapsto \frac{1}{2} p$  has divergence
    \begin{align*}
        \div\vec{F}
        = \frac{1}{2}\left(\frac{\partial F^1}{\partial p_1} + \frac{\partial F^2}{\partial p_2}\right)
        = 1
    \end{align*}
    we can apply Gauss's theorem \citep[section~10.4]{Steinmetz2024} to translate the position integral into a contour integral and get
    \begin{align*}
        \P(L_a=\v)
         & = \frac{1}{\abs{B}}\int_{r_{\min}}^{r_{\max}} f_X(r)\int_{\p_{r,\v}}\div\vec{F}(p)\mathbf{d}p\mathbf{d}r                                         \\
         & = \frac{1}{\abs{B}}\int_{r_{\min}}^{r_{\max}} f_X(r)\oint_{\partial \p_{r,\v}}\skp{\vec{F}(p),\vec{n}(p)}\mathbf{d}\partial \p_{r,\v}\mathbf{d}r
    \end{align*}
    where $\vec{n}$ is the unit normal vector field of the boundary.

    Following Lemma~\ref{lem:boundary parametrization}, we can parametrize the boundary $\partial\p_{r,\v}$ by a closed curve $\gamma_{r,\v}\colon T \to \partial \p_{r,\v} \subset \R^2$ and apply the definition of line integrals to get
    \begin{align*}
        \P(L_a=\v)
         & = \frac{1}{\abs{B}}\int_{r_{\min}}^{r_{\max}} f_X(r) \oint_{\gamma_{r,\v}} \skp{\vec{F},\vec{n}}\mathbf{d}s\mathbf{d}r                                                          \\
         & = \frac{1}{\abs{B}}\int_{r_{\min}}^{r_{\max}} f_X(r) \oint_{T} \skp{\vec{F}(\gamma_{r,\v}(t)),\vec{n}(\gamma_{r,\v}(t))}\norm{\dot{\gamma}_{r,\v}(t)} \mathbf{d}t\mathbf{d}r    \\
         & = \frac{1}{\abs{B}}\int_{r_{\min}}^{r_{\max}} f_X(r) \oint_T \frac{1}{2} \skp{\gamma_{r,\v}(t),\vec{n}(\gamma_{r,\v}(t))} \norm{\dot{\gamma}_{r,\v}(t)} \mathbf{d}t\mathbf{d}r.
    \end{align*}

    Let $\{[r_l,r_{l+1})\}_{l\in[L-1]}$ be the sequence of subintervals of $[r_{\min},r_{\max})$ such that the number of smooth subcurves $K_l$ in $\gamma_{r,\v}$ is constant (Lemma~\ref{lem:constant number of subcurves}) and let $\gamma_{r,\v}^k$ denote the $k$th smooth segment of $\gamma_{r,\v}$.
    We can then segment the integrals into sums over segments yielding
    \begin{align*}
        \P(L_a=\v) = \frac{1}{\abs{B}}\sum_{l=1}^L\int_{r_l}^{r_{l+1}} f_X(r) \sum_{k=1}^{K_l} \frac{1}{2} \int_{t_k}^{t_{k}'} \skp{\gamma^k_{r,\v}(t),\vec{n}(\gamma^k_{r,\v}(t))}\norm{\dot{\gamma}^k_{r,\v}(t)} \mathbf{d}t\mathbf{d}r.
    \end{align*}
\end{proof}

\begin{rem*}
    The above theorem can be proven in a more general form such that the derivation only requires shape positions uniformly distributed in $B$ and some smooth leaf shape with a continuous size distribution.
\end{rem*}

Using the above result and the knowledge about the critical radii (Lemma~\ref{lem:constant number of subcurves}), relevant intersection points (Lemma~\ref{lem:relevant ip}), and the oriented boundary parametrization (Lemma~\ref{lem:boundary parametrization}), we can derive the following constructive configuration of the leaf probability.

\begin{prop}[Constructive computation of leaf probability]\label{prop:constructive computation}
    Let $\v \subset a$ be a non-empty set of pixels and $\{r_l\}_{l\in[L]}$ the sequence of critical radii (Lemma~\ref{lem:constant number of subcurves}).
    Let $I_r$ be the set of intersection points $x_{ij\pm}$ for $x_i$, $x_j \in a$ and circle radius $r$.
    The leaf probability $\P(L_a=\v)$ can be computed through
    \begin{align*}
          & \P(L_a = \v)                                                                                                                                                                                                \\
        = & \frac{1}{\abs{B}(r_{\min}^{-2}-r_{\max}^{-2})}\sum_{l=1}^L \left(\sum_{x_{ij\pm}\in I_r} \delta_{r_l,\v}(x_{ij\pm})\cdot c_{\v}(x_{ij\pm})\cdot \left(b(r_{l+1};x_{ij\pm}) - b(r_l;x_{ij\pm})\right)\right)
    \end{align*}
    with boundary orientation
    \begin{align*}
        c_{\v}(x_{ij\pm}) = \begin{cases} \mp1, & \text{if either } x_i\in\v \text{ or } x_j\in \v, \\ \pm1, &\text{else,} \end{cases}
    \end{align*}
    singular point boolean as defined in Lemma~\ref{lem:relevant ip}
    \begin{align*}
        \delta_{\v,r}(x_{ij\pm}) = \begin{cases}
                                       1, & \text{if }\forall x\in\v\colon \norm{x_{ij\pm}-x}\leq r \text{ and }\forall x\in a\setminus\v\colon\norm{x_{ij\pm}-x}\geq r \\
                                       0, & \text{else,}
                                   \end{cases}
    \end{align*}
    and
    \begin{align*}
        b(r;x_{ij\pm}) =
         &
        \alpha_{ij}\log(r)
        \mp  \left(\frac{1}{2}\operatorname{Cl}_2(2\beta_{ij}(r)+\pi)+\beta_{ij}(r)\log\left(\frac{\norm{x_i-x_j}}{r}\right)\right)                                                                                                 \\
         & - \frac{\norm{x_i-x_j}}{4r^2} \skp{x_i,\vec{n}\left(\alpha_{ij}-\frac{\pi}{2}\right)} \pm \frac{2\skp{x_i,\vec{n}(\alpha_{ij})}}{\norm{x_i-x_j}}\left(\beta_{ij}(r) - \frac{1}{2}\sin\left(2\beta_{ij}(r)\right)\right),
    \end{align*}
    where $\alpha_{ij}$ and $\beta_{ij}(r)$ are the angles defined in Lemma~\ref{lem:singular points} and $\operatorname{Cl}_2$ is the Clausen function of order $2$ defined as
    \begin{align*}
        \operatorname{Cl}_2(\theta) = -\int_0^{\theta} \log\left(\abs{2\sin\left(\frac{t}{2}\right)}\right)\mathbf{d}t.
    \end{align*}
    The leaf probability of any non-empty leaf $\P(L_a\neq \emptyset)$ can be computed in the same manner using the boundary orientation $c_{\emptyset}(x_{ij\pm})=\mp 1$ and the singular point boolean
    \begin{align*}
        \delta_{\emptyset,r}(x_{ij\pm}) = \begin{cases}
                                              1, & \text{if }\forall x \in a\colon \norm{x_{ij\pm}-x}\geq r, \\
                                              0, & \text{else.}
                                          \end{cases}
    \end{align*}
\end{prop}

The main idea to derive this result from the previous theorem is to use the fundamental theorem of calculus and boundary parametrization from Lemma~\ref{lem:boundary parametrization} to transform the area integral into a sum which we can then further integrate.

\begin{proof}
    By Theorem~\ref{thm:leaf prob}, the leaf probability is given by
    \begin{align*}
        \P(L_a =\v) = \frac{1}{\abs{B}}\sum_{l=1}^L \left(\int_{r_l}^{r_{l+1}}f_X(r)\sum_{k = 1}^{K_l}\left(\frac{1}{2}\int_{t_k}^{t_k'}\skp{\gamma_{r,\v}^k(t),\vec{n}\left(\gamma_{r,\v}^k(t)\right)}\norm{\dot{\gamma}_{r,\v}^k(t)}\mathbf{d}t\right)\mathbf{d}r\right).
    \end{align*}
    Let us first consider the area given by one boundary segment through the antiderivative
    \begin{align*}
        A_k(t;r) = \frac{1}{2} \int \skp{\gamma_{r,\v}^k(t),\vec{n}(\gamma_{r,\v}(t))}\norm{\dot{\gamma}_{r,\v}^k(t)}\mathbf{d}t.
    \end{align*}
    According to Lemma~\ref{lem:boundary parametrization}, for any $\gamma_{r,\v}^k$ there exists an $x_{i_k} \in a$ such that $\gamma_{r,\v}^k(t) = x_{i_k} + r \cdot \begin{pmatrix}
            \cos(t) \\ \sin(t)
        \end{pmatrix}$ with outward pointing unit normal vector
    \begin{align*}
        \vec{n}\left(\gamma_{r,\v}^k(t)\right) = c_{\v}(x_{i_k})\cdot\normal{t} \quad\text{where}\quad c_{\v}(x_{i_k})=\begin{cases}
                                                                                                                           1,  & \text{if }x_{i_k}\in \v, \\
                                                                                                                           -1, & \text{else}.
                                                                                                                       \end{cases}
    \end{align*}
    Consequently, the arc's length $\norm{\dot{\gamma}_{r,\v}^k (t)}$ is
    \begin{align*}
        \norm{\dot{\gamma}_{r,\v}^k  (t)}  = \norm{r \cdot \begin{pmatrix}-\sin(t) \\ \cos(t)\end{pmatrix}}  = \sqrt{r^2(\sin^2(t) + \cos^2(t))} = r.
    \end{align*}
    Hence, the antiderivative $A_k(t;r)$ is given by
    \begin{align*}
        A_k(t;r)
         & = \frac{1}{2} \int r \cdot \skp{x_{i_k} + r \cdot \normal{t}, c_{\v}(x_{i_k})\cdot\normal{t}} \mathbf{d}t               \\
         & = \frac{1}{2} \int c_{\v}(x_{i_k})\cdot r^2 + \skp{r\cdot x_{i_k},\vec{n}(t)}\mathbf{d}t                                \\
         & = \frac{c_{\v}(x_{i_k})}{2} \left(r^2t + \skp{r\cdot x_{i_k}, \begin{pmatrix} \sin(t) \\ -\cos(t) \end{pmatrix}}\right)
    \end{align*}

    We can then write the leaf probability as
    \begin{align*}
        \P(L_a =\v)
         & = \sum_{l=1}^L\left(\int_{r_l}^{r_{l+1}}f_X(r)\sum_{k=1}^{K_l} \left(A_k(t_k';r) - A_k(t_k;r)\right)\mathbf{d}r\right)
    \end{align*}

    Since our end points $t_k$ depend on the radius, we use Lemma~\ref{lem:singular points} to rewrite any $t_k$ as
    \begin{align*}
        t_k
        = t_{ij\pm}
        = \alpha_{ij} \pm \beta_{ij}(r)
         & = s_{ij} \cdot \cos^{-1}\left(\frac{x_{1i_k}-x_{1j_k}}{\norm{x_{i_k} - x_{j_k}}}\right) \pm \cos^{-1}\left(\frac{\norm{x_{i_k} - x_{j_k}}}{2r}\right)
    \end{align*}
    for some $x_{i_k}$, $x_{j_k} \in a$.
    Moving the sign $c_{\v}$ outside the term and reparameterizing $A_k$ with respect to $x_{ij\pm}$, we can identify all dependencies of $r$ in $A_k$:
    \begin{align*}
        A(x_{ij\pm};r)
        = & \frac{r^2}{2}\cdot \left(\alpha_{ij} \pm \beta_{ij}(r)\right)
        + \frac{r}{2}\skp{x_i,\begin{pmatrix}
                                      \sin(\alpha_{ij} \pm \beta_{ij}(r)) \\ -\cos(\alpha_{ij} \pm \beta_{ij}(r))
                                  \end{pmatrix}}                                                                                                             \\
        = & \frac{r^2}{2}\alpha_{ij}
        \pm \frac{r^2}{2}\cos^{-1}\left(\frac{\norm{x_i-x_j}}{2r}\right)                                                                                                                                                                                         \\
          & + \frac{r}{2}\skp{x_i,\begin{pmatrix}\sin(\alpha_{ij})\cos(\beta_{ij}(r))\pm\cos(\alpha_{ij})\sin(\beta_{ij}(r))\\-\cos(\alpha_{ij})\cos(\beta_{ij}(r))\pm\sin(\alpha_{ij})\sin(\beta_{ij}(r))\end{pmatrix}} \\
        = & \frac{r^2}{2}\alpha_{ij}
        \pm \frac{r^2}{2}\cos^{-1}\left(\frac{\norm{x_i-x_j}}{2r}\right)                                                                                                                                                                                         \\
          & + \frac{r}{2} \cos(\beta_{ij}(r))\skp{x_i,\begin{pmatrix}\sin(\alpha_{ij})\\-\cos(\alpha_{ij})\end{pmatrix}}
        \pm \frac{r}{2} \sin(\beta_{ij}(r))\skp{x_i,\begin{pmatrix}\cos(\alpha_{ij})\\\sin(\alpha_{ij})\end{pmatrix}}                                                                                                                                            \\
        = & \frac{r^2}{2}\alpha_{ij}
        \pm \frac{r^2}{2}\cos^{-1}\left(\frac{\norm{x_i-x_j}}{2r}\right)                                                                                                                                                                                         \\
          & + \frac{\norm{x_i-x_j}}{4}\skp{x_i,\vec{n}\left(\alpha_{ij}-\frac{\pi}{2}\right)} \pm \frac{r}{2} \sqrt{1-\left(\frac{\norm{x_i-x_j}}{2r}\right)^2}\skp{x_i,\vec{n}(\alpha_{ij})}.                                                                   \\
    \end{align*}

    According to Lemma~\ref{lem:relevant ip}, the mapping
    \begin{align*}
        \delta_{\v,r}(x_{ij\pm})=\begin{cases}
                                     1, & \text{if }\forall x\in\v\colon \norm{x_{ij\pm}-x}\leq r \text{ and } \forall x\in a\setminus\v\colon\norm{x_{ij\pm}-x}\geq r \\
                                     0, & \text{else}
                                 \end{cases}
    \end{align*}
    identifies the singular points in $\gamma_{r,\v}$.
    While the sign $c_{\v}(x_i)$ encodes the direction of the unit normal vector, the sign $c_{\v}(x_j)$ encodes whether the intersection point $x_{ij\pm}$ is the start or the end point of the integral (Figure~\ref{fig:boundary orientation}).

    \begin{figure*}[ht]\centering
        \begin{subfigure}[t]{0.24\textwidth}\raggedleft
            \caption{}\label{subfig:intersection sign 1}
            \includegraphics[width=0.9\linewidth]{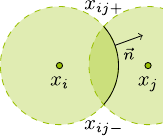}
        \end{subfigure}
        \hfill
        \begin{subfigure}[t]{0.24\textwidth}\raggedleft
            \caption{}\label{subfig:intersection sign 2}
            \includegraphics[width=0.9\linewidth]{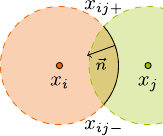}
        \end{subfigure}
        \hfill
        \begin{subfigure}[t]{0.24\textwidth}\raggedleft
            \caption{}\label{subfig:intersection sign 3}
            \includegraphics[width=\linewidth]{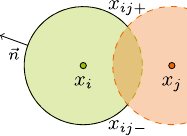}
        \end{subfigure}
        \hfill
        \begin{subfigure}[t]{0.24\textwidth}\raggedleft
            \caption{}\label{subfig:intersection sign 4}
            \includegraphics[width=0.9\linewidth]{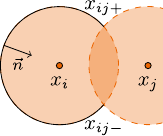}
        \end{subfigure}
        \vspace*{4pt}

        \footnotesize
        \begin{itemize*}[itemjoin=\quad]
            \item[\ip] Included positions
            \item[\ep] Excluded positions
        \end{itemize*}
        \caption{
            Case combinations for boundary orientation sign.
            (\subref{subfig:intersection sign 1}) Outward pointing normal vector is pointing away from the reference point $x_i$ ($c_{\v}(x_i) = 1$) on the inner boundary segment ($c_{\v}(x_j) = 1$).
            (\subref{subfig:intersection sign 2}) Outward pointing normal vector is pointing towards the reference point $x_i$ ($c_{\v}(x_i) = -1$) on the inner boundary segment ($c_{\v}(x_j) = 1$).
            (\subref{subfig:intersection sign 3}) Outward pointing normal vector is pointing away from the reference point $x_i$ ($c_{\v}(x_i) = 1$) on the outer boundary segment ($c_{\v}(x_j) = -1$).
            (\subref{subfig:intersection sign 4}) Outward pointing normal vector is pointing towards the reference point $x_i$ ($c_{\v}(x_i) = -1$) on the outer boundary segment ($c_{\v}(x_j) = -1$).
        }\label{fig:boundary orientation}
    \end{figure*}

    We can therefore use the combined sign
    \begin{align*}
        c_{\v}(x_{ij\pm})= (\pm 1)\cdot c_{\v}(x_i) \cdot c_{\v}(x_j) = \begin{cases}
                                                                            \mp 1, & \text{if either }x_i\in\v \text{ or } x_j\in\v, \\
                                                                            \pm 1, & \text{ else,}
                                                                        \end{cases}
    \end{align*}
    to write the leaf probability as
    \begin{align*}
        \P(L_a=\v) = \frac{1}{B}\sum_{l=1}^L\left(\int_{r_l}^{r_{l+1}}f_X(r)\sum_{x_{ij\pm}\in I_r}\delta_{r,\v}(x_{ij\pm})\cdot c_{\v}(x_{ij\pm})\cdot A(x_{ij\pm};r)\mathbf{d}r\right).
    \end{align*}
    Since $c_{\v}$, $\delta_{r,\v}$, and the set of intersection points $I_r$ are constant with each radius interval $(r_l,r_{l+1})$, we can move them together with the sum outside the integral
    \begin{align*}
        \P(L_a=\v) = \frac{1}{\abs{B}}\sum_{l=1}^L\left(\sum_{x_{ij\pm}\in I_{r_{l+\frac{1}{2}}}}\delta_{r_{l+\frac{1}{2}},\v}(x_{ij\pm})\cdot c_{\v}(x_{ij\pm})\int_{r_l}^{r_{l+1}}f_X(r)\cdot A(x_{ij\pm};r)\mathbf{d}r\right)
    \end{align*}
    where $r_{l+\frac{1}{2}}=\frac{1}{2}(r_l+r_{l+1})$.

    The remaining step is now to solve the radius integral
    \begin{align*}
        \int f_X(r) A(x_{ij\pm};r) \mathbf{d}r.
    \end{align*}
    Recall that the density function $f_X(r)$ is given by
    \begin{align*}
        f_X(r) = \frac{2r^{-3}}{r_{\min}^{-2}-r_{\max}^{-2}} \quad\text{for }r \in [r_{\min},r_{\max}].
    \end{align*}
    Moving the constant factor $\frac{1}{r_{\min}^{-2}-r_{\max}^{-2}}$ outside the integral, the antiderivative $b$ can be written as
    \begin{align*}
        b(r;x_{ij\pm}) =
          & \int 2r^{-3} A(x_{ij\pm};r) \mathbf{d}r                                                              \\
        = & \int r^{-1}\alpha_{ij} \pm r^{-1}\cos^{-1}\left(\frac{\norm{x_i-x_j}}{2r}\right)                     \\
          & + \frac{\norm{x_i-x_j}r^{-3}}{2}\skp{x_i,\vec{n}\left(\alpha_{ij}-\frac{\pi}{2}\right)}              \\
          & \pm r^{-2}\sqrt{1-\left(\frac{\norm{x_i-x_j}}{2r}\right)^2}\skp{x_i,\vec{n}(\alpha_{ij})}\mathbf{d}r \\
    \end{align*}
    Let us consider each summand of this integral separately for simplification.
    The first summand simply gives
    \begin{align*}
        \int \frac{\alpha_{ij}}{r}\mathbf{d}r
        = \log(r)\alpha_{ij}.
    \end{align*}
    The second term is more complex (figuratively and literally).
    For the initial integral, we need the dilogarithm $\operatorname{Li}_2$ and use that
    \begin{align*}
          & \int x^{-1}\cos^{-1}\left(\frac{a}{x}\right)\mathbf{d}x                                               \\
        = & \frac{\mathrm{i}}{2}\operatorname{Li}_2\left(-e^{2\mathrm{i}\cos^{-1}\left(\frac{a}{x}\right)}\right)
        +\frac{\mathrm{i}}{2}\cos^{-1}\left(\frac{a}{x}\right)^2
        -\cos^{-1}\left(\frac{a}{x}\right)\log\left(1+e^{2\mathrm{i}\cos^{-1}(\frac{a}{x})}\right)
    \end{align*}
    In our setting, we have $a=\frac{\norm{x_i-x_j}}{2}$ and $x=r$ such that $\cos^{-1}\left(\frac{a}{x}\right) = \cos^{-1}\left(\frac{\norm{x_i-x_j}}{2r}\right)=\beta_{ij}(r)$.
    Therefore, the integral is given by
    \begin{align*}
        \int r^{-1}\cos^{-1}\left(\frac{\norm{x_i-x_j}}{2r}\right) \mathbf{d}r
         & = \frac{\mathrm{i}}{2}\operatorname{Li}_2\left(-e^{2\mathrm{i}\beta_{ij}(r)}\right) + \frac{\mathrm{i}}{2}\beta_{ij}^2(r) - \beta_{ij}(r)\log\left(1+e^{2\mathrm{i}\beta_{ij}(r)}\right).
    \end{align*}
    The imaginary part of this antiderivative is constant and therefore vanishes for any definite integral.
    Consequently, it suffices to consider the real part of the integral.
    Additionally, we can write the negative exponential function as phase shift of $\pi$ instead ($-e^{\mathrm{i}\theta}=e^{\mathrm{i}(\theta+\pi)}$) such that we get
    \begin{align*}
          & \int r^{-1}\cos^{-1}\left(\frac{\norm{x_i-x_j}}{2r}\right) \mathbf{d}r                                                                                                                                                    \\
        = & \operatorname{Re}\left(\frac{\mathrm{i}}{2}\operatorname{Li}_2\left(e^{\mathrm{i}(2\beta_{ij}(r)+\pi)}\right) + \frac{\mathrm{i}}{2}\beta_{ij}^2(r) - \beta_{ij}(r)\log\left(1+e^{2\mathrm{i}\beta_{ij}(r)}\right)\right) \\
        = & -\frac{1}{2}\operatorname{Im}\left(\operatorname{Li}_2\left(e^{\mathrm{i}(2\beta_{ij}(r)+\pi)}\right)\right) - \beta_{ij}(r)\operatorname{Re}\left(\log\left(1+e^{2\mathrm{i}\beta_{ij}(r)}\right)\right).
    \end{align*}
    We then utilize that $\operatorname{Im}(\operatorname{Li}_2(e^{i\theta})) = \operatorname{Cl}_2(\theta)$ where $\operatorname{Cl}_2$ is the Clausen function of order~2:
    \begin{align*}
        \int r^{-1}\cos^{-1}\left(\frac{\norm{x_i-x_j}}{2r}\right) \mathbf{d}r
         & = -\frac{1}{2}\operatorname{Cl}_2(2\beta_{ij}(r)+\pi) - \beta_{ij}(r)\log\left(\frac{\norm{x_i-x_j}}{r}\right).
    \end{align*}
    The third term is again simpler as it is essentially a polynomial w.\,r.\,t. the radius and therefore
    \begin{align*}
        \int \frac{\norm{x_i-x_j}}{2r^3}\skp{x_i,\vec{n}\left(\alpha_{ij}-\frac{\pi}{2}\right)} \mathbf{d}r
        = - \frac{\norm{x_i-x_j}}{4r^2}\skp{x_i,\vec{n}\left(\alpha_{ij}-\frac{\pi}{2}\right)}.
    \end{align*}
    Finally, for the last summand we substitute with $r = g(\theta) = \frac{\norm{x_i-x_j}}{2\cos(\theta)}$ to solve the integral. This substitution function has the inverse function $\theta = g^{-1}(r) = \beta_{ij}(r)$ and the derivative is $g'(\theta) = \frac{\norm{x_i-x_j}\sin(\theta)}{2\cos^2(\theta)}$. Integration by substitution yields
    \begin{align*}
          & \int \frac{\skp{x_i,\vec{n}(\alpha_{ij})}}{r^2} \sqrt{1-\left(\frac{\norm{x_i-x_j}}{2r}\right)^2} \mathbf{d}r                                                         \\
        = & \int \frac{4\cos^2(\theta)\skp{x_i,\vec{n}(\alpha_{ij})}}{\norm{x_i-x_j}^2} \sqrt{1-\cos^2(\theta)}\frac{\norm{x_i-x_j}\sin(\theta)}{2\cos^2(\theta)}\mathbf{d}\theta \\
        = & \int \frac{2\skp{x_i,\vec{n}(\alpha_{ij})}\sin^2(\theta)}{\norm{x_i-x_j}}\mathbf{d}\theta
        = \int \frac{2\skp{x_i,\vec{n}(\alpha_{ij})}}{\norm{x_i-x_j}}\left(1-\cos(2\theta)\right)\mathbf{d}\theta                                                                 \\
        = & \frac{2\skp{x_i,\vec{n}(\alpha_{ij})}}{\norm{x_i-x_j}}\left(\theta - \frac{1}{2}\sin(2\theta)\right)                                                                  \\
        = & \frac{2\skp{x_i,\vec{n}(\alpha_{ij})}}{\norm{x_i-x_j}}\left(\beta_{ij}(r) - \frac{1}{2}\sin\left(2\beta_{ij}(r)\right)\right)                                         \\
    \end{align*}
    Taking these results together we get
    \begin{align*}
        b(r;x_{ij\pm}) = &
        \alpha_{ij}\log(r)
        \mp  \left(\frac{1}{2}\operatorname{Cl}_2(2\beta_{ij}(r)+\pi)+\beta_{ij}(r)\log\left(\frac{\norm{x_i-x_j}}{r}\right)\right)                                                                                                                 \\
                         & - \frac{\norm{x_i-x_j}}{4r^2} \skp{x_i,\vec{n}\left(\alpha_{ij}-\frac{\pi}{2}\right)} \pm \frac{2\skp{x_i,\vec{n}(\alpha_{ij})}}{\norm{x_i-x_j}}\left(\beta_{ij}(r) - \frac{1}{2}\sin\left(2\beta_{ij}(r)\right)\right).
    \end{align*}
    The derivation for $\P(L_a\neq \emptyset)$ follows analogously, with $\delta_{\emptyset,r}(x_{ij\pm})$ following from Lemma~\ref{lem:relevant ip}.
    While the direction of the unit normal vector is always positive for the area of impossible positions for the empty leaf, we integrate from positive to negative intersection points.
    Consequently, the boundary orientation sign is given by
    \begin{align*}
        c_{\emptyset}(x_{ij\pm}) = \mp 1.
    \end{align*}
\end{proof}

We now have a way to compute the leaf probability for all cases where the boundary of possible positions has some singular points.
It remains to consider the special case when $\v$ only contains a single element and the radius is smaller than the smallest critical radius of this point.

\begin{cor}\label{cor:area before first cr}
    Let $x_i \in a$ and $r_{\min} < r^{\ast}_i = \min_{x_j \in a, j \neq i} \frac{\norm{x_i - x_j}}{2}$. Then the first summand of the leaf probability $\P(L_a=\{x_i\})$ is given by
    \begin{align*}
        \frac{1}{\abs{B}}\int_{r_{\min}}^{r_i^{\ast}} f_X(r) A(r)\mathbf{d}r = \frac{2\pi}{\abs{B}(r_{\min}^{-2}-r_{\max}^{-2})} \log\left(\frac{r_i^{\ast}}{r_{\min}}\right)
    \end{align*}
\end{cor}

\begin{proof}
    Following the proof of Lemma~\ref{lem:boundary parametrization}, we have $\p_{r,\{x_i\}} = L(r,x_i)$ for all $r < r_i^{\ast}$. Therefore, the boundary $\gamma_r$ is given by
    \begin{align*}
        \gamma_r(t) = x_i + r \begin{pmatrix}
                                  \cos(t) \\ \sin(t)
                              \end{pmatrix} \quad \text{ for } t \in [0,2\pi).
    \end{align*}
    Hence, for $r_{\min} \leq r < r_{i}^{\ast}$ we get
    \begin{align*}
          & \frac{1}{\abs{B}}\int_{r_{\min}}^{r_{i}^{\ast}} f_X(r) \frac{1}{2} \int_0^{2\pi} r\cdot \skp{\gamma_{r}(t), \vec{n}(t)}\mathbf{d}t \mathbf{d}r                                              \\
        = & \frac{1}{\abs{B}}\int_{r_{\min}}^{r_i^{\ast}} \frac{r^{-3}}{r_{\min}^{-2} - r_{\max}^{-2}} \int_0^{2\pi} r^2 + r\cdot\skp{x_i,\vec{n}(t)}\mathbf{d}t \mathbf{d}r                            \\
        = & \frac{1}{\abs{B}}\int_{r_{\min}}^{r_i^{\ast}} \frac{r^{-3}}{r_{\min}^{-2} - r_{\max}^{-2}} \left[ r^2t + r\cdot\skp{x_i,\vec{n}\left(t-\frac{\pi}{2}\right)} \right]_{0}^{2\pi} \mathbf{d}r \\
        = & \frac{1}{\abs{B}}\int_{r_{\min}}^{r_i^{\ast}} \frac{r^{-3}}{r_{\min}^{-2} - r_{\max}^{-2}} \cdot 2\pi r^2 \mathbf{d}r                                                                       \\
        = & \frac{2\pi}{\abs{B}(r_{\min}^{-2} - r_{\max}^{-2})} \log\left(\frac{r_i^{\ast}}{r_{\min}}\right)
    \end{align*}
\end{proof}

With this we can now set up an algorithm to compute the prior probability of a dead leaves partition.

\subsubsection{Implementation}
The computation is done based on the theoretical ideas presented before by iterating through the layers of the dead leaves partition.

Since each leaf probability $\P(L_a=\v)$ is normalized with $\P(L_a\neq \emptyset)$, and both probabilities contain the constant scaling factors $\frac{1}{r_{\min}^{-2}-r_{\max}^{-2}}$ and $\frac{1}{\abs{B}}$, these cancel each other out, and we can remove them from the computation.

The angle-dependent antiderivatives $b(r_l;x_{ij\pm})$ are only exact up to an ambiguity of multiples of $2\pi$.
While this ambiguity cancels out in the theoretical case, for the actual computation we need a way to handle it.
For the area of possible positions, the area can maximally subtend the full circle, and we can get rid of the ambiguity by reducing the result to the remainder when removing full circles, i.\,e., computing the result modulo $2\pi\log\left(\frac{r_{l+1}}{r_l}\right)$.
Since this only works for the area of possible positions but not for the area of impossible positions, we fall back to computing the probability $\P(L_a\neq\emptyset)$ via
\begin{align*}
    \P(L_a\neq\emptyset) = \sum_{\v\in\mathcal{P}(a)\setminus \{\emptyset\}} \P(L_a=\v).
\end{align*}
Since we are interested in the ideal observer, we already need to compute the prior for every partition, so this workaround does not require additional computation.

\SetKw{Init}{Initialization:}
\SetKwFunction{Prior}{Prior}
\SetKwProg{Fn}{Function}{:}{}

\begin{algorithm}\caption{Prior}\label{alg:prior}
    \DontPrintSemicolon
    \SetNoFillComment
    \KwIn{$\m_a, (r_{\min},r_{\max})$}
    \Fn{\Prior{$\m_a$}}{
    \If{$\m_a$ is empty}{
        \KwRet{$\P(\m_a)=1$}
    }
    \Init{$\P(L_a=\v) \gets 0$ for $\v \in \m_a$, $\P(L_a\neq \emptyset)\gets 0$\;}
    \tcc{Area before first critical radius}
    \For{$x_i\in a$}{
    \If{$\{x_i\} \in \m_a$}{
    $r_i^{\ast} \gets \min_{x_j\in a\setminus \{x_i\}} r^{\ast}_{ij}$\tcp*{Lem.~\ref{lem:constant number of subcurves}}
    $\P(L_a=\{x_i\}) \mathrel{+}= 2\pi\log\left(\frac{r_i^{\ast}}{r_{\min}}\right)$\tcp*{Cor.~\ref{cor:area before first cr}}
    }
    }
    \tcc{Critical radii}
    $r^{\ast} \gets$ list($r_{\min},r_{\max}$)\;
    \For{$x_i\neq x_j \in a$}{
        Append $r^{\ast}_{ij}$ to list $r^\ast$\tcp*{Lem.~\ref{lem:constant number of subcurves}}
    }
    \For{$x_i\neq x_j \neq x_k \in a$}{
        Append $r^{\ast}_{ijk}$ to list $r^\ast$\tcp*{Lem.~\ref{lem:constant number of subcurves}}
    }
    $r^{\ast} \gets$ sort($r^{\ast}$)\;
    \tcc{Area for each radius interval}
    \For{
    $l = 1,\dots,$ length($r^{\ast}$)$-1$
    }{
    $\Delta\P \gets 0$\;
    $I \gets$ list of intersection points\;
    \For{$x_{ij\pm}\in I$}{
    \For{$\v\in \m_a$}{
    $\delta_{\v} \gets$ AND($\norm{x_{ij\pm}-x}\leq r^{\ast}_l~\forall x\in\v$, $\norm{x_{ij\pm}-x}\geq r^{\ast}_l~\forall x\in a\setminus\v$)\tcp*{Lem.~\ref{lem:relevant ip}}
    $c_{\v} \gets \pm (-1)$ ** XOR($x_i\in\v$,$x_j\in\v$)\tcp*{Prop.~\ref{prop:constructive computation}}
    $\Delta\P \mathrel{+}= \delta_{\v}\cdot c_{\v} \cdot (b(r^{\ast}_{l+1};x_{ijs})-b(r^{\ast}_l;x_{ijs}))$\tcp*{Prop.~\ref{prop:constructive computation}}
    }
    }
    $\P(L_a = \v) \mathrel{+}= \Delta\P$ mod $2\pi\log\left(\frac{r_{l+1}^{\ast}}{r_{l}^{\ast}}\right)$\;
    }
    \For{$\v\in \m_a$}{
        $\P(\m_{a\setminus \v}) \gets$ \Prior{$\m_{a}\setminus \v$}\;
    }
    $\P(L_a\neq\emptyset) \gets \sum_{\v\in \mathcal{P}(a)\setminus \{\emptyset\}} \P(L_a = \v)$\;
    $\P(\m_a) \gets \sum_{\v\in\m_a}\frac{\P(L_a=\v)}{\P(L_a\neq\emptyset)}\P(\m_{a\setminus \v})$\tcp*{Thm.~\ref{thm:prior decomposition}}
    \KwRet{$\P(\m_a)$}}
    \KwOut{$\P(\m_a)$}
\end{algorithm}

Even though this algorithm scales badly with the size of $a$, it can be modulated to obtain more efficient versions depending on the use case.
For the computation of all prior probabilities, we can avoid computing the same integral multiple times by iterating through the intersection points and adding the integral to the corresponding leaf probabilities instead of first iterating through the pixel sets and recomputing each integral.
For grids of pixels, we can additionally use symmetries to avoid multiple computations of the same results.

\begin{exone} \label{ex1: prior subset}
    Let us consider our previous example. Let $a=\{0,1,2\}^2$ and $\m_a = \{\v_1,\v_2,\v_3\}$ where
    \begin{align*}
        \v_1 & = \{(1,2),(2,1),(1,1),(2,2)\}, \\
        \v_2 & = \{(0,0),(0,1),(1,0),(0,2)\}, \\
        \v_3 & = \{(2,0)\},
    \end{align*}
    Following Theorem~\ref{thm:prior decomposition}, the prior probability is given by
    \begin{align*}
        \P(\m_a) = \prod_{i=1}^3 \frac{\P(L_{a_{\setminus i-1}, i} = \v_{a,i})}{\P(L_{a_{\setminus i-1}, i} \neq \emptyset)}
    \end{align*}
    with $a_{\setminus 0} = a$, $a_{\setminus 1} = \{(0,0),(0,1),(1,0),(0,2),(2,0)\}$ and $a_{\setminus 2} = \{(2,0)\}$.

    We will start with the first layer.
    At this stage we have multiple critical radii where the number of subcurves in the boundary can change.
    Due to the grid arrangement of the points, eight cases cover all critical radii for pairs and triplets of the given nine points (Figure~\ref{fig:critical radii}).
    \begin{figure*}[ht]\centering
        \begin{subfigure}[t]{0.24\textwidth}\centering
            \caption{}
            \vspace{3.4mm}
            \includegraphics[width=0.63\linewidth]{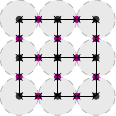}
        \end{subfigure}
        \begin{subfigure}[t]{0.24\textwidth}\centering
            \caption{}
            \vspace{2.8mm}
            \includegraphics[width=0.69\linewidth]{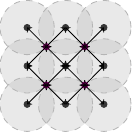}
        \end{subfigure}
        \begin{subfigure}[t]{0.24\textwidth}\centering
            \caption{}
            \vspace{0.3mm}
            \includegraphics[width=0.86\linewidth]{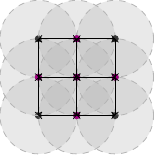}
        \end{subfigure}
        \begin{subfigure}[t]{0.24\textwidth}\centering
            \caption{}
            \includegraphics[width=0.87\linewidth]{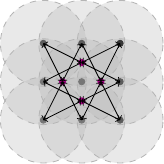}
        \end{subfigure}

        \begin{subfigure}[t]{0.24\textwidth}\centering
            \caption{}
            \vspace{2mm}
            \includegraphics[width=0.88\linewidth]{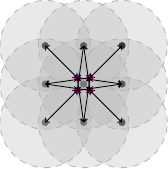}
        \end{subfigure}
        \begin{subfigure}[t]{0.24\textwidth}\centering
            \caption{}
            \vspace{1.8mm}
            \includegraphics[width=0.89\linewidth]{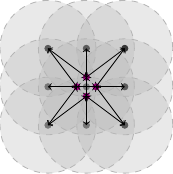}
        \end{subfigure}
        \begin{subfigure}[t]{0.24\textwidth}\centering
            \caption{}
            \vspace{1.6mm}
            \includegraphics[width=0.9\linewidth]{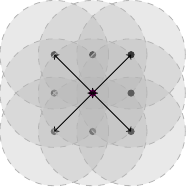}
        \end{subfigure}
        \begin{subfigure}[t]{0.24\textwidth}\centering
            \caption{}
            \includegraphics[width=\linewidth]{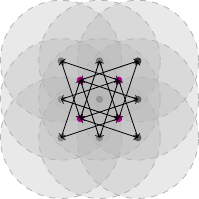}
        \end{subfigure}
        \vspace*{4pt}

        \footnotesize
        \begin{itemize*}[itemjoin=\quad]
            \item[\CIRCLE] Center point
            \item[$\leftrightarrow$] Critical radius to intersection point
            \item[\textcolor{tud-10b}{\CIRCLE}] Intersection point for critical radius
        \end{itemize*}
        \caption{
            Critical radii for a $3\times3$ grid of pixels.
            (a) Any points that are next to each other on one of the axes.
            (b) Points which are next to each other on a diagonal, which is also the critical radius for the triplet/quadruplet intersection of a point with its neighbors along the axes.
            (c) Point pairs which are separated by one point but still on the same axes, and triplet/quadruplet with the point the between the previous points in an adjacent row or column.
            (d) Points which are in adjacent columns or row in one dimension and are separated by one point along the other axis.
            (e) Triplet of two center edge points and the opposite corner point.
            (f) Triplet of two corner points on the same axis and the center point of the opposite edge.
            (g) Pair of opposite corners.
            (h) Opposite corners and edge center points.
        }\label{fig:critical radii}
    \end{figure*}
    With radius limits $r_{\min} = 1$ and $r_{\max}=2$, we have the following eight radii to consider:
    \begin{align*}
        r_{\min}=1, \frac{\sqrt{5}}{2}, \frac{5\sqrt{2}}{6}, \frac{5}{4}, \sqrt{2}, \frac{\sqrt{10}}{2}, 2=r_{\max}.
    \end{align*}

    For the minimal radius $r_{\min}=1$, the set $\p_{r,\v_{a,1}}$ is already non-empty (Figure~\ref{fig:ex1 prior first layer r3}).
    The shape of the area of possible positions changes twice at $r=\frac{\sqrt{5}}{2}$ (Figure~\ref{fig:ex1 prior first layer r5}) and $r=\frac{5\sqrt{2}}{6}$ (Figure~\ref{fig:ex1 prior first layer r6}).
    The shape then stays the same with three relevant intersection points for all larger radii (Figure~\ref{fig:ex1 prior first layer r9}).
    Using the above algorithm, we get unscaled leaf probability $\P(L_a=\v_{a,1})\approx0.418$.
    Similarly, the probability for the leaf being non-empty can be computed with only two radius intervals: for the minimal radius until the first critical radius and from the first critical radius to the maximal radius.
    The algorithm then yields the unscaled value $\P(L_a\neq\emptyset)\approx15.156$.
    \begin{figure*}[ht]\centering
        \begin{subfigure}[t]{0.24\textwidth}\centering
            \caption{}\label{fig:ex1 prior first layer r3}
            \vspace{2.5mm}
            \includegraphics[width=0.78\linewidth]{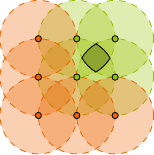}
        \end{subfigure}
        \hfill
        \begin{subfigure}[t]{0.24\textwidth}\centering
            \caption{}\label{fig:ex1 prior first layer r5}
            \vspace{2mm}
            \includegraphics[width=0.83\linewidth]{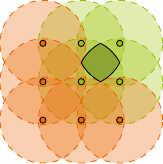}
        \end{subfigure}
        \hfill
        \begin{subfigure}[t]{0.24\textwidth}\centering
            \caption{}\label{fig:ex1 prior first layer r6}
            \vspace{1mm}
            \includegraphics[width=0.9\linewidth]{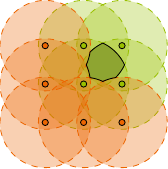}
        \end{subfigure}
        \hfill
        \begin{subfigure}[t]{0.24\textwidth}\centering
            \caption{}\label{fig:ex1 prior first layer r9}
            \includegraphics[width=\linewidth]{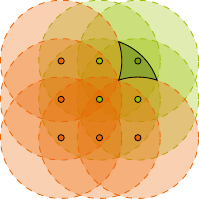}
        \end{subfigure}
        \vspace*{4pt}

        \footnotesize
        \begin{itemize*}[itemjoin=\quad]
            \item[\pp] Possible positions
            \item[\ip] Included positions
            \item[\ep] Excluded positions
        \end{itemize*}
        \caption{
            Area of possible positions for a leaf to match $\v_{a,1}$ for different radii.
            (\subref{fig:ex1 prior first layer r3}) Minimal radius $r_{\min} = 1$.
            (\subref{fig:ex1 prior first layer r5}-\subref{fig:ex1 prior first layer r6}) Radii where the number of singularities in the boundary changes $r=\frac{\sqrt{5}}{2}$ and $r=\frac{5\sqrt{2}}{6}$.
            (\subref{fig:ex1 prior first layer r9}) Radius greater than the largest critical radius.
        }
        \label{fig:ex1 prior first layer}
    \end{figure*}

    For the second layer, the pixels in $\v_{a,1}$ are occluded by the first leaf, and we only consider the points in $a\setminus \v_{a,1}$ (Figure~\ref{fig:ex1 prior second layer}).

    The area of possible positions is empty until $r=\frac{\sqrt{5}}{2}$ (Figure~\ref{fig:ex1 prior second layer r1}) and then changes shape only once (Figure~\ref{fig:ex1 prior second layer r4}).
    The resulting leaf probability is $\P(L_{a_{\setminus1}}=\v_{a,2})\approx0.321$.
    Since we reduced the set of pixels, we also need to compute the corresponding probability for a non-empty leaf, which is $\P(L_{a_{\setminus 1}}\neq\emptyset)\approx11.918$.

    \begin{figure*}[ht]\centering
        \begin{subfigure}[t]{0.3\textwidth}\centering
            \caption{}\label{fig:ex1 prior second layer r1}
            \vspace{5.5mm}
            \includegraphics[width=0.65\linewidth]{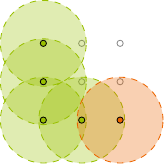}
        \end{subfigure}
        \hfill
        \begin{subfigure}[t]{0.3\textwidth}\centering
            \caption{}\label{fig:ex1 prior second layer r3}
            \vspace{4.5mm}
            \includegraphics[width=0.75\linewidth]{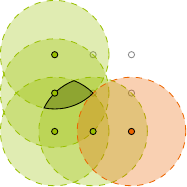}
        \end{subfigure}
        \hfill
        \begin{subfigure}[t]{0.3\textwidth}\centering
            \caption{}\label{fig:ex1 prior second layer r4}
            \includegraphics[width=\linewidth]{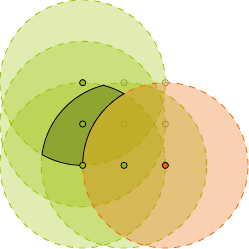}
        \end{subfigure}
        \vspace*{4pt}

        \footnotesize
        \begin{itemize*}[itemjoin=\quad]
            \item[\pp] Possible positions
            \item[\ip] Included positions
            \item[\ep] Excluded positions
        \end{itemize*}
        \caption{
            Area of possible positions for a leaf to match $\v_{a,2}$ on $a\setminus\v_{a,1}$ for different radii.
            (\subref{fig:ex1 prior second layer r1}) Smallest possible radius.
            (\subref{fig:ex1 prior second layer r3}) Radius where the number of singularities in the boundary changes.
            (\subref{fig:ex1 prior second layer r4}) Radius greater than the largest critical radius.
        }\label{fig:ex1 prior second layer}
    \end{figure*}

    For the last layer, the set of pixels we are looking at is reduced to $a_{\setminus 2} = \v_3 = \{(2,0)\}$.
    Since this set of pixels only contains a single point, the area of possible positions is simply the circle around this point for all possible radii.
    Hence, the unscaled leaf probability and the probability for a non-empty leaf are the same at $\P(L_{a_{\setminus 2}}=\v_{a,3})=\P(L_{a_{\setminus 2}}\neq\emptyset)\approx4.355$.

    By scaling every leaf probability by the probability of a non-empty leaf and multiplying the results, we get the prior probability for this ordered partition
    \begin{align*}
        \P(\m_a)
        = \prod_{i=1}^3 \frac{\P(L_{a_{\setminus i-1}, i} = \v_{a,i})}{\P(L_{a_{\setminus i-1}, i} \neq \emptyset)}
        \approx\frac{0.418}{15.156} \cdot \frac{0.321}{11.918} \cdot \frac{4.355}{4.355}
        \approx 7.438\cdot10^{-4}.
    \end{align*}
\end{exone}

\subsection{Likelihood} \label{sec: likelihood}
The likelihood describes the probability of an image $\mathbf{s}$ on the set $a$ given some partition $\m_a=\{\v_{a,i}\}_{i=1,\dots,n}$.
Therefore, it depends on the memberships defined by the model as well as the distributions of color and texture.

\begin{prop}\label{prop:likelihood}
    Let $\mathbf{s}$ be a dead leaves image on $a$.
    Then the likelihood for this image given some dead leaves partition $\m_a=\{\v_{a,i}\}_{i\in[n]}$ is given by
    \begin{align*}
        \P(\mathbf{s}\mid \m_a) = \prod_{i=1}^n \left(\sum_{c}\P_c(c)\left(\prod_{x_j\in \v_i}\P_t(\mathbf{s}_{x_j}-c)\right)\right),
    \end{align*}
    where $\mathbf{s}_{x_j}$ is the value of the image at pixel $x_j$.
\end{prop}

This proposition employs discrete color and texture distributions.
For continuous distributions, the sum becomes an integral, and the probabilities become probability density functions.
While the likelihood is then no longer a probability, we can still use it to compute the maximum posterior.

\begin{proof}
    The probability of a given image $\mathbf{s}$ is the joint probability for the value of each single pixel.
    \begin{align*}
        \P(\mathbf{s}\mid \m_a) = \P(\mathbf{s}_{x_1},\dots,\mathbf{s}_{x_n}\mid \m_a)
    \end{align*}

    Since all leaves are independent in terms of size, position, color, and texture, the pixel values are independent for pixels in different leaves \citep[Compare to][eq.~1]{Pitkow2010}.
    Therefore,
    \begin{align*}
        \P(\mathbf{s}_{x_1},\dots,\mathbf{s}_{x_n}\mid \m_a) = \prod_{i=1}^n \P(\mathbf{s}_{x_j} \mid x_j\in \v_i).
    \end{align*}
    The law of total probability yields
    \begin{align*}
        \P(\mathbf{s}_{x_j} \mid x_j\in \v_i) = \sum_{c}\P_c(c)\P(\mathbf{s}_{x_j} \mid x_j\in \v_i, c).
    \end{align*}
    Since the texture was independently added, we have
    \begin{align*}
        \P(\mathbf{s}_{x_j} \mid x_j\in \v_i, c) = \prod_{x_j\in \v_i}\P_t(\mathbf{s}_{x_j}-c).
    \end{align*}
\end{proof}

Depending on the specific color and texture distributions, the likelihood can become very simple, and might not contain relevant information to solve the segmentation task.
In the simplest case with discrete uniform color and texture the likelihood even reduces to a constant value:
\begin{align*}
    \P(\mathbf{s}\mid \m_a)
     & = \prod_{i=1}^n\left(\sum_c\P_c(c)\left(\prod_{x_j\in \v_i}\P_t(\mathbf{s}_{x_j}-c)\right)\right)
    = \prod_{i=1}^n\left(\sum_c p_c\left(\prod_{x_j\in\v_i}p_t\right)\right)                             \\
     & = \prod_{i=1}^n\left(\sum_c p_c \cdot p_t^{\abs{\v_i}}\right)
    = \prod_{i=1}^n  p_t^{\abs{\v_i}}
    = p_t^{\sum_{i=1}^n \abs{\v_i}}
    = p_t^{\abs{a}}.
\end{align*}

We have a more informative likelihood when using independent Gaussian distributions for color and texture \cite[see also][]{Pitkow2010}.
With zero-mean texture, we then have three parameters: color mean $\mu_c$, and standard deviations $\sigma_c$ and $\sigma_t$.
The dead leaves partition then influences the covariance of pixel values.
Since the texture is added pixel-wise to the leaf color, the joint probability for pixels of one leaf is distributed according to
\begin{align*}
    \P(\mathbf{s}_{x_j}\mid x_j\in \v_i) \sim \mathcal{N}(\mu_c,\Sigma_i), \quad\text{with}\quad
    \Sigma_{i;x,y} = \begin{cases}
                         \sigma_c^2 + \sigma_t^2, & \text{if } x=y,  \\
                         \sigma_c^2,              & \text{otherwise}
                     \end{cases} \in \R^{\abs{\v_i} \times \abs{\v_i}}.
\end{align*}
The likelihood is then given by
\begin{align*}
    f(\mathbf{s}\mid \m_a) = \prod_{i=1}^n f_i(\mathbf{s}_{x_j\in \v_i})
\end{align*}
where $f_i$ is the Gaussian probability density function for leaf $\v_i$ with mean $\mu_c$ and covariance $\Sigma_i$.
For colored images with independent color channels, the likelihood would additionally be the product of the likelihood for each single channel.

For better computational performance, we might want to use the $\log$-likelihood such that
\begin{align*}
    \log\left(f(\mathbf{s}\mid \m_a)\right) = \sum_{i=1}^n \log(f_i(\mathbf{s}_{x_j\in \v_i})).
\end{align*}

We here already make use of the information provided by the partition by computing the likelihood as product or sum of the leaf-wise likelihood.
Alternatively, one could consider the full likelihood with $\mu_c\in \R^{\abs{a}}$ and $\Sigma\in\R^{\abs{a}\times\abs{a}}$.
In this instance the covariance matrix mirrors the structure of the partition:
\begin{align*}
    \Sigma = \begin{pmatrix}
                 \Sigma_1 & 0        & \cdots & 0        \\
                 0        & \Sigma_2 & \ddots & \vdots   \\
                 \vdots   & \ddots   & \ddots & 0        \\
                 0        & \cdots   & 0      & \Sigma_n \\
             \end{pmatrix},
\end{align*}
where $\Sigma_i$ are the leaf-wise covariances as described above.
For computational reasons, we here refrain from using such a full likelihood description and employ the partition structure by computing the leaf-wise likelihood.

\begin{exone} \label{ex1: likelihood full}
    We recall the partition and image from before:
    \begin{center}
        \begin{tikzpicture}[scale = 0.5]
            \foreach \x in {0,1,2}{
                    \draw[tud-0a] (-1,{\x}) node {$\x$};
                    \draw[tud-0a] ({\x},-1) node {$\x$};
                }

            \draw (0,0) node {$2$};
            \draw (0,1) node {$2$};
            \draw (1,0) node {$2$};
            \draw (1,1) node {$1$};
            \draw (2,0) node {$3$};
            \draw (2,1) node {$1$};
            \draw (2,2) node {$1$};
            \draw (0,2) node {$2$};
            \draw (1,2) node {$1$};
        \end{tikzpicture}
        \quad
        \definecolor{t1}{rgb}{0.2577, 0.3319, 0.5822}
        \definecolor{t2}{rgb}{0.2571, 0.3444, 0.5866}
        \definecolor{t3}{rgb}{0.2714, 0.7642, 0.7442}
        \definecolor{t4}{rgb}{0.2688, 0.3927, 0.5983}
        \definecolor{t5}{rgb}{0.5611, 0.2157, 0.6222}
        \definecolor{t6}{rgb}{0.5631, 0.2172, 0.5969}
        \definecolor{t7}{rgb}{0.2536, 0.3439, 0.5845}
        \definecolor{t8}{rgb}{0.5559, 0.2208, 0.6352}
        \definecolor{t9}{rgb}{0.5406, 0.2098, 0.6082}
        \begin{tikzpicture}[scale = 0.5]
            \fill[t1] (0,0) rectangle +(1,1);
            \fill[t2] (1,0) rectangle +(1,1);
            \fill[t3] (2,0) rectangle +(1,1);
            \fill[t4] (0,1) rectangle +(1,1);
            \fill[t5] (1,1) rectangle +(1,1);
            \fill[t6] (2,1) rectangle +(1,1);
            \fill[t7] (0,2) rectangle +(1,1);
            \fill[t8] (1,2) rectangle +(1,1);
            \fill[t9] (2,2) rectangle +(1,1);
            \foreach \x in {0,1,2}{
                    \draw[tud-0a] (-0.5,{\x+0.5}) node {$\x$};
                    \draw[tud-0a] ({\x+0.5},-0.5) node {$\x$};
                }
        \end{tikzpicture}
    \end{center}

    With a discrete uniform color distribution on 8-bit RGB, the probability for each color would be $\frac{1}{256^3}$, and analogously, the texture probability on a range from $-10$ to $10$ for each color channel would be $p_t=\frac{1}{21^3}$.
    Since we have $9$ pixels, the likelihood is then $\P(\mathbf{s}\mid \m_a) = \left(\frac{1}{21^3}\right)^9 \approx 1.996\cdot 10^{-36}$.

    With Gaussian color and texture, let us assume HSV with values between $0$ and $1$ for each channel.
    Assume we have color parameters $\mu_c = 0.6$, $\sigma_c=0.1$, and texture standard deviation $\sigma_t=0.01$.
    The leaf-wise covariance matrices for each color are then of the form
    \begin{align*}
        \begin{pNiceArray}{cccc}[first-row]
            \multicolumn{2}{c}{\Sigma_3} & \Sigma_1 & \Sigma_2 \\
            \sigma_{c+t}^2 & \sigma_c^2 & \sigma_c^2 & \sigma_c^2 \\
            \sigma_c^2 & \sigma_{c+t}^2 & \sigma_c^2 & \sigma_c^2 \\
            \sigma_c^2 & \sigma_c^2 & \sigma_{c+t}^2 & \sigma_c^2 \\
            \sigma_c^2 & \sigma_c^2 & \sigma_c^2 & \sigma_{c+t}^2 \\
            \CodeAfter
            \SubMatrix{.}{1-1}{2-2}{.}[code = \tikz \draw[dotted] (1-|3) -- (3-|3) -- (3-|1);]
            \SubMatrix{.}{1-1}{3-3}{.}[code = \tikz \draw[dotted] (1-|4) -- (4-|4) -- (4-|1);]
        \end{pNiceArray}
        =
        \begin{pNiceArray}{cccc}[first-row]
            \multicolumn{2}{c}{\Sigma_3} & \Sigma_1 & \Sigma_2 \\
            .0101 & .01 & .01 & .01 \\
            .01 & .0101 & .01 & .01 \\
            .01 & .01 & .0101 & .01 \\
            .01 & .01 & .01 & .0101 \\
            \CodeAfter
            \SubMatrix{.}{1-1}{2-2}{.}[code = \tikz \draw[dotted] (1-|3) -- (3-|3) -- (3-|1);]
            \SubMatrix{.}{1-1}{3-3}{.}[code = \tikz \draw[dotted] (1-|4) -- (4-|4) -- (4-|1);]
        \end{pNiceArray}
    \end{align*}
    Since we are working with three independent color channels, the likelihood of each leaf is given by
    \begin{align*}
        f_i(\mathbf{s}_{x_j\in \v_i})=\prod_{c=1}^{3} f_{i,c}(\mathbf{s}_{x_j\in \v_i})
    \end{align*}
    where $f_{i,c}$ is the probability density function for each individual color channel.

    With these parameters we get the probability densities
    \begin{align*}
        \log(f_1(\mathbf{s}_{x_j\in \v_1})) & \approx 26.814, &
        \log(f_2(\mathbf{s}_{x_j\in \v_2})) & \approx 32.048, &
        \log(f_3(\mathbf{s}_{x_j\in \v_3})) & \approx 2.137,
    \end{align*}

    which results in the total image $\log$-likelihood of $\log(f(\mathbf{s}\mid \m_a))\approx 60.999$.
    Since this is the value of a probability density function, it is hard to interpret without context.
    For comparison, the partition
    \begin{center}
        \begin{tikzpicture}[scale = 0.5]
            \foreach \x in {0,1,2}{
                    \draw[tud-0a] (-1,{\x}) node {$\x$};
                    \draw[tud-0a] ({\x},-1) node {$\x$};
                }

            \draw (0,0) node {$2$};
            \draw (0,1) node {$2$};
            \draw (1,0) node {$3$};
            \draw (1,1) node {$1$};
            \draw (2,0) node {$3$};
            \draw (2,1) node {$1$};
            \draw (2,2) node {$1$};
            \draw (0,2) node {$2$};
            \draw (1,2) node {$1$};
        \end{tikzpicture}
    \end{center}
    has a $\log$-likelihood of approximately $-75.793$; and the original partition is the partition with the maximal likelihood.
\end{exone}

As the last step, we can combine the prior probability and the likelihood to obtain the posterior probability of a partition given an image.

\subsection{Posterior} \label{sec: posterior}

Taking everything together, we can derive the posterior probability of a dead leaves partition through
\begin{align*}
    \P(\m_a\mid \mathbf{s}) \propto \P(\mathbf{s}\mid \m_a)\P(\m_a),
\end{align*}
where the prior (Theorem~\ref{thm:prior decomposition}) is given by
\begin{align*}
    \P(\m_a) = \sum_{i=1}^n \frac{\P(L_{a,i}=\v_{a,i})}{\P(L_{a,i}\neq \emptyset)}\P(\m_{a\setminus\v_{a,i}})
\end{align*}
with leaf probability:
\begin{align*}
    \P(L_a=\v) & = \int_{-\infty}^\infty f_X(r)\int_{\R^2}f_{P}(p)\P(L_a(r,p)=\v\mid r,p)\mathbf{d}p\mathbf{d}r.
\end{align*}
This leaf probability can either be approximated, for example, through a grid approximation, or derived analytically (Theorem~\ref{thm:leaf prob}, Proposition~\ref{prop:constructive computation}), yielding
\begin{align*}
      & \P(L_a=\v)                                                                                                                                                                                              \\
    = & \sum_{l=1}^L \int_{r_l}^{r_{l+1}}f_X(r)\sum_{k=1}^{K_l}\frac{1}{2}\int_{t_k}^{t_k'}\skp{\gamma_{r,\v}^k(t),\vec{n}(\gamma_{r,\v}^k(t))}\norm{\dot{\gamma}_{r,\v}^k(t)}\mathbf{d}t\mathbf{d}r            \\
    = & \frac{1}{\abs{B}(r_{\min}^{-2}-r_{\max}^{-2})}\sum_{l=1}^L\left(\sum_{x_{ij\pm\in I_r}}\delta_{r_l,\v}(x_{ij\pm})\cdot c_{\v}(x_{ij\pm})\cdot\left(b(r_{l+1};x_{ij\pm})-b(r_l;x_{ij\pm})\right)\right).
\end{align*}
The likelihood (Proposition~\ref{prop:likelihood}) is given by:
\begin{align*}
    \P(\mathbf{s}\mid \m_a) = \prod_{i=1}^n\left(\sum_c\P_c(c)\left(\prod_{x_j\in\v_i}\P_t(\mathbf{s}_{x_j}-c)\right)\right).
\end{align*}
Hence, the posterior probability is
\begin{align*}
    \P(\m_a\mid \mathbf{s}) \propto \left(\prod_{i=1}^n\left(\sum_c\P_c(c)\left(\prod_{x_j\in\v_i}\P_t(\mathbf{s}_{x_j}-c)\right)\right)\right)\cdot \sum_{i=1}^n \frac{\P(L_{a,i}=\v_{a,i})}{\P(L_{a,i}\neq \emptyset)}\P(\m_{a\setminus\v_{a,i}}).
\end{align*}
For an ordered partition, we get
\begin{align*}
    \P(\m_a\mid \mathbf{s})
     & \propto \left(\prod_{i=1}^n\left(\sum_c\P_c(c)\left(\prod_{x_j\in\v_i}\P_t(\mathbf{s}_{x_j}-c)\right)\right)\right)\cdot \prod_{i=1}^n \frac{\P(L_{a_{\setminus i-1},i}=\v_{a,i})}{\P(L_{a_{\setminus i},i}\neq \emptyset)} \\
     & = \prod_{i=1}^n\left(\left(\sum_c\P_c(c)\left(\prod_{x_j\in\v_i}\P_t(\mathbf{s}_{x_j}-c)\right)\right) \cdot \frac{\P(L_{a_{\setminus i-1},i}=\v_{a,i})}{\P(L_{a_{\setminus i},i}\neq \emptyset)}\right).                   \\
\end{align*}

For the partition with maximal posterior probability, we then need to compute the posterior probability of any possible partition.
In this instance we can arrive at the actual posterior probability by computing the unscaled posterior for all partitions and then scaling them by the sum of the results.

\subsection{Implementation} \label{sec: ideal observer implementation}

To implement the ideal observer we now combine the previous sections.
Since the probabilities become exponentially smaller with number of pixels we will perform the computations in $\log$-space to avoid underflow.
The resulting algorithm~\ref{alg:io} including the prior probability computation (algorithm~\ref{alg:prior}) is available as Python module\footnote{\url{https://github.com/ag-perception-wallis-lab/ideal_observer_dead_leaves_segmentation}}.

\begin{algorithm}\caption{Ideal observer}\label{alg:io}
    \DontPrintSemicolon
    \KwIn{$S$, $\P_c$, $\P_t$}
    \Init{$M^{\ast}\gets$ None, MAP $\gets$ 0\;}
    \For{$\m_a\in \M_a$}{
        prior $\gets \log($\Prior($M_a$)$)$\tcp*{Alg.~\ref{alg:prior}}
        likelihood $\gets \sum_{i=1}^n \log\left(\sum_{c}\P_c(c)\left(\prod_{x_j\in\v_i}\P_t(\mathbf{s}_{x_j}-c)\right)\right)$\tcp*{Prop.~\ref{prop:likelihood}}
        posterior $\gets$ prior $+$ likelihood\;
        \If{posterior $>$ MAP}{
            $M^{\ast} \gets \m_a$\;
            MAP $\gets$ posterior\;
        }
    }
    \KwOut{$M^\ast$, MAP}
\end{algorithm}

For images with state-of-the-art resolutions and color depth the computation of this ideal observer is infeasible since the number of possible partitions of the pixel space explodes.
For example, on our current workstation, computing the maximum posterior for $9$~pixels Figure~\ref{fig:all probs for ex} took $15$~minutes, but already $10$~pixels would take roughly $20$~hours since we also have fewer symmetries we can exploit.
The next larger square with $16$~pixels then has $10^{10}$~partitions which we are not able to store simultaneously with $32$~GB of RAM.
For practical use it is therefore necessary to perform the computation on a small set of pixels within an image or on very small images.

\begin{exone}
    Using algorithm~\ref{alg:io} we can compute the posterior probabilities of all possible partitions for our example image.
    Figure~\ref{fig:all probs for ex} shows the resulting probabilities for the $15$ partitions with the highest posterior probabilities.
    For this example the ideal observer produces the original partition as most probable.

    \begin{figure}[ht]\centering
        \begin{subfigure}[t]{0.48\textwidth}
            \caption{}\label{fig:all probs for ex}
            \includegraphics[height=4.8cm]{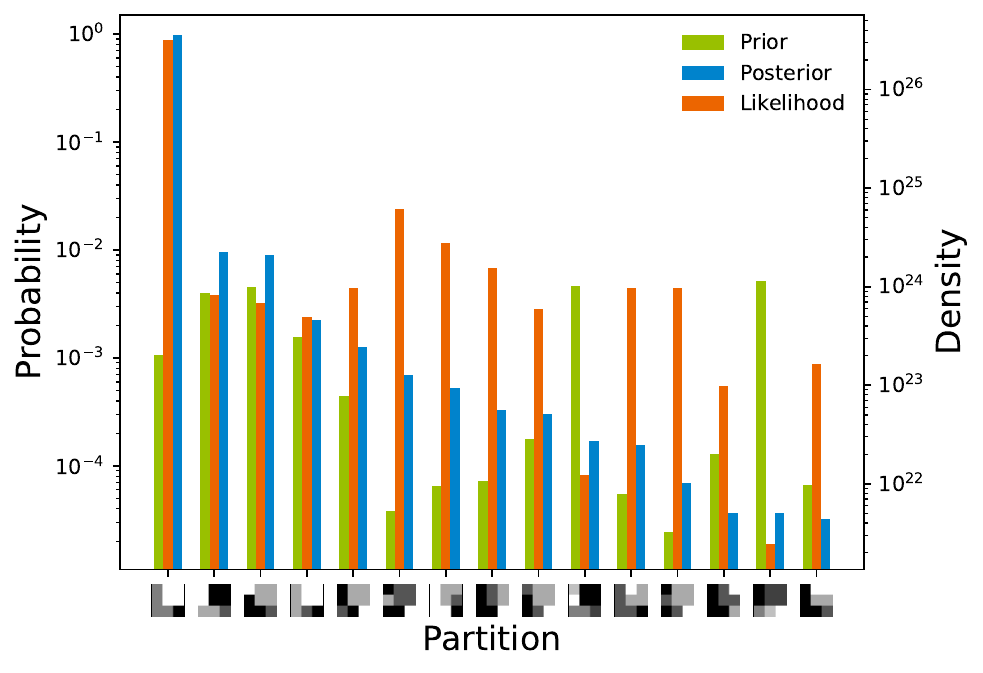}
        \end{subfigure}
        \hfill
        \begin{subfigure}[t]{0.48\textwidth}
            \caption{}\label{fig:all priors for ex}
            \includegraphics[height=4.8cm]{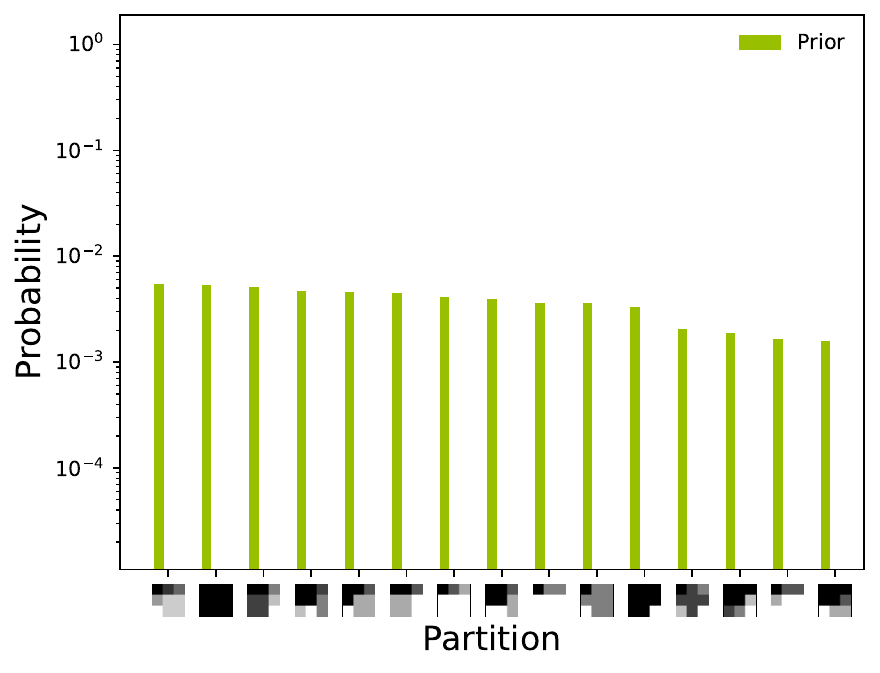}
        \end{subfigure}
        \caption{
            Resulting partition probabilities of the ideal observer.
            (\subref{fig:all probs for ex}) Prior and posterior probabilities and likelihood density function values for the fifteen partitions with the highest posterior probabilities.
            (\subref{fig:all priors for ex}) Prior probabilities for the fifteen partitions with the highest unique prior probability.
            The x-axis depict the corresponding unordered partition of the nine considered pixels.
        }
    \end{figure}

    In this example the likelihood for the original partition is very large since the leaves have a large color difference in comparison to the texture that is introduced on a pixel-wise basis.
    Changing the membership of pixels reduces the likelihood, in particular if the respective pixel has a very similar color.
    For example the partition with the second-highest likelihood (6th partition in Figure~\ref{fig:all probs for ex}) assumes that the pixel $(0,1)$ is on an individual leaf.
    In the original image this pixel is the one with the highest texture, i.\,e., the largest difference to the other pixels on the leaf so it is still relatively likely that the picture was generated from this partition with similar colors for two of the leaves.

    For the prior we can see that some partitions are less probable than others, for example it is improbable that we have a leaf with a single pixel that occludes another leaf (e.\,g., partitions 6 and 12 in Figure~\ref{fig:all probs for ex}).
    In fact there are also partitions with a prior probability of zero, for example any partition with $\{\{(0,0),(1,1)\},\{(0,1),(1,0)\}\}$ since this configuration can not be produced through occlusion of circles.
    While in general, due to the power law distributed radius, smaller leaves are more probable, here we have a minimal radius set to $1$ such that it is quite likely to sample a leaf that will include a square of pixels as can be observed from the high prior probabilities in Figure~\ref{fig:all priors for ex}.
    Additionally, the prior probability is invariant to mirroring and rotations due to the symmetry of the leaves such there are many partitions which have the same prior probability (Figure~\ref{fig:all priors for ex} only shows exemplars with a unique prior probability).
\end{exone}

\subsection{Performance}

To empirically quantify the performance of the ideal observer in this segmentation task, we created $108$ test datasets, each containing $1000$ dead leaves image patches varying in patch size, leaf size, and texture intensity.
We generated the patches by sampling $128\times128$~pixel images and sampling random patches of various sizes (see below) from these images.
The leaf colors were sampled in HSV with mean $\mu_c = 0.5$ and standard deviation $\sigma_c = 0.1$ for all channels (as in Figure~\ref{img:random dlm}).
We added global Gaussian texture (i.\,e. Gaussian pixel noise) with intensities $\sigma_t = 0.01$ (low), $\sigma_t = 0.05$ (moderate), or $\sigma_t = 0.1$ (high).
Leaf sizes were power-law distributed with one of four pixel radii ranges, resulting in four leaf size conditions: small ($r_{\min}=4,r_{\max}=8$), medium ($r_{\min}=8,r_{\max}=12$), large ($r_{\min}=12,r_{\max}=16$), and mixed ($r_{\min}=4,r_{\max}=16$).

We then computed the accuracy (the ratio of images for which the predicted segmentation matches the ground truth) for maximum posterior prediction, maximum likelihood prediction, and maximum prior prediction for all datasets.
By construction, the observer model based on maximum likelihood makes predictions based only on pixel value similarity, and the maximum prior observer makes predictions based only on the base rate of each possible partition.
In addition, we compute the performance of an observer model that chooses one of the possible segmentations at random, representing empirical chance performance.
We also computed the adjusted Rand index, to measure the similarity between the predicted segmentations and ground truth segmentation.
The adjusted Rand index quantifies the portion of pixel pairs on which two segmentations are in agreement normalized by chance level performance \citep{Rand1971,Hubert1985}.
Figure~\ref{fig:io performance} summarizes the results of this analysis.

\begin{figure}[ht]\centering
    \begin{subfigure}[t]{0.8\linewidth}
        \caption{}\label{subfig:accuracy by patch size}
        \includegraphics[width=\linewidth]{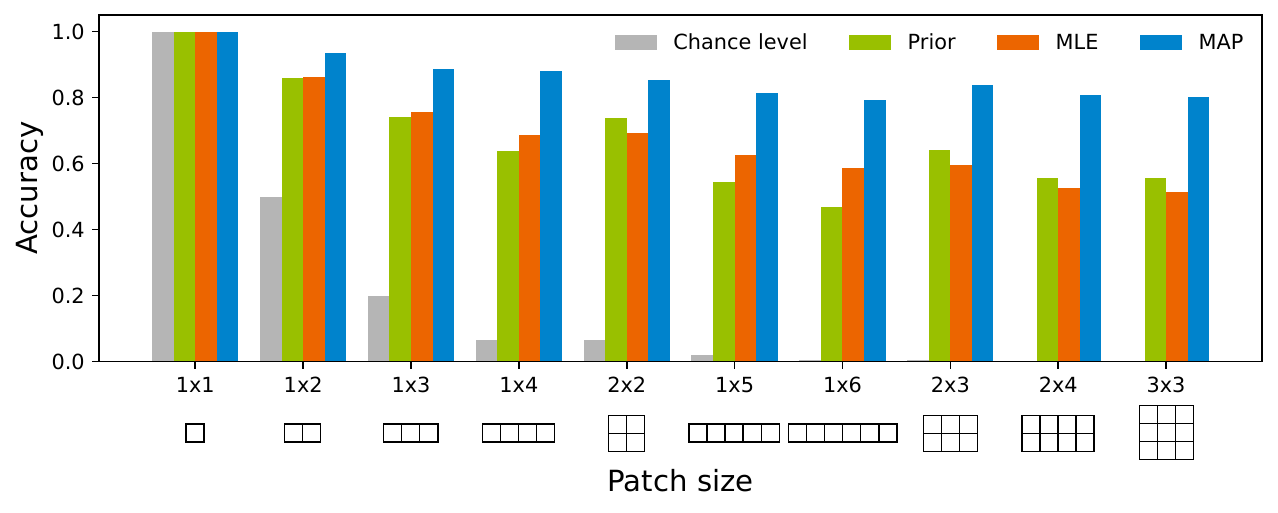}
    \end{subfigure}

    \begin{subfigure}[t]{0.8\linewidth}
        \caption{}\label{subfig:adi by parameter}
        \includegraphics[width=\linewidth]{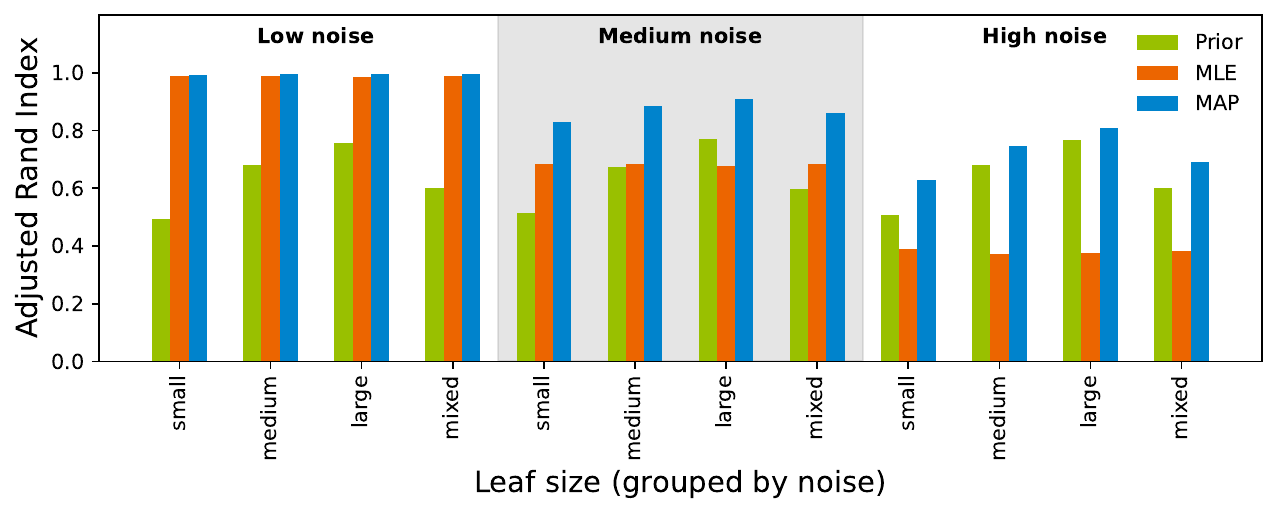}
    \end{subfigure}
    \caption{
        Segmentation task performance of the maximum posterior prediction (\textcolor{tud-2b}{$\blacksquare$} MAP), the maximum likelihood prediction (\textcolor{tud-8b}{$\blacksquare$} MLE), and the maximum prior prediction (\textcolor{tud-4b}{$\blacksquare$} Prior).
        (\subref{subfig:accuracy by patch size}) Accuracy of each observer model by patch size and layout averaged over all leaf sizes and noise intensities. As a baseline, , we also show the performance of an observer that randomly chooses a possible segmentation (\textcolor{tud-0b}{$\blacksquare$}). 
        (\subref{subfig:adi by parameter}) Mean adjusted Rand index for each observer model over all patch sizes grouped by leaf size and noise intensity. Since we use the adjusted Rand index, chance level performance here is $0$.
    }\label{fig:io performance}
\end{figure}

For all dead leaves observer models, accuracy drops more slowly as patch sizes increases than the random choice observer (Figure~\ref{subfig:accuracy by patch size}).
The performance of the random model drops quickly with the number of pixels, specifically, it scales inversely to Bell's number.
Similarly, the accuracy of maximum likelihood predictions decreases with the number pixels independently of the pixel layout, since it only uses pixel value similarity information.
For the maximum prior predictions (which depend on geometry), accuracy is higher for more square patches compared with long patches of the same or even smaller pixel count.
The maximum posterior combines prior and likelihood information, so its estimates outperform the other observer models in all conditions and are better for squarer patches, as expected.

Each observer shows systematic variations in performance depending on the dead leaves model parameters (Figure~\ref{subfig:adi by parameter}).
Since the likelihood only depends on the pixel values and uses pixel value similarity to make a prediction, the maximum likelihood prediction drops with increasing noise but is agnostic to the sizes of the leaves.
Conversely, the maximum prior prediction only varies with leaf sizes but not with the noise intensity, since the prior does not use the pixel values.
Specifically, the performance of the maximum prior prediction increases with increasing leaf size, which here is driven by an increasingly large prior probability for having only one leaf in the image patch.
Since the ideal observer combines the information of prior and likelihood, we find decreased performance with both higher noise intensity and with smaller leaf sizes.
For high-noise patches the likelihood is less informative, so the ideal observer performs closer to the maximum prior predictions, whereas for patches with very little noise, it performs similarly to maximum likelihood predictions.

\section{Discussion}
In this work, we provide a self-contained, step-by-step reformulation of the theory of \citet{Pitkow2010}.
We emphasize how this theory can be used to compute an ideal observer for a segmentation task on dead leaves images.
We provide the resulting algorithms necessary to compute the ideal observer and give examples for each step in the computation.
Finally, we empirically evaluated the performance of the ideal observer against a pixel-similarity-based observer (likelihood) and a geometry-based observer (prior), showing how performance depends on image patch size, leaf size and texture (noise) intensity.

While the presented ideal observer in theory is computable for arbitrary finite pixel counts, the practical computation of full posterior partition probabilities is limited to small image patches.
Since the number of possible segmentations of a set of pixel scales according to Bell's number (at least exponentially), the computation of the maximum posterior over all possible segmentations is infeasible for more than a few pixels.
Computing the ideal observer for $10$~pixels already takes around $20$~hours.
The high dimensionality also results in very small probabilities for any single image or partition, in particular, for images where the segmentation relies on the prior.
While the partition with the maximal prior for $2$~pixels has a probability of roughly $0.6$, for $2\times2$~pixels that number already drops to $0.16$ and to $0.005$ for $3\times3$~pixels.
Consequently, even though the ideal observer analytically produces the correct results, we might run into numerical issues that inhibit the identification of the true MAP.
When designing tasks using this setup, one should therefore consider the dimensionality of the tasks carefully.

It might be possible to derive segmentation methods for larger images, by using the smoothness of the partition probabilities for smart sampling.
The probabilities depicted in Figure~\ref{fig:all probs for ex} suggest that the probability functions are relatively ``smooth'' on the set of possible segmentations.
In the settings we have examined here, it seems that changing the segmentation only slightly  (switching the membership of one pixel for example) has a small effect on the probability.
Smart sampling algorithms could be developed to exploit this observation and find a maximal posterior segmentation for images with more pixels, for which the computation of all posterior probabilities is infeasible.
However, work in this direction would require a sound implementation for computing the leaf probability for a non-empty leaf.

Another possible route to extend the segmentation ideal observer to larger images would be to use a sliding window approach.
Computing the ideal observer for successive $3\times3$ patches of a larger image is tractable, since the prior remains the same for each patch.
One could then combine partitions with high posterior probabilities into a segmentation of the full image, which could result in an observer model feasible for larger images using the information of the ideal observer.
Large scale information could be incorporated by computing the ideal observer prediction for pixel grids with larger pixel distances.

Currently, the ideal observer model is constrained to a specific leaf shape distribution (circles with power-law distributed radii).
In theory, the ideal observer could be extended to other shapes and distributions, so long as the shape can be parameterized by a piecewise smooth boundary.
Whether there is an analytical solution for the leaf probability then depends on the specific setup.
Additionally, the theory depends only on the leaves being sampled independently --- not on the independence of the leaves' features.
Consequently, it could be extended to models with a dependency (for example, between position and color).

For segmentation tasks on a limited number of pixels, this ideal observer already provides a measure for the ideal performance given the generative dead leaves model.
While the presented method is intractable for producing a prediction of segmentation for a full size image, it can be used for computing ideal observer performance for targeted questions, for example, whether two points in a dead leaves image belong to the same object or to two different ones.
This ideal performance can be used as a benchmark for comparison with both human observers and other algorithms such as deep neural networks.
For example, combining dead leaves stimuli and the ideal observer model with a recent experimental technique to measure probabilistic perceptual segmentation maps \citep{Vacher2023} could represent a fruitful avenue.
Future work could also aim to develop and test approximations to this benchmark solution with the goal of using these approximations to extend this work to different dead leaves models and tasks of relevant dimensionality.

\section*{Acknowledgements}
We would like to thank Johann Bauer for many helpful discussions when working through this theory, as well as Xaq Pitkow for his insights on solving the dead leaves model.

This work was supported by the Deutsche Forschungsgemeinschaft (German Research Foundation, DFG) under Germany’s Excellence Strategy (EXC 3066/1 “The Adaptive Mind”, Project No. 533717223).
Additionally, this work was co-funded by the European Union (ERC, SEGMENT, 101086774).
Views and opinions expressed are however those of the author(s) only and do not necessarily reflect those of the European Union or the European Research Council.
Neither the European Union nor the granting authority can be held responsible for them.

\section*{Contributions}
TW and SM developed the central research questions.
SM together with MO derived the main propositions, lemmas, and theorems, provided the formal proofs, and validated their internal consistency.
MO further refined the mathematical arguments.
SM translated the theoretical results into algorithmic procedures, prepared the codebase, and ensured computational correctness.
LR tested the codebase.
LR and SM implemented the code to evaluate the performance of the models.
SM generated the examples, ensured their alignment with the theoretical assumptions.
SM also designed all figures illustrating algorithmic behavior and theoretical examples.
SM supported by MO wrote the main theoretical sections.
SM wrote the examples and framing text.
MO critically revised the draft for mathematical accuracy and clarity.
SM, MO, LR, and TW contributed to iterative review and editing.
SM coordinated the collaborative work, and TW provided conceptual guidance.

\section*{Code availability}
All code implementing this theory and evaluating the observer performance can be found at: \url{https://doi.org/10.5281/zenodo.22209810}.

All dead leaves images were generated using the \texttt{deadleaves} Python library \citep{Mahncke2026}.

\bibliography{bib}

@incollection{Achddou2021,
  doi       = {10.1007/978-3-030-75549-2_27},
  url       = {https://doi.org/10.1007/978-3-030-75549-2_27},
  year      = {2021},
  publisher = {Springer International Publishing},
  pages     = {333--345},
  author    = {Raphaël Achddou and 
               Yann Gousseau and 
               Saïd Ladjal},
  title     = {Synthetic Images as a Regularity Prior for Image Restoration Neural Networks},
  booktitle = {Lecture Notes in Computer Science}
}

@misc{Achddou2025,
  doi       = {10.48550/ARXIV.2504.10201},
  url       = {https://arxiv.org/abs/2504.10201},
  author    = {Achddou,  Raphael and Gousseau,  Yann and Ladjal,  Saïd and S\"{u}sstrunk,  Sabine},
  title     = {VibrantLeaves: A principled parametric image generator for training deep restoration models},
  publisher = {arXiv},
  year      = {2025},
  copyright = {arXiv.org perpetual,  non-exclusive license}
}

@incollection{Alvarez1999,
  doi       = {10.1016/s1076-5670(08)70218-0},
  url       = {https://doi.org/10.1016/s1076-5670(08)70218-0},
  year      = {1999},
  publisher = {Elsevier},
  pages     = {167--242},
  author    = {Luis Alvarez and 
               Yann Gousseau and 
               Jean-Michel Morel},
  title     = {The Size of Objects in Natural and Artificial Images},
  booktitle = {Advances in Imaging and Electron Physics}
}

@article{Ariely2001,
  doi       = {10.1111/1467-9280.00327},
  url       = {https://doi.org/10.1111/1467-9280.00327},
  year      = {2001},
  month     = mar,
  publisher = {{SAGE} Publications},
  volume    = {12},
  number    = {2},
  pages     = {157--162},
  author    = {Dan Ariely},
  title     = {Seeing Sets: Representation by Statistical Properties},
  journal   = {Psychological Science}
}

@inproceedings{Arnab2016,
  doi       = {10.5244/c.30.19},
  url       = {https://doi.org/10.5244/c.30.19},
  year      = {2016},
  publisher = {British Machine Vision Association},
  author    = {Anurag Arnab and 
               Philip Torr},
  title     = {Bottom-up Instance Segmentation using Deep Higher-Order {CRFs}},
  booktitle = {Proceedings of the British Machine Vision Conference 2016}
}

@article{Artmann2018,
  title     = {Measurement of Noise using the dead leaves pattern},
  volume    = {30},
  issn      = {2470-1173},
  url       = {http://dx.doi.org/10.2352/ISSN.2470-1173.2018.12.IQSP-341},
  doi       = {10.2352/issn.2470-1173.2018.12.iqsp-341},
  number    = {12},
  journal   = {Electronic Imaging},
  publisher = {Society for Imaging Science & Technology},
  author    = {Artmann,  Uwe},
  year      = {2018},
  month     = Jan,
  pages     = {341–1–341–6}
}

@article{Attneave1954,
  title     = {Some informational aspects of visual perception.},
  volume    = {61},
  issn      = {0033-295X},
  url       = {http://dx.doi.org/10.1037/h0054663},
  doi       = {10.1037/h0054663},
  number    = {3},
  journal   = {Psychological Review},
  publisher = {American Psychological Association (APA)},
  author    = {Attneave,  Fred},
  year      = {1954},
  pages     = {183–193}
}

@misc{Baradad2021,
  doi       = {10.48550/ARXIV.2106.05963},
  url       = {https://arxiv.org/abs/2106.05963},
  author    = {Baradad,  Manel and Wulff,  Jonas and Wang,  Tongzhou and Isola,  Phillip and Torralba,  Antonio},
  title     = {Learning to See by Looking at Noise},
  publisher = {arXiv},
  year      = {2021},
  copyright = {arXiv.org perpetual,  non-exclusive license}
}

@inbook{Barlow1961,
  title     = {Possible Principles Underlying the Transformations of Sensory Messages},
  url       = {http://dx.doi.org/10.7551/mitpress/9780262518420.003.0013},
  doi       = {10.7551/mitpress/9780262518420.003.0013},
  booktitle = {Sensory Communication},
  publisher = {The MIT Press},
  author    = {Barlow,  H. B.},
  year      = {1961},
  month     = sep,
  pages     = {216–234}
}

@article{Bordenave2006,
  doi       = {10.1239/aap/1143936138},
  url       = {https://doi.org/10.1239/aap/1143936138},
  year      = {2006},
  month     = mar,
  publisher = {Cambridge University Press ({CUP})},
  volume    = {38},
  number    = {1},
  pages     = {31--46},
  author    = {Charles Bordenave and 
               Yann Gousseau and 
               Fran{\c{c}}ois Roueff},
  title     = {The dead leaves model: a general tessellation modeling occlusion},
  journal   = {Advances in Applied Probability}
}

@book{Brodatz1999,
  address    = {Mineola, N.Y},
  title      = {Textures: a photographic album for artists and designers},
  isbn       = {978-0-486-40699-2},
  shorttitle = {Textures},
  publisher  = {Dover Publications},
  author     = {Brodatz, Phil},
  year       = {1999}
}

@article{Burge2020,
  doi       = {10.1146/annurev-vision-030320-041134},
  url       = {https://doi.org/10.1146/annurev-vision-030320-041134},
  year      = {2020},
  month     = sep,
  publisher = {Annual Reviews},
  volume    = {6},
  number    = {1},
  pages     = {491--517},
  author    = {Johannes Burge},
  title     = {Image-Computable Ideal Observers for Tasks with Natural Stimuli},
  journal   = {Annual Review of Vision Science}
}

@inproceedings{Cao2010,
  doi       = {10.1117/12.838902},
  url       = {https://doi.org/10.1117/12.838902},
  year      = {2010},
  month     = jan,
  publisher = {{SPIE}},
  author    = {Fr{\'{e}}d{\'{e}}ric Cao and 
               Fr{\'{e}}d{\'{e}}ric Guichard and 
               Herv{\'{e}} Hornung},
  editor    = {Francisco Imai and 
               Nitin Sampat and 
               Feng Xiao},
  title     = {Dead leaves model for measuring texture quality on a digital camera},
  booktitle = {Digital Photography {VI}}
}

@inproceedings{Chen2019,
  doi       = {10.1109/iccv.2019.00215},
  url       = {https://doi.org/10.1109/iccv.2019.00215},
  year      = {2019},
  month     = oct,
  publisher = {{IEEE}},
  author    = {Xinlei Chen and 
               Ross Girshick and 
               Kaiming He and 
               Piotr Dollar},
  title     = {{TensorMask}: A Foundation for Dense Object Segmentation},
  booktitle = {2019 {IEEE}/{CVF} International Conference on Computer Vision ({ICCV})}
}

@article{CPIQ2023,
  author   = {CPIQ},
  journal  = {IEEE Std 1858-2023 (Revision of IEEE Std 1858-2016)},
  title    = {{IEEE Standard for Camera Phone Image Quality (CPIQ)}},
  year     = {2023},
  pages    = {1-170},
  doi      = {10.1109/IEEESTD.2023.10205967}
}

@article{Creswell2018,
  doi       = {10.1109/msp.2017.2765202},
  url       = {https://doi.org/10.1109/msp.2017.2765202},
  year      = {2018},
  month     = jan,
  publisher = {Institute of Electrical and Electronics Engineers ({IEEE})},
  volume    = {35},
  number    = {1},
  pages     = {53--65},
  author    = {Antonia Creswell and 
               Tom White and 
               Vincent Dumoulin and 
               Kai Arulkumaran and 
               Biswa Sengupta and 
               Anil A. Bharath},
  title     = {Generative Adversarial Networks: An Overview},
  journal   = {{IEEE} Signal Processing Magazine}
}

@article{deVries1943,
  doi       = {10.1016/s0031-8914(43)90575-0},
  url       = {https://doi.org/10.1016/s0031-8914(43)90575-0},
  year      = {1943},
  month     = jul,
  publisher = {Elsevier {BV}},
  volume    = {10},
  number    = {7},
  pages     = {553--564},
  author    = {H.L. de Vries},
  title     = {The quantum character of light and its bearing upon threshold of vision,  the differential sensitivity and visual acuity of the eye},
  journal   = {Physica}
}

@article{Eckstein2011,
  doi       = {10.1167/11.5.14},
  url       = {https://doi.org/10.1167/11.5.14},
  year      = {2011},
  month     = dec,
  publisher = {Association for Research in Vision and Ophthalmology ({ARVO})},
  volume    = {11},
  number    = {5},
  pages     = {14--14},
  author    = {Miguel P. Eckstein},
  title     = {Visual search: A retrospective},
  journal   = {Journal of Vision}
}

@article{Elder2002,
  doi       = {10.1167/2.4.5},
  url       = {https://doi.org/10.1167/2.4.5},
  year      = {2002},
  month     = aug,
  publisher = {Association for Research in Vision and Ophthalmology ({ARVO})},
  volume    = {2},
  number    = {4},
  pages     = {5},
  author    = {James H. Elder and 
               Richard M. Goldberg},
  title     = {Ecological statistics of Gestalt laws for the perceptual organization of contours},
  journal   = {Journal of Vision}
}

@article{Fruend2019,
  doi       = {10.1101/711804},
  url       = {https://doi.org/10.1101/711804},
  year      = {2019},
  month     = jul,
  publisher = {Cold Spring Harbor Laboratory},
  author    = {Ingo Fründ and 
               Jaykishan Patel and 
               Elee D. Stalker},
  title     = {Contrast invariant tuning in human perception of image content},
  journal   = {bioRxiv}
}

@article{Geisler2009,
  doi       = {10.1017/s0952523808080875},
  url       = {https://doi.org/10.1017/s0952523808080875},
  year      = {2009},
  month     = jan,
  publisher = {Cambridge University Press ({CUP})},
  volume    = {26},
  number    = {1},
  pages     = {109--121},
  author    = {Wilson S. Geisler and 
               Jeffery S. Perry},
  title     = {Contour statistics in natural images: Grouping across occlusions},
  journal   = {Visual Neuroscience}
}

@article{Geisler2011,
  doi       = {10.1016/j.visres.2010.09.027},
  url       = {https://doi.org/10.1016/j.visres.2010.09.027},
  year      = {2011},
  month     = apr,
  publisher = {Elsevier {BV}},
  volume    = {51},
  number    = {7},
  pages     = {771--781},
  author    = {Wilson S. Geisler},
  title     = {Contributions of ideal observer theory to vision research},
  journal   = {Vision Research}
}

@article{Gold2008,
  doi       = {10.3758/pp.70.1.88},
  url       = {https://doi.org/10.3758/pp.70.1.88},
  year      = {2008},
  month     = jan,
  publisher = {Springer Science and Business Media {LLC}},
  volume    = {70},
  number    = {1},
  pages     = {88--95},
  author    = {Jason M. Gold and 
               Duje Tadin and 
               Susan C. Cook and 
               Randolph Blake},
  title     = {The efficiency of biological motion perception},
  journal   = {Perception {\&} Psychophysics}
}

@article{Goussea2003,
  doi       = {10.48550/ARXIV.MATH/0312035},
  url       = {https://arxiv.org/abs/math/0312035},
  author    = {Gousseau,  Yann and 
               Roueff,  Francois},
  title     = {The dead leaves model : general results and limits at small scales},
  publisher = {arXiv},
  year      = {2003},
  journal   = {arXiv}
}

@article{Gousseau2007,
  doi       = {10.1137/060659041},
  url       = {https://doi.org/10.1137/060659041},
  year      = {2007},
  month     = jan,
  publisher = {Society for Industrial {\&} Applied Mathematics ({SIAM})},
  volume    = {6},
  number    = {1},
  pages     = {105--134},
  author    = {Yann Gousseau and 
               Fran{\c{c}}ois Roueff},
  title     = {Modeling Occlusion and Scaling in Natural Images},
  journal   = {Multiscale Modeling {\&} Simulation}
}

@article{Gu2022,
  doi       = {10.1016/j.imavis.2022.104401},
  url       = {https://doi.org/10.1016/j.imavis.2022.104401},
  year      = {2022},
  month     = apr,
  publisher = {Elsevier {BV}},
  volume    = {120},
  pages     = {104401},
  author    = {Wenchao Gu and 
               Shuang Bai and 
               Lingxing Kong},
  title     = {A review on 2D instance segmentation based on deep neural networks},
  journal   = {Image and Vision Computing}
}

@article{Hafiz2020,
  doi       = {10.1007/s13735-020-00195-x},
  url       = {https://doi.org/10.1007/s13735-020-00195-x},
  year      = {2020},
  month     = jul,
  publisher = {Springer Science and Business Media {LLC}},
  volume    = {9},
  number    = {3},
  pages     = {171--189},
  author    = {Abdul Mueed Hafiz and 
               Ghulam Mohiuddin Bhat},
  title     = {A survey on instance segmentation: state of the art},
  journal   = {International Journal of Multimedia Information Retrieval}
}

@inproceedings{He2017,
  author    = {He, Kaiming and 
               Gkioxari, Georgia and 
               Dollar, Piotr and 
               Girshick, Ross},
  title     = {Mask R-CNN},
  booktitle = {Proceedings of the IEEE International Conference on Computer Vision (ICCV)},
  month     = {Oct},
  year      = {2017},
  doi       = {10.48550/arXiv.1703.06870},
  url       = {https://arxiv.org/abs/1703.06870},
  publisher = {arXiv}
}

@article{Hoppe2019,
  doi       = {10.1038/s41598-018-37536-0},
  url       = {https://doi.org/10.1038/s41598-018-37536-0},
  year      = {2019},
  month     = jan,
  publisher = {Springer Science and Business Media {LLC}},
  volume    = {9},
  number    = {1},
  author    = {David Hoppe and 
               Constantin A. Rothkopf},
  title     = {Multi-step planning of eye movements in visual search},
  journal   = {Scientific Reports}
}

@article{Hubert1985,
  title     = {Comparing partitions},
  volume    = {2},
  issn      = {1432-1343},
  url       = {http://dx.doi.org/10.1007/BF01908075},
  doi       = {10.1007/bf01908075},
  number    = {1},
  journal   = {Journal of Classification},
  publisher = {Springer Science and Business Media LLC},
  author    = {Hubert,  Lawrence and Arabie,  Phipps},
  year      = {1985},
  month     = Dec,
  pages     = {193–218}
}

@incollection{Landy2014,
  title     = {Texture Analysis and Perception},
  author    = {Landy, Michael S.},
  year      = {2014},
  booktitle = {The new visual neurosciences},
  chapter   = {45},
  publisher = {MIT Press},
  pages     = {639--652}
}

@article{Lee2001,
  doi       = {10.1023/a:1011109015675},
  url       = {https://doi.org/10.1023/a:1011109015675},
  year      = {2001},
  publisher = {Springer Science and Business Media {LLC}},
  volume    = {41},
  number    = {1/2},
  pages     = {35--59},
  author    = {Ann B. Lee and 
               David Mumford and 
               Jinggang Huang},
  journal   = {International Journal of Computer Vision},
  title     = {Occlusion Models for Natural Images: A Statistical Study of a Scale-Invariant Dead Leaves Model}
}

@article{Madhusudana2022,
  doi       = {10.1109/lsp.2021.3132289},
  url       = {https://doi.org/10.1109/lsp.2021.3132289},
  year      = {2022},
  publisher = {Institute of Electrical and Electronics Engineers ({IEEE})},
  volume    = {29},
  pages     = {209--213},
  author    = {Pavan C. Madhusudana and 
               Seok-Jun Lee and 
               Hamid R. Sheikh},
  title     = {Revisiting Dead Leaves Model: Training With Synthetic Data},
  journal   = {{IEEE} Signal Processing Letters}
}

@article{Mahncke2026,
  doi       = {10.21105/joss.10531},
  url       = {https://doi.org/10.21105/joss.10531},
  year      = {2026},
  publisher = {The Open Journal},
  volume    = {11},
  number    = {123},
  author    = {Mahncke, Swantje and Wallis, Thomas S. A. and Schmittwilken, Lynn},
  title     = {deadleaves: Creating cluttered visual stimuli in Python},
  journal   = {Journal of Open Source Software}
}

@article{Maiello2017,
  doi       = {10.1167/17.5.3},
  url       = {https://doi.org/10.1167/17.5.3},
  year      = {2017},
  month     = jun,
  publisher = {Association for Research in Vision and Ophthalmology ({ARVO})},
  volume    = {17},
  number    = {5},
  pages     = {3},
  author    = {Guido Maiello and 
               Lenna Walker and 
               Peter J. Bex and 
               Fuensanta A. Vera-Diaz},
  title     = {Blur perception throughout the visual field in myopia and emmetropia},
  journal   = {Journal of Vision}
}

@techreport{Matheron1968,
  address     = {Paris, France},
  type        = {Technial report},
  title       = {Schéma booléen séquentiel de partition aléatoire},
  url         = {http://cg.ensmp.fr/bibliotheque/public/MATHERON_Rapport_00121.pdf},
  language    = {fr},
  number      = {89},
  institution = {Centre de Morphologie Mathématique de MINES},
  author      = {Matheron, Georges},
  year        = {1968},
  pages       = {1--17}
}

@book{Molchanov2017,
  doi       = {10.1007/978-1-4471-7349-6},
  url       = {https://doi.org/10.1007/978-1-4471-7349-6},
  year      = {2017},
  publisher = {Springer London},
  author    = {Ilya Molchanov},
  title     = {Theory of Random Sets}
}

@article{Mumford2001,
  doi       = {10.1090/qam/1811096},
  url       = {https://doi.org/10.1090/qam/1811096},
  year      = {2001},
  publisher = {American Mathematical Society ({AMS})},
  volume    = {59},
  number    = {1},
  pages     = {85--111},
  author    = {David Mumford and 
               Basilis Gidas},
  title     = {Stochastic models for generic images},
  journal   = {Quarterly of Applied Mathematics}
}

@article{Najemnik2005,
  doi       = {10.1038/nature03390},
  url       = {https://doi.org/10.1038/nature03390},
  year      = {2005},
  month     = mar,
  publisher = {Springer Science and Business Media {LLC}},
  volume    = {434},
  number    = {7031},
  pages     = {387--391},
  author    = {Jiri Najemnik and 
               Wilson S. Geisler},
  title     = {Optimal eye movement strategies in visual search},
  journal   = {Nature}
}

@article{Neri2010,
  doi       = {10.1016/j.visres.2009.12.015},
  url       = {https://doi.org/10.1016/j.visres.2009.12.015},
  year      = {2010},
  month     = mar,
  publisher = {Elsevier {BV}},
  volume    = {50},
  number    = {6},
  pages     = {557--563},
  author    = {Peter Neri and 
               Alicia Liu and 
               Dennis M. Levi},
  title     = {Human efficiency for classifying natural versus random text},
  journal   = {Vision Research}
}

@article{Olshausen1996a,
  title     = {Natural image statistics and efficient coding},
  volume    = {7},
  issn      = {1361-6536},
  url       = {http://dx.doi.org/10.1088/0954-898X_7_2_014},
  doi       = {10.1088/0954-898x_7_2_014},
  number    = {2},
  journal   = {Network: Computation in Neural Systems},
  publisher = {Informa UK Limited},
  author    = {Olshausen,  B A and Field,  D J},
  year      = {1996},
  month     = jan,
  pages     = {333–339}
}

@article{Olshausen1996b,
  title     = {Emergence of simple-cell receptive field properties by learning a sparse code for natural images},
  volume    = {381},
  issn      = {1476-4687},
  url       = {http://dx.doi.org/10.1038/381607a0},
  doi       = {10.1038/381607a0},
  number    = {6583},
  journal   = {Nature},
  publisher = {Springer Science and Business Media LLC},
  author    = {Olshausen,  Bruno A. and Field,  David J.},
  year      = {1996},
  month     = jun,
  pages     = {607–609}
}

@article{Peters2021,
  doi       = {10.1038/s41562-021-01194-6},
  url       = {https://doi.org/10.1038/s41562-021-01194-6},
  year      = {2021},
  month     = sep,
  publisher = {Springer Science and Business Media {LLC}},
  volume    = {5},
  number    = {9},
  pages     = {1127--1144},
  author    = {Benjamin Peters and 
               Nikolaus Kriegeskorte},
  title     = {Capturing the objects of vision with neural networks},
  journal   = {Nature Human Behaviour}
}

@article{Peterson2012,
  doi       = {10.1073/pnas.1214269109},
  url       = {https://doi.org/10.1073/pnas.1214269109},
  year      = {2012},
  month     = nov,
  publisher = {Proceedings of the National Academy of Sciences},
  volume    = {109},
  number    = {48},
  author    = {Matthew F. Peterson and 
               Miguel P. Eckstein},
  title     = {Looking just below the eyes is optimal across face recognition tasks},
  journal   = {Proceedings of the National Academy of Sciences}
}

@article{Pitkow2010,
  doi       = {10.1167/10.14.42},
  url       = {https://doi.org/10.1167/10.14.42},
  year      = {2010},
  month     = dec,
  publisher = {Association for Research in Vision and Ophthalmology ({ARVO})},
  volume    = {10},
  number    = {14},
  pages     = {42--42},
  author    = {Xaq Pitkow},
  title     = {Exact feature probabilities in images with occlusion},
  journal   = {Journal of Vision}
}

@inproceedings{Prince2007,
  doi       = {10.1109/iccv.2007.4409052},
  url       = {https://doi.org/10.1109/iccv.2007.4409052},
  year      = {2007},
  publisher = {{IEEE}},
  author    = {Simon J.D. Prince and 
               James H. Elder},
  title     = {Probabilistic Linear Discriminant Analysis for Inferences About Identity},
  booktitle = {2007 {IEEE} 11th International Conference on Computer Vision}
}

@article{Rand1971,
  title     = {Objective Criteria for the Evaluation of Clustering Methods},
  volume    = {66},
  issn      = {1537-274X},
  url       = {http://dx.doi.org/10.1080/01621459.1971.10482356},
  doi       = {10.1080/01621459.1971.10482356},
  number    = {336},
  journal   = {Journal of the American Statistical Association},
  publisher = {Informa UK Limited},
  author    = {Rand,  William M.},
  year      = {1971},
  month     = Dec,
  pages     = {846–850}
}

@article{Rose1948,
  doi       = {10.1364/josa.38.000196},
  url       = {https://doi.org/10.1364/josa.38.000196},
  year      = {1948},
  month     = feb,
  publisher = {The Optical Society},
  volume    = {38},
  number    = {2},
  pages     = {196},
  author    = {Albert Rose},
  title     = {The Sensitivity Performance of the Human Eye on an Absolute Scale},
  journal   = {Journal of the Optical Society of America}
}

@incollection{Rosenholtz2014,
  doi       = {10.1093/oxfordhb/9780199686858.013.058},
  url       = {https://doi.org/10.1093/oxfordhb/9780199686858.013.058},
  year      = {2014},
  month     = jul,
  publisher = {Oxford University Press},
  author    = {Ruth Rosenholtz},
  chapter   = {9},
  title     = {Texture Perception},
  booktitle = {The Oxford Handbook of perceptual organization}
}

@article{Ruderman1994,
  title     = {Statistics of natural images: Scaling in the woods},
  volume    = {73},
  issn      = {0031-9007},
  url       = {http://dx.doi.org/10.1103/PhysRevLett.73.814},
  doi       = {10.1103/physrevlett.73.814},
  number    = {6},
  journal   = {Physical Review Letters},
  publisher = {American Physical Society (APS)},
  author    = {Ruderman,  Daniel L. and Bialek,  William},
  year      = {1994},
  month     = aug,
  pages     = {814–817}
}

@article{Ruderman1997,
  doi       = {10.1016/s0042-6989(97)00008-4},
  url       = {https://doi.org/10.1016/s0042-6989(97)00008-4},
  year      = {1997},
  month     = dec,
  publisher = {Elsevier {BV}},
  volume    = {37},
  number    = {23},
  pages     = {3385--3398},
  author    = {Daniel L. Ruderman},
  title     = {Origins of scaling in natural images},
  journal   = {Vision Research}
}

@inproceedings{Schlecht2011,
  doi       = {10.5244/c.25.50},
  url       = {https://doi.org/10.5244/c.25.50},
  year      = {2011},
  publisher = {British Machine Vision Association},
  author    = {Joseph Schlecht and Bj\"{o}rn Ommer},
  title     = {Contour-based object detection},
  booktitle = {Proceedings of the British Machine Vision Conference 2011}
}

@article{Sharma2022,
  doi       = {10.1007/s42979-022-01407-3},
  url       = {https://doi.org/10.1007/s42979-022-01407-3},
  year      = {2022},
  month     = sep,
  publisher = {Springer Science and Business Media {LLC}},
  volume    = {3},
  number    = {6},
  author    = {Rabi Sharma and 
               Muhammad Saqib and 
               C. T. Lin and 
               Michael Blumenstein},
  title     = {A Survey on Object Instance Segmentation},
  journal   = {{SN} Computer Science}
}

@book{Steinmetz2024,
  title     = {Analysis},
  isbn      = {9783662680865},
  url       = {http://dx.doi.org/10.1007/978-3-662-68086-5},
  doi       = {10.1007/978-3-662-68086-5},
  publisher = {Springer Berlin Heidelberg},
  author    = {Steinmetz,  Norbert},
  year      = {2024}
}

@article{Straub2022,
  doi       = {10.7554/elife.76635},
  url       = {https://doi.org/10.7554/elife.76635},
  year      = {2022},
  month     = oct,
  publisher = {{eLife} Sciences Publications,  Ltd},
  volume    = {11},
  author    = {Dominik Straub and 
               Constantin A Rothkopf},
  title     = {Putting perception into action with inverse optimal control for continuous psychophysics},
  journal   = {{eLife}}
}

@article{Taylor2015,
  doi       = {10.1016/j.visres.2014.12.025},
  url       = {https://doi.org/10.1016/j.visres.2014.12.025},
  year      = {2015},
  month     = oct,
  publisher = {Elsevier {BV}},
  volume    = {115},
  pages     = {142--150},
  author    = {Christopher Patrick Taylor and 
               Peter J. Bex},
  title     = {On the number of perceivable blur levels in naturalistic images},
  journal   = {Vision Research}
}

@article{Tsao2022,
  doi       = {10.1073/pnas.2204248119},
  url       = {https://doi.org/10.1073/pnas.2204248119},
  year      = {2022},
  month     = oct,
  publisher = {Proceedings of the National Academy of Sciences},
  volume    = {119},
  number    = {41},
  author    = {Thomas Tsao and 
               Doris Y. Tsao},
  title     = {A topological solution to object segmentation and tracking},
  journal   = {Proceedings of the National Academy of Sciences}
}

@article{Vacher2023,
  title     = {Measuring uncertainty in human visual segmentation},
  volume    = {19},
  issn      = {1553-7358},
  url       = {http://dx.doi.org/10.1371/journal.pcbi.1011483},
  doi       = {10.1371/journal.pcbi.1011483},
  number    = {9},
  journal   = {PLOS Computational Biology},
  publisher = {Public Library of Science (PLoS)},
  author    = {Vacher,  Jonathan and Launay,  Claire and Mamassian,  Pascal and Coen-Cagli,  Ruben},
  editor    = {Robinson,  Emma Claire},
  year      = {2023},
  month     = sep,
  pages     = {e1011483}
}

@article{Wallis2012,
  doi       = {10.1167/12.7.6},
  url       = {https://doi.org/10.1167/12.7.6},
  year      = {2012},
  month     = jul,
  publisher = {Association for Research in Vision and Ophthalmology ({ARVO})},
  volume    = {12},
  number    = {7},
  pages     = {6--6},
  author    = {Thomas S. A. Wallis and 
               Peter J. Bex},
  title     = {Image correlates of crowding in natural scenes},
  journal   = {Journal of Vision}
}

@article{Walther2014,
  title     = {Nonaccidental Properties Underlie Human Categorization of Complex Natural Scenes},
  volume    = {25},
  issn      = {1467-9280},
  url       = {http://dx.doi.org/10.1177/0956797613512662},
  doi       = {10.1177/0956797613512662},
  number    = {4},
  journal   = {Psychological Science},
  publisher = {SAGE Publications},
  author    = {Walther,  Dirk B. and Shen,  Dandan},
  year      = {2014},
  month     = jan,
  pages     = {851–860}
}

@article{Webster1997,
  doi       = {10.1016/s0042-6989(97)00125-9},
  url       = {https://doi.org/10.1016/s0042-6989(97)00125-9},
  year      = {1997},
  month     = dec,
  publisher = {Elsevier {BV}},
  volume    = {37},
  number    = {23},
  pages     = {3283--3298},
  author    = {Michael A. Webster and 
               J.D. Mollon},
  title     = {Adaptation and the color statistics of natural images},
  journal   = {Vision Research}
}

@article{Webster2007,
  doi       = {10.1080/09548980701654405},
  url       = {https://doi.org/10.1080/09548980701654405},
  year      = {2007},
  month     = jan,
  publisher = {Informa {UK} Limited},
  volume    = {18},
  number    = {3},
  pages     = {213--233},
  author    = {Michael A. Webster and 
               Yoko Mizokami and 
               Shernaaz M. Webster},
  title     = {Seasonal variations in the color statistics of natural images},
  journal   = {Network: Computation in Neural Systems}
}

@incollection{Whitney2014,
  title     = {From Textures to Crowds: Multiple Levels of Summary Statistical Perception},
  booktitle = {The new visual neurosciences},
  author    = {Whitney, David and 
               Haberman, Jason and 
               Sweeny, Timothy},
  year      = {2014},
  pages     = {695--710},
  chapter   = {49},
  publisher = {MIT Press}
}

@article{Wilder2018,
  title     = {Spatial relationships between contours impact rapid scene classification},
  volume    = {18},
  issn      = {1534-7362},
  url       = {http://dx.doi.org/10.1167/18.8.1},
  doi       = {10.1167/18.8.1},
  number    = {8},
  journal   = {Journal of Vision},
  publisher = {Association for Research in Vision and Ophthalmology (ARVO)},
  author    = {Wilder,  John and Dickinson,  Sven and Jepson,  Allan and Walther,  Dirk B.},
  year      = {2018},
  month     = aug,
  pages     = {1}
}

@article{Yao2020,
  doi       = {10.1145/3391743},
  url       = {https://doi.org/10.1145/3391743},
  year      = {2020},
  month     = may,
  publisher = {Association for Computing Machinery ({ACM})},
  volume    = {11},
  number    = {4},
  pages     = {1--47},
  author    = {Rui Yao and 
               Guosheng Lin and 
               Shixiong Xia and 
               Jiaqi Zhao and 
               Yong Zhou},
  title     = {Video Object Segmentation and Tracking},
  journal   = {{ACM} Transactions on Intelligent Systems and Technology}
}

@article{Zou2023,
  doi       = {10.1109/jproc.2023.3238524},
  url       = {https://doi.org/10.1109/jproc.2023.3238524},
  year      = {2023},
  month     = mar,
  publisher = {Institute of Electrical and Electronics Engineers ({IEEE})},
  volume    = {111},
  number    = {3},
  pages     = {257--276},
  author    = {Zhengxia Zou and 
               Keyan Chen and 
               Zhenwei Shi and 
               Yuhong Guo and 
               Jieping Ye},
  title     = {Object Detection in 20 Years: A Survey},
  journal   = {Proceedings of the {IEEE}}
}

\newpage

\setcounter{section}{0}
\renewcommand{\thesection}{\Alph{section}}

\section{Random closed sets}\label{app:preliminaries}
We can think of an object in the real world as the realization of a connected random closed set, i.\,e., a volume in space including its boundary.
Formally, a random closed set can be defined as follows.

\begin{defi}[Random closed set {\cite[p.~2]{Molchanov2017}}]\label{def:random closed set}
    Let $(\Omega,\Sigma,\P)$ be a probability space. Let $\F$ denote the family of closed subsets of $\R^d$, and let $(\F,\mathcal{B}(\F))$ be the measurable space where $\mathcal{B}(\F)$ is the Borel $\sigma$-algebra on $\F$.
    A measurable map
    \[X\colon (\Omega,\Sigma,\P) \to (\F,\mathcal{B}(\F))\]
    is called a \emph{random closed set}.
\end{defi}

A random set can be regarded as a random variable that takes sets as values rather than scalars. When sampling, each set is selected with probability given by the probability measure $\P$.
Figure~\ref{fig:random_sets} shows several example samples from different simple distributions of closed sets in $\R^2$.

\begin{figure}[ht]\centering
    \adjustbox{minipage=1em,raise=2mm}{\subcaption{}\label{fig:random circle}}
    \begin{subfigure}{0.4\textwidth}\centering
        \includegraphics[width=\linewidth, trim=75 0 75 0,clip]{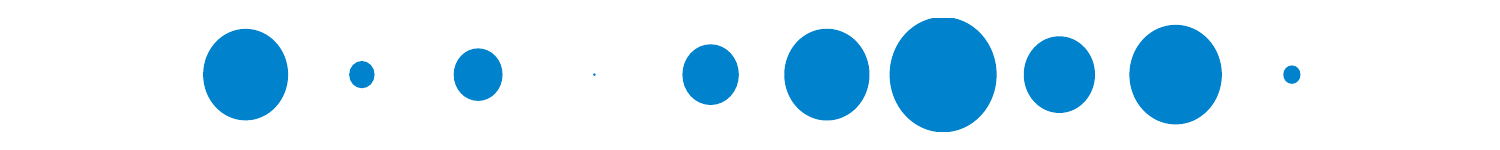}
    \end{subfigure}

    \adjustbox{minipage=1em,raise=2mm}{\subcaption{}\label{fig:random ellipses}}
    \begin{subfigure}{0.4\textwidth}\centering
        \includegraphics[width=\linewidth, trim=75 0 75 0,clip]{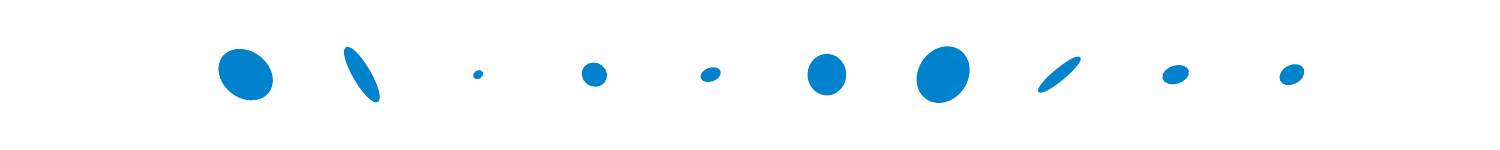}
    \end{subfigure}

    \adjustbox{minipage=1em,raise=2mm}{\subcaption{}\label{fig:random rectangles}}
    \begin{subfigure}{0.4\textwidth}\centering
        \includegraphics[width=\linewidth, trim=90 0 60 0,clip]{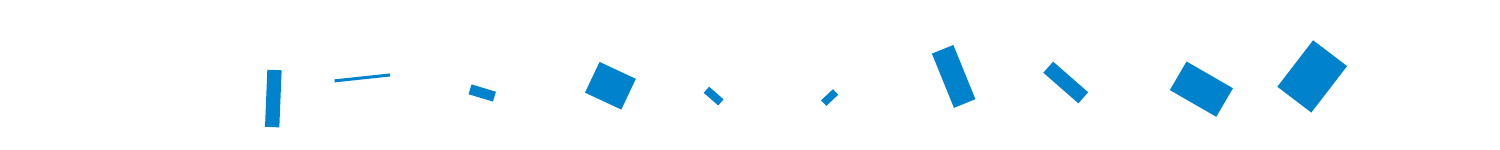}
    \end{subfigure}
    \caption{Simple examples of random sets. The figures show samples of circles with uniformly distributed radius (\subref{fig:random circle}), ellipses with uniformly distributed width, height, and angle (\subref{fig:random ellipses}), and rectangles with uniformly distributed width, height, and angle (\subref{fig:random rectangles}).}\label{fig:random_sets}
\end{figure}

In our case we consider the three-dimensional space ($d=3$), so that the probability distribution of $P$ reflects the occurrence of specific objects in the real world. In an ideal model, we would know this distribution. Since it is likely impossible to know the true distribution of objects, we must rely on the best approximation available in the later model construction.

\end{document}